\documentclass{article}

\usepackage{fullpage}

\newcommand{\mA}{{\bf A}}
\newcommand{\mD}{{\bf D}}
\newcommand{\mH}{{\bf H}}
\newcommand{\mI}{{\bf I}}
\newcommand{\mM}{{\bf M}}

\newcommand{\cO}{{\cal O}}
\newcommand{\eqdef}{\coloneqq}
\newcommand{\cD}{{\cal D}}
\newcommand{\argmin}{\mbox{argmin}}
\newcommand{\dotprod}[1]{\left< #1\right>} 
\newcommand{\norm}[1]{ \left\| #1 \right\|}      
\newcommand{\R}{\mathbb{R}} 
\newcommand{\N}{\mathbb{N}} 

\newcommand{\diag}[1]{\mathbf{diag}\left( #1\right)}

\newcommand{\E}[1]{\mathbb{E}\left[#1\right] }
\newcommand{\EE}[2]{\mathbb{E}_{#1}\left[#2\right] }

\newcommand{\ones}{{\bf 1}}

\usepackage[colorinlistoftodos,bordercolor=orange,backgroundcolor=orange!20,linecolor=orange,textsize=scriptsize]{todonotes}

\newcommand{\ie}{\textit{i.e. }} 
\newcommand*{\vertbar}{\rule[-1ex]{0.5pt}{2.5ex}}
\newcommand*{\horzbar}{\rule[.5ex]{2.5ex}{0.5pt}}

\usepackage{bbm}
\usepackage{tikz}

\usepackage[utf8]{inputenc}
\usepackage{booktabs}       

\usepackage{enumitem}

\usepackage{fancyhdr} 
\newcommand\figsizefour{3.3cm}

\usepackage{microtype}
\usepackage{graphicx}
\usepackage{subfigure}
\usepackage{booktabs} 

\usepackage{hyperref}
\graphicspath{{./figures/}}

\usepackage{amsmath}
\usepackage{amssymb}
\usepackage{mathtools}
\usepackage{amsthm}

\usepackage{mdframed}
\usepackage{thmtools}

\definecolor{shadecolor}{gray}{0.90}
\declaretheoremstyle[
headfont=\normalfont\bfseries,
notefont=\mdseries, notebraces={(}{)},
bodyfont=\normalfont,
postheadspace=0.5em,
spaceabove=1pt,
mdframed={
  roundcorner=30pt,
  skipabove=8pt,
  skipbelow=8pt,
  hidealllines=true,
  backgroundcolor={shadecolor},
  innerleftmargin=4pt,
  innerrightmargin=4pt}
]{shaded}

\declaretheorem[style=shaded,within=section]{definition}
\declaretheorem[style=shaded,sibling=definition]{theorem}

\declaretheorem[style=shaded,sibling=definition]{lemma}
\declaretheorem[style=shaded,sibling=definition]{remark}

\usepackage[utf8]{inputenc} 
\usepackage[T1]{fontenc}    
\usepackage{hyperref}       
\usepackage{url}            
\usepackage{booktabs}       
\usepackage{amsfonts}       
\usepackage{nicefrac}       
\usepackage{microtype}      
\usepackage{xcolor}         

\title{Cutting Some Slack for SGD with Adaptive Polyak Stepsizes}

\author{Robert M. Gower \\ CCM, Flatiron Institute\footnote{Work down while R. M. Gower was a Visiting Scientist at Google Research, Brain Team.} 
\and Mathieu Blondel \\ Google Research, Brain Team 
\and Nidham Gazagnadou\\ T\'el\'ecom Paris, IPP, Paris, France
\and Fabian Pedregosa \\ Google Research, Brain Team}

\begin{document}

\maketitle

\begin{abstract}
  Tuning the step size of stochastic gradient descent is tedious and error prone.
This has motivated the development of methods that automatically adapt the step size using readily available information. In this paper, we consider the family of SPS (Stochastic gradient with a Polyak Stepsize) adaptive methods.  These are methods that make use of gradient and loss value at the sampled points to adaptively adjust the step size.  We first show that SPS and its recent variants can all be seen as extensions of the Passive-Aggressive methods applied to nonlinear problems.  We use this insight to develop new variants of the SPS method that are better suited to nonlinear models.
Our new variants are based on introducing a slack variable into the interpolation equations. This single slack variable tracks the loss function across iterations and is used in setting a stable step size.  We provide extensive numerical results supporting our new methods and a convergence theory.
\end{abstract}

\section{Introduction}

\texttt{SGD} (stochastic gradient descent) can be efficient at training large machine learning models if one  carefully tunes a step size scheduler. 
However, tuning step sizes is tedious and costly, especially with large-scale deep learning models. For this reason, a fruitful line of research has focused on
developing \emph{adaptive} methods, that require little to no tuning of the step sizes. A popular family of methods are those that use coordinate-wise (i.e., feature-wise) scaling, such as \texttt{Adagrad}~\cite{ADAGRAD}, \texttt{ADAM}~\cite{ADAM} and \texttt{AMSgrad}~\cite{AMSgrad}. More recently, a new family of methods called \texttt{SPS} (SGD with Polyak Stepsizes) has been proposed \cite{ALI-G,SPS}.
These methods make use of both the loss values and gradient norms at the sampled points, to automatically adjust the step size. 
This allows to use a step size that changes from iteration to another depending on the current progress, and also adapts to the particular scaling of
the loss function at hand.

Variants of \texttt{SPS} have been shown to work very well on models that interpolate the data, and even have competitive convergence guarantees in convex~\cite{SPS} and some non-convex settings~\cite{SGDstruct}. The  effectiveness of \texttt{SPS} in the interpolated setting is not surprising, given that
\texttt{SPS} directly  solves the interpolation equations~\cite{ALI-G,TASPS}. However, when the model is far from interpolation, \texttt{SPS} and its variants can fail to converge, unless additional safe guards are put in place. As a consequence \texttt{SPS} methods tend to be less efficient when using regularization (weight decay), data augmentation (such as image rotations), and generally only work well on models with many more parameters than data. 


In this paper, we develop new \emph{slack} variants of \texttt{SPS} that are less sensitive to interpolation.
Specifically,
we make the following contributions:

\paragraph{Extending (PA) Passive-Aggressive Methods.} We first show that \texttt{SPS} can be seen as a natural extension of the  \emph{Passive-Aggressive} (PA) methods~\cite{Crammer06} to nonlinear models.

\paragraph{Variational Viewpoint of Stochastic Polyak Methods.} Building upon this viewpoint, we then show that  \texttt{SPS}$_{\max}$~\cite{SPS} and \texttt{ALI-G}~\cite{ALI-G}, two recent variants of the \texttt{SPS} method, can be interpreted as \emph{slack} variants of the PA methods. That is, they allow for a margin of \emph{slack} in the interpolation equations. This in turn  
provides a 
\emph{variational formulation} of \texttt{SPS}$_{\max}$ and \texttt{ALI-G}. This viewpoint also gives  meaning to the tunable parameter $\lambda$ of \texttt{SPS}$_{\max}$ and \texttt{ALI-G}, namely, it can be interpreted as controlling the strength of a {slack variable}. The higher $\lambda$, the less slack we allow in the interpolation equations.

\paragraph{Max loss Viewpoint.} We also show that all slack variants of the PA methods, including  \texttt{SPS}$_{\max}$ and \texttt{ALI-G}, are in fact solving an auxiliary problem, whereby the objective is to have the \textbf{smallest maximum loss} over the data. 

\paragraph{New Slack Polyak Methods.} The variational formulation then leads to the insight of further \emph{regularizing} the slack variable. This enables the slack variable to track the loss function across iterations and set a stable step size. We call our proposed variants \emph{\texttt{SPSL1}} and \emph{\texttt{SPSL2}}.

\paragraph{Convergence Theory.}  We study the \emph{convergence} of our methods, and prove convergence at $\cO(1/t +\mbox{\small constant})$ rate for both \texttt{SPSL1} and \texttt{SPSL2}, where constant is zero under interpolation.

\paragraph{Experimental Validation.}  We conclude with extensive experimental results, where we test the sensitivity of each of the slack methods  \texttt{SPS}$_{\max}$, \texttt{ALI-G}, \texttt{SPSL1} and \texttt{SPSL2} to how far the underlying problem is to interpolation. We also show that the new methods  \texttt{SPSL1} and \texttt{SPSL2} are \emph{competitive} on the benchmark vision classification problems, despite \emph{not using any additional safe guards} as required to make \texttt{SPS}$_{\max}$ and \texttt{ALI-G} stable and competitive.






\section{Related work}
\label{sec:related_work}

\paragraph{ Setup and Notation.}
Let $(x_i,y_i)$ denote the  $i$\textsuperscript{th} data point, where $x_i$ is an input and $y_i$ is an output.
We use $\ell_i(w) \eqdef \ell(w ; (x_i,y_i))$ to denote the loss evaluated at parameters $w \in \R^d$ over the $i$\textsuperscript{th} data point.
Our objective is to minimize the \emph{average loss} $\ell(w) \eqdef \frac{1}{n} \sum_{i=1}^n \ell_i(w)$ over a given data set with $n$ data examples. We note that our results also hold for a true expectation where $\ell(w) = \EE{(x,y) \sim \cD}{\ell(w ; (x,y)}$, also known as the test error. Furthermore, our results hold verbatim for an \emph{online regret} given by $\frac{1}{T}\sum_{t=1}^T \ell_{t}(w),$ where $\ell_t$ is an online loss that could be chosen adversarially, which is the original setting of \texttt{PA} methods. For  simplicity, we use the average loss.  

\paragraph{ Passive-Aggressive for linear models.}
The Passive-Aggressive (\texttt{PA}) methods \cite{Crammer06} are online learning algorithms that update a model only when the loss on a new example is greater than zero. They do this by introducing the minimal perturbation into the parameters so that this new example is correctly labelled. This is the \emph{aggressive} step. On the other hand, if the example is correctly labelled already, then the methods are \emph{passive} and do nothing.
A similar online algorithm, but using Kullback-Leibler projections instead of Euclidean ones, was proposed in \cite{Herbster01}.
\texttt{PA} methods have been especially successful in natural language processing \cite{mcdonald_2005,blitzer_2006,chechik_2010}.

To formally introduce the \texttt{PA} methods, suppose  that the loss is non-negative and that $\ell_i(w) =0$ if our model correctly labels the $i$\textsuperscript{th} example. Let $w^t \in \R^d$ be the current parameters of our model at some time $t$. 
At iteration $t+1$, \texttt{PA} methods update the parameters to $w^{t+1}$ such that the $i$\textsuperscript{th} example, usually randomly sampled, is correctly classified, that is $\ell_i(w^{t+1})=0$.
They do this by making the smallest possible perturbation to the current parameter vector $w^t$ by solving the projection  
\begin{equation}
    w^{t+1} = \argmin_{w\in\R^d} \norm{w-w^t}^2 
    \mbox{subject to} ~ \ell_i(w) = 0 \enspace. \label{eq:PAproj}
\end{equation}
For simple loss functions, the update~\eqref{eq:PAproj} often enjoys a closed-form solution. 
For instance, let us consider binary classification problems where data points $(x_i,y_i)$ are such that $y_i\in \{-1,  1\}$. If one minimizes the hinge-loss $\ell_i(w) = (1- y_i\dotprod{w,x_i})_+$ over a linear model, where $a_+ \eqdef \max\{0, a\}$ denotes the non-negative part, then the closed-form solution to~\eqref{eq:PAproj} is given by  
\begin{equation} \label{eq:PA}
    w^{t+1} = w^t +\frac{\ell_i(w^t)}{\norm{x_i}^2} y_i \, x_i \enspace.
\end{equation}
See Appendix A in~\cite{Crammer06} for the proof.
The \texttt{PA} method~\eqref{eq:PAproj} is appealing because it does not require learning rates and the update~\eqref{eq:PA} adapts to the \emph{scale} of $x_i$. Unfortunately for most nonlinear models and loss functions there is no closed form solution to the projection in~\eqref{eq:PAproj}, greatly restricting the applications of \texttt{PA} methods.


\paragraph{ Passive-Aggressive extensions.}
The \texttt{PA} methods \cite{Crammer06} were initially applied to regression, classification and structured prediction problems, but always with an underlying linear model. \texttt{PA} methods have subsequently been adapted to solving the non-negative matrix factorization (NMF) problem ~\cite{BlondelKU14}. NMF can be cast as solving a nonlinear matrix equation, which is bilinear in the two unknown matrices. In \cite{BlondelKU14} the authors proposed to solve this bilinear problem by alternating between freezing one of the two unknown matrices and applying a \texttt{PA} step. This approach cannot be extended to general nonlinear models and problems since it relies on the bilinear nature of NMF.

The \texttt{PA} framework has also been very recently used in conjunction with a class of quadratic models given by a difference-of-squares (DoS)~\cite{Saul2021}. Although there is no closed-form update for~\eqref{eq:PAproj} for these DoS models, the authors propose a simple bounded one-dimensional search based on the S-procedure to find an approximate solution.

\paragraph{ \texttt{SGD} with Polyak stepsize (\texttt{SPS}).} \texttt{SPS} methods use both the gradient norm and the loss value of the sampled points. The vanilla \texttt{SPS} update is 
\begin{align}
    w^{t+1} = w^t -\frac{\ell_i(w^t)}{\norm{\nabla \ell_i(w^t)}^2} \nabla \ell_i(w^t) \enspace. \label{eq:SPS}
\end{align}
\texttt{SPS} methods have drawn much attention recently as an effective method for training deep neural networks, even though to work effectively in practice, the update above needs some additional alterations and safe-guards; otherwise the method can diverge when $\norm{\nabla \ell_i(w^t)}$ in the denominator is near zero.

A practical update to avoid this issue is that of \texttt{ALI-G} \cite{ALI-G}:  
\begin{equation}
     w^{t+1} = w^t - \min\left\{\frac{\ell_i(w^t)}{\norm{\nabla \ell_i(w^t)}^2+\varepsilon}, \; \eta \right\} \nabla \ell_i(w^t) \enspace,
\label{eq:exact_ALIG}
\end{equation}
where $\varepsilon \geq 0$ and $\eta > 0$ are tunable parameters.
If $\varepsilon$ is set to $0$ and $\ell_i (w^t) - \ell_i^*$, with nonzero $\ell_i^*$, is used in the numerator of the Polyak step size of~\eqref{eq:exact_ALIG}, then one recovers \texttt{SPS}$_{\max}$ later introduced in~\cite{SPS}.

In~\cite{TASPS}, the authors showed that \texttt{SPS} is formally a subsampled Newton-Raphson method~\cite{SNR}  for solving the nonlinear equations 
\begin{equation}\label{eq:interpolate}
    \ell_i(w) = 0 \; , \quad \mbox{for }i=1,\ldots, n \enspace. 
\end{equation} 
There only exists a solution to~\eqref{eq:interpolate} if we are in the \emph{interpolation regime}, that is to say, that our model has sufficient degrees of freedom that it can \emph{perfectly fit} all of the given examples. 
This assumption is typically met with over-parametrized deep neural networks~\cite{ZhangBHRV17,MaBB18} or in non-parametric regression models~\cite{belkin2019does,liang2020just}.
In~\cite{TASPS} the authors develop a weaker type of interpolation assumption that rests on knowing the total optimal loss $\ell(w^\star)= \frac{1}{n}\sum_{i=1}^n\ell_i(w^\star)$ and using this as a \emph{target} value. 



\paragraph{ Adaptive methods.}
The ultimate goal of adaptive methods is to have a method that works well with little to no parameter tuning.
One such adaptive method is \texttt{ADAM}~\cite{ADAM}, which works remarkably well on a wide range of DNNs with default parameter settings. 
The practical success of \texttt{ADAM} has inspired a large number of papers that attempts to improve upon \texttt{ADAM}~\cite{AMSgrad}, understand \texttt{ADAM}~\cite{balles2018dissecting,AMSgrad}, or extend it~\cite{Zaheer2018}.
Due to the vast number of variants of \texttt{ADAM} that now exist in the literature, we cannot do the literature justice, since it would require a dedicated survey paper.
Here, we restrict ourselves to the admittedly limited scope of developing a foundation for adaptive methods based on Polyak step sizes. 
\texttt{SPS} has an intuitive and mathematically grounded motivation, one that we further develop here. 
Though our methods share the same overarching objective as adaptive methods such as \texttt{ADAM} and its variants, our methods represent a distinct approach, and enjoy a formal derivation.




Model-based methods \cite{asi2019importance, Chada-accel-model-2021}  also make use of linearizations in combination with proximal operators to derive new stochastic gradient methods. In fact, the \texttt{SPS} and \texttt{ALI-G} methods can be understood as a model based method~\cite{TASPS}. 
For an excellent introduction to  online learning, with connections to model based and  PA methods we recommend the monograph~\cite{Orabona:2019} and its related blog.
Finally, another approach that is being actively researched recently is the use of \texttt{SGD} in combination with a line search~\cite{vaswani2019painless}.




\section{Variational perspective of SPS}
\label{sec:slack}
We now present our novel insights into viewing SPS and its variants as 
an extension of the PA (Passive-Aggressive) methods. In doing so, we also give new variational perspectives on SPS methods, show that the new variants of the SPS methods can be interpreted as methods that solve a weaker form of interpolation, and as consequence, also as methods that solve a certain \emph{max} loss problem.

\subsection{Passive-Aggressive for general nonlinear models}

Although Passive-Aggressive methods have been developed for some nonlinear models
(cf.\ Section \ref{sec:related_work}), we are not aware of an efficient approach
that works for general nonlinear models.
Indeed, for most nonlinear models, such as Deep Neural Networks (DNNs), there is no closed-form solution to the projection~\eqref{eq:PAproj}. To retain much of the appeal of \texttt{PA} methods and make use of nonlinear models, we propose to linearize the loss around $w^t$ and project 
\begin{align}
    w^{t+1} =  & \argmin_{w\in\R^d} \norm{w-w^t}^2 \label{eq:PAprojlin} \;\;  \mbox{ s.t. }\ell_i(w^t)+ \dotprod{\nabla \ell_i(w^t), w-w^t}=0 \enspace. 
\end{align}
The linearization in~\eqref{eq:PAprojlin} is a reasonable approximation to $\ell_i(w)$ when $w$ is close to $w^t$, and because of the projection, we choose $w$ to be as close as possible to $w^t$. It turns out that the closed-form solution of \eqref{eq:PAprojlin} is exactly given by the SPS update in \eqref{eq:SPS}.
 
Using this connection between the \texttt{SPS} and \texttt{PA} methods, we first formalize typical
safe-guarded variants of \texttt{SPS} as \texttt{PA} methods that make use of a \emph{slack} variable. We
 use these  insights to develop new and efficient slack \texttt{SPS} variants.

\subsection{Adding slack}

In most settings, the \texttt{PA} methods are in fact too aggressive. In particular,
outside of the interpolation regime, where there exists no solution to~\eqref{eq:interpolate}, setting all losses to zero is not possible, and thus eventually forcing a particular $\ell_i(w) =0$ will deviate other losses from zero. To better control the trade-off between losses over different examples, in \cite{Crammer06} a slack variable $s \in \R$ was introduced and finding $w \in \R^d$ such that
\begin{equation}\label{eq:slackconstraint}
    \ell_i(w) \leq s, \quad \mbox{for }i=1,\ldots, n \enspace,
\end{equation}
while simultaneously forcing $s$ to be as small as possible. To make $s$ small, we can either 
minimize what we call the \emph{L1 slack formulation} given by
\begin{align}
    &\min_{w\in\R^d, \, s \geq 0} s  \label{eq:slackobjmax} \quad\mbox{s.t.} \quad \ell_i(w)  \leq s, \quad \mbox{for }i=1,\ldots, n \enspace,
\end{align}
or the \emph{L2 slack formulation} given by 
\begin{align}
    \label{eq:slackobjL2}
    &\min_{w \in \R^d, \, s \in \R} s^2 \quad \mbox{s.t.} \quad \ell_i(w)  \leq s, \quad \mbox{for }i=1,\ldots, n \enspace. 
\end{align}
The slack variable allows us to control the global trade-off between different losses.

When $\ell_i(w)$ is a relatively simple function, such as the hinge loss over a linear model, then after sampling the $i$\textsuperscript{th} constraint in either~\eqref{eq:slackobjmax} or~\eqref{eq:slackobjL2}, there is a closed-form solution to the resulting variational problem as shown in~\cite{Crammer06}. 
For most nonlinear losses and models, however, there is no solution, and thus we resort to linearization once again. 
But first, we present an alternative interpretation of~\eqref{eq:slackobjmax} and~\eqref{eq:slackobjL2}.

\subsection{Max loss Viewpoint}
\label{sec:maxloss}
The slack formulations can also be seen through the lens of robust learning.
Indeed, we can substitute out the slack variable in~\eqref{eq:slackobjmax} by
observing that the smallest possible $s$ that is greater than every $\ell_i(w)$
is necessarily $\displaystyle s = \max_{i=1,\ldots, n} \ell_i(w).$
Thus~\eqref{eq:slackobjmax}  and~\eqref{eq:slackobjL2} are equivalent to  
\begin{equation}
    \min_{w \in \R^d} \max_{i=1,\ldots, n} \ell_i(w) \enspace. \label{eq:minmax}
\end{equation}
We refer to~\eqref{eq:minmax} 
as the max loss.
 The max loss~\eqref{eq:minmax} 
can be seen as robust learning tasks, where the aim is to find $w$ such that no single example is badly classified. Contrast this with the Empirical Risk Minimization (ERM) task where the goal is to minimize the average of the losses. In particular~\eqref{eq:minmax} is a conservative upper bound on the ERM since 
\begin{equation} \label{eq:maxbound}
    \ell(w) = \frac{1}{n} \sum_{i=1}^n \ell_i(w) \leq \max_{i=1,\ldots, n} \ell_i(w) \enspace.
\end{equation}
If interpolation holds, then the above is a tight over-approximation, in that the ERM and the robust objective have the same solution set. Next, we propose methods for approximately solving~\eqref{eq:slackobjmax} and~\eqref{eq:slackobjL2}.

\subsection{A Variational form of \texttt{SPS}$_{\max}$}

Let  $\lambda>0$ be a  \emph{slack} parameter. 
Our first method for solving~\eqref{eq:slackobjmax} proceeds by  
by sampling the $i$\textsuperscript{th} constraint, linearizing and projecting as follows
\begin{align}
    w^{t+1}, s^{t+1} =& \argmin_{w\in\R^d, s \geq 0} \tfrac{1}{2}\norm{w - w^t}^2 +\lambda s, \nonumber \\
    &\; \; \mbox{s.t. } \ell_i(w^t) + \dotprod{\nabla \ell_i(w^t), w - w^t} \leq s \enspace. \label{eq:SPSmaxproj}
\end{align}
Fortunately this method has a simply closed form solution.
\begin{lemma}\label{lem:SPSmax}
The closed-form solution to~\eqref{eq:SPSmaxproj} is
\begin{align}
    w^{t+1} &= w^t - \min\left\{\frac{\ell_i(w^t)}{\norm{\nabla \ell_i(w^t)}^2}, \; \lambda \right\} \nabla \ell_i(w^t) \enspace, \quad  \label{eq:SPSmax}\\
s^{t+1} &= \max\{\ell_i(w^t)- \lambda \norm{\nabla \ell_i(w^t)}^2 , \; 0 \} \enspace.
\end{align}
\end{lemma}
\begin{proof}
    The proof follows by applying Lemma~\ref{lem:slackL1ineqconst}
\end{proof}

Update~\eqref{eq:SPSmax} was recently introduced in an adhoc manner in~\cite{SPS} as a practical variant of the SPS method that caps the step size and named the \texttt{SPS}$_{\max}$ method. 
Lemma~\ref{lem:SPSmax} thus gives a novel formalization of \texttt{SPS}$_{\max}$ as a projection method, and as an extension of the \texttt{PA-I} method
proposed in~\cite{Crammer06} to nonlinear models. 
Analogously, the ALI-G method~\cite{ALI-G} can also be seen as an extension of the \texttt{PA-II} method proposed in~\cite{Crammer06}, as we show next.



\section{A Variational form of \emph{Almost} \texttt{ALI-G}}
\label{sec:ALI-G-almost}
Analogously to the previous section, we derive an iterative method for solving~\eqref{eq:slackobjL2} by sampling a single constraint, linearizing and projecting. 
\begin{align}
    w^{t+1}, s^{t+1} =& \argmin_{w\in\R^d} \tfrac{1}{2}\norm{w - w^t}^2 +\lambda s^2 \nonumber  \\
    &\; \; \mbox{s.t. } \ell_i(w^t) + \dotprod{\nabla \ell_i(w^t), w - w^t} \leq s \enspace.   
    \label{eq:ALIGproj} 
\end{align}

The key difference with \eqref{eq:SPSmaxproj} is that $s$ is squared. The closed-form solution is given as follows.
\begin{lemma} 
Let $\lambda  \geq 0$.  
The solution to~\eqref{eq:ALIGproj} is 
\begin{align} \nonumber  
    w^{t+1} & =  w^t -  \frac{\ell_i(w^t)}{ \lambda^{-1} + \norm{\nabla \ell_i(w^t)}^2} \nabla \ell_i(w^t) \enspace, \\
    s^{t+1} & =   \frac{ \ell_i(w^t)}{ 1 +\lambda \norm{\nabla \ell_i(w^t)}^2} \enspace. \label{eq:SPSdam}
\end{align}
We refer to this method as \texttt{SPS}$_{dam}$.
\end{lemma}
\begin{proof}
The solution follows by Lemma~\ref{lem:slackL2ineqconst}.
\end{proof}
We call~\eqref{eq:SPSdam} the \texttt{SPS}$_{dam}$ method because it is equivalent to \texttt{SPS} but with an additional dampening term $\lambda^{-1}$.
\texttt{SPS}$_{dam}$  is an extension of the \texttt{PA-II} method given in~\cite{Crammer06} to nonlinear models. It is also very similar to the \texttt{ALI-G} update in \eqref{eq:exact_ALIG} with $\varepsilon = \lambda^{-1}$.
The only difference being that \texttt{ALI-G}  caps the step size at a prescribed maximum $\eta$.

From the perspective of \texttt{SPS} methods, both slack variants help stabilize the \texttt{SPS} method~\eqref{eq:SPS}, and avoid a large step size by capping it as in~\eqref{eq:SPSmax}, or by dampening it in~\eqref{eq:SPSdam}. 
But both methods still need to carefully choose
$\lambda$ to work well in practice. 
Indeed, $\lambda$ here plays much the same role as a step size. If $\lambda$ is small, or converges to zero, both \texttt{SPS}$_{\max}$~\eqref{eq:SPSmax} and \texttt{SPS}$_{dam}$~\eqref{eq:SPSdam} halt. 
On the other extreme, if $\lambda$ is large, both methods can diverge.
In \cite{SPS} the authors propose a rule of thumb that chooses a different $\lambda$ at each iteration for~\eqref{eq:SPSmax}. As for~\eqref{eq:SPSdam}, in \cite{ALI-G} the authors cap the step size, which requires another tunable parameter.

In the next section, we propose new \texttt{SPS} variants that iteratively update the slack variable and then use it to adjust the step size.

\section{Regularizing the Slack}

Both variational forms~\eqref{eq:SPSmaxproj} and~\eqref{eq:ALIGproj} rely on linearizing around $w^t$ and $s^t$ the constraints~\eqref{eq:slackconstraint} that depend on both $w$ and $s$. These linearizations are only accurate when both $w$ and $s$ are close to $w^t$ and $s^t$. Yet, the objective in both~\eqref{eq:SPSmaxproj} and~\eqref{eq:ALIGproj}
only incentivises $w$ to be close to $w^t$, but allows $s$ to be far from $s^t.$ We find this to be a cause of instability in these methods.
Furthermore, because of this,  both in \eqref{eq:SPSmax} and \eqref{eq:SPSdam}, the slack variable $s^{t+1}$ plays no direct role in updating $w^{t+1}$. This is a missed opportunity, because the slack variable can be used to track the progress towards the solution and ultimately help halt the method.

\subsection{L1-Regularized Slack}

Suppose we have a current estimate of the slack variable $s^t$.
We can use our current estimate of the slack to add some regularization to the variational form of \texttt{SPS}$_{\max}$ in~\eqref{eq:SPSmaxproj} as follows
\begin{align}
    w^{t+1},  s^{t+1}& = \underset{w\in\R^d, s\geq 0}{\argmin} \tfrac{1}{2}\norm{w - w^t}^2  + \tfrac{1}{2} (s-s^t)^2+ \lambda s \nonumber \\ \,  &\quad  \mbox{s.t. } \;\ell_i(w^t) + \dotprod{\nabla \ell_i(w^t), w - w^t} \leq s \enspace.     \label{eq:SPSL1proj}
\end{align}
With this small change, the resulting method now makes use of the slack variable to set the step size, as can be seen in the following closed-form solution to~\eqref{eq:SPSL1proj}.
\begin{lemma} \label{lem:ProjL1pludL2}
The closed-form solution to~\eqref{eq:SPSL1proj} is given by~\eqref{eq:SPSL1} . We call this the \texttt{SPSL1} method.
\begin{align} 
    \gamma_t^{L1} &= \frac{(\ell_i(w^t)-s^t+\lambda)_+}{1+ \norm{\nabla \ell_i(w^t)}^2} \nonumber  \\
    \gamma_t &= \min\left\{ \gamma_t^{L1}, \; \frac{ \ell_i(w^t)}{\norm{\nabla \ell_i(w^t)}^2} \right\}\nonumber \\
    w^{t+1} &= w^t - \gamma_t \nabla \ell_i(w^t) \\
    s^{t+1} &= \left(s^t - \lambda + \gamma_t^{L1}\right)_+ +\label{eq:SPSL1}  
\end{align}

\end{lemma}
Missing proofs can  be found in the appendix.


What sets the \texttt{SPSL1} update~\eqref{eq:SPSL1} apart from the previous two slack methods in Section~\ref{sec:slack} is that the slack variable is now being updated iteratively and is used to update the step size adaptively. 
For instance,  starting from $s^0 = 0$ and if $\lambda$ is small enough\footnote{Specifically if $\lambda \leq \tfrac{\ell_i(w^t)}{\norm{\nabla \ell_i(w^t)}^2}$, which in the convex and $L_{\max}$--smooth setting occurs if $\lambda \leq \frac{1}{2L_{\max}}.$ This is exactly the restriction we apply to $\lambda$ in the forthcoming Theorem~\ref{theo:SPSL1smooth}. }, the  \texttt{SPSL1} takes the more aggressive \texttt{SPS} updates in $w^t$~\eqref{eq:SPS} since then $\gamma_t = \ell_i(w^t)/\norm{\nabla \ell_i(w^t)}^2$.
As $s^t$ starts to increase, the more conservative smaller step sizes $\gamma_t = \gamma_t^{L1}$ will be used. Finally if $s^t$ is too large, then $\gamma_t = \gamma_t^{L1} =0$ and the method will not update $w^t$, but will instead decrease $s^t$ since in this case $s^{t+1} = (s^t -\lambda)_+$ 
We can observe this behavior experimentally  in Figure~{\bf a)}.

\begin{minipage}[l]{0.5\textwidth}
Regularizing the slack variable can also smooth the trajectory of $s^t$, as we see 
in Figure~{\bf a)} where we show how $s^t$ for both \texttt{SPSL1} and \texttt{SPS}$_{\max}$ converge to the same value, but the slack variable of \texttt{SPSL1} takes a smoother route.
Finally, we can also see that  \texttt{SPS}$_{\max}$ is more sensitive to the tuning of $\lambda$ as compared to \texttt{SPSL1} by comparing the two update formulas~\eqref{eq:SPSmax} and~\eqref{eq:SPSL1}.
When $\lambda$ is small for \texttt{SPS}$_{\max}$ the method necessarily slows down, and in the extreme when $\lambda \rightarrow 0$ the \texttt{SPS}$_{\max}$ stops. 
In contrast, for $\lambda$ small, we have that $\gamma_t = \gamma_t^{L1}$ and \texttt{SPSL1} switches to a more conservative update but does not stop.
\end{minipage}
\begin{minipage}[r]{0.5\textwidth}
    \centering
     \includegraphics[width=3.0cm]{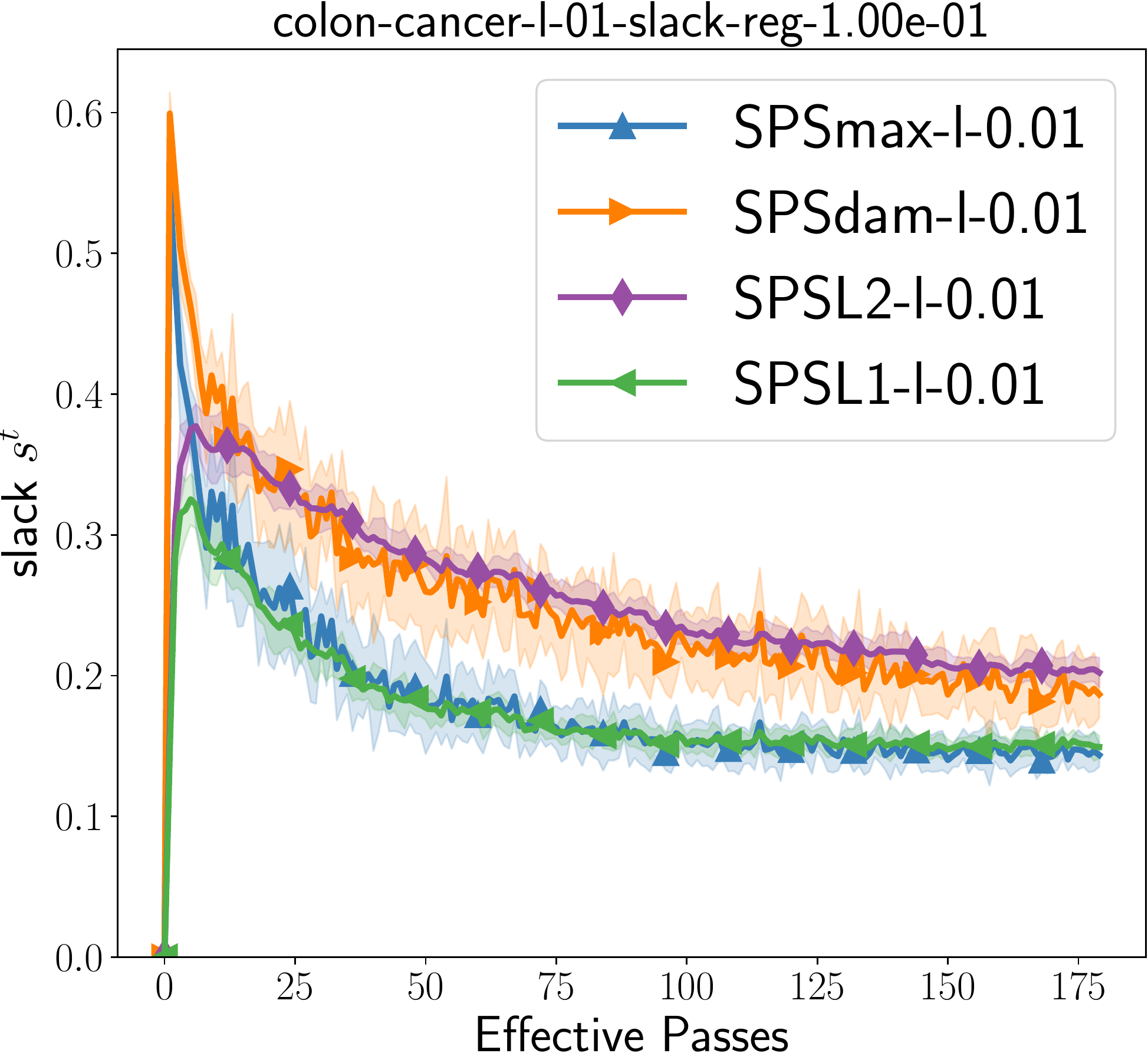} \hspace{0.2cm}
     \includegraphics[width=3.0cm]{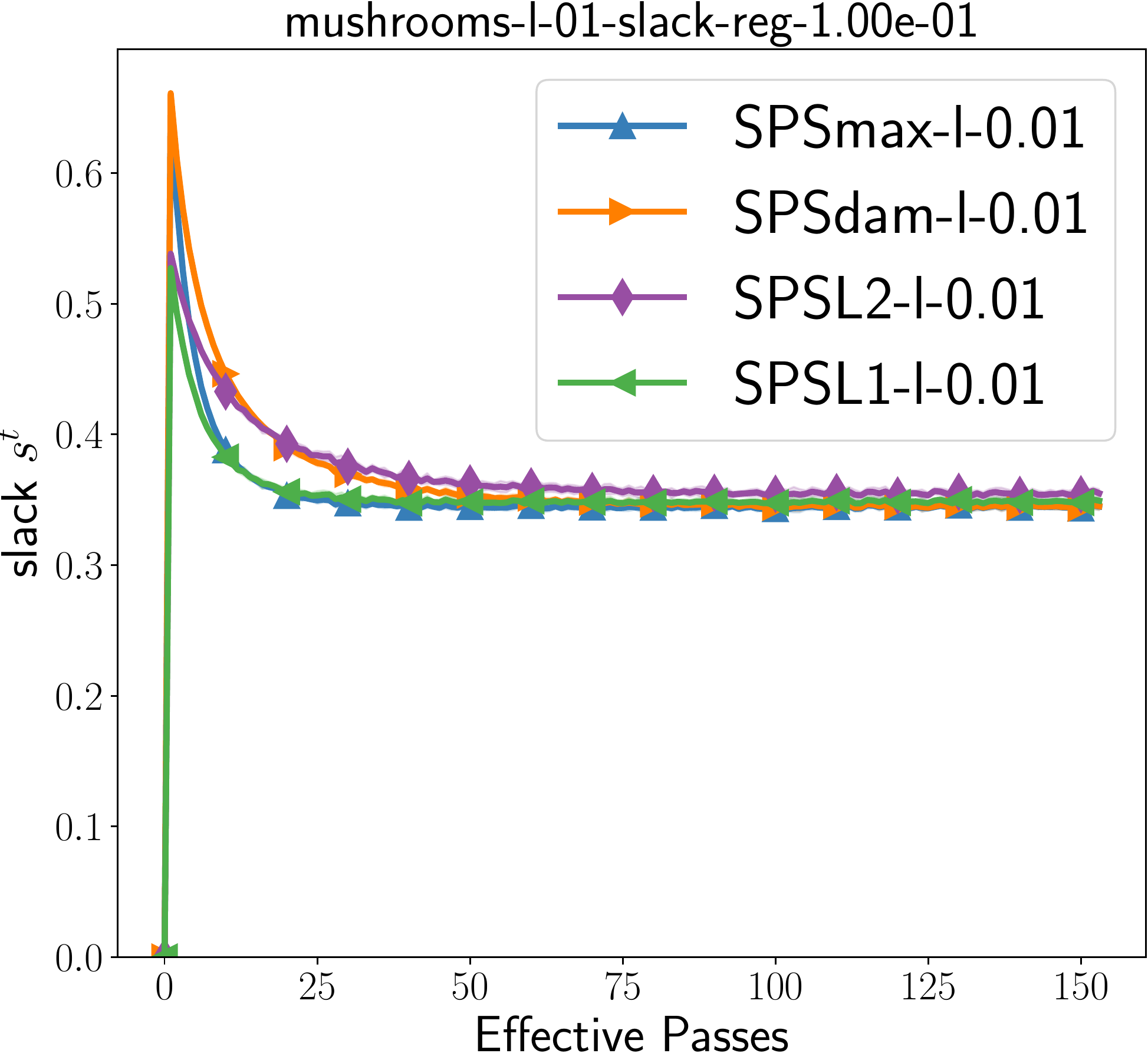}\\ 
    { {\bf Figure a).} Slack variable $s^t$ for \texttt{SPS}$_{\max}$, \texttt{SPS}$_{dam}$, \texttt{SPSL2} and \texttt{SPSL1} all with $\lambda =0.01$ on logistic regression. Left: \texttt{colon-cancer} dataset. Right:  \texttt{mushrooms} dataset.}
\end{minipage}

\subsection{L2-Regularized Slack}

We can also regularize the slack variable in the variational form of \texttt{SPS}$_{dam}$ in~\eqref{eq:ALIGproj} as follows 
\begin{align}
    \underset{w \in \R^d, s \in \R}{\min}&  \;\norm{w - w^t}^2 +(s-s^t)^2 + \lambda s^2,\, \mbox{  s.t.  } \ell_i(w^t) + \dotprod{\nabla \ell_i(w^t), w - w^t} \leq s \enspace,    \label{eq:SPSL2proj}
\end{align}
where again $\lambda >0$ is a regularization parameter.
\begin{lemma}\label{lem:SPSdam}
The closed-form update of~\eqref{eq:SPSL2proj} is given by
\begin{align}
    w^{t+1} &= w^t - \frac{(\ell_i(w^t) - \hat{\lambda} s^t)_+}{\norm{\nabla \ell_i(w^t)}^2 +\hat{\lambda}} \nabla \ell_i(w^t)\\
    s^{t+1} &= \hat{\lambda}\left(s^t + \frac{(\ell_i(w^t) - \hat{\lambda}s^t)_+}{\norm{\nabla \ell_i(w^t)}^2 +\hat{\lambda}}\right). \label{eq:SPSL2}
\end{align}
where  $\hat{\lambda} \eqdef 1/(1 + \lambda)$. 
We refer to~\eqref{eq:SPSL2}  as \texttt{SPSL2}.

\end{lemma}

This new \texttt{SPSL2} method now uses the slack variables $s^t$ to adjust the step size, which helps integrate information across iterations. 
As we start from $s^0 = 0$, the \texttt{SPSL2} update in $w$ is similar to the \texttt{ALI-G} method~\eqref{eq:exact_ALIG} (with $\eta$ very large).
As the slack variable increases,  the resulting step size  decreases, which in turn makes $s^t$ decrease, thus making an effective feedback between slack and step size. See Figure~{\bf a)} for a sample trajectory of the slack, and this ``increase then decrease'' phenomenon.

Furthermore, as $\lambda$ increases,  \texttt{SPSL2}  is closer  to the \texttt{SPS} update~\eqref{eq:SPS}. 
In the extreme where $\lambda \rightarrow \infty$ (thus $\hat{\lambda} \rightarrow 0$) we have that  \texttt{SPSL2} is equivalent to \texttt{SPS}.



\section{Convergence of \texttt{SPSL1}}
\label{sec:sonv_spspl1_main}

To prove convergence of \texttt{SPSL1}, we consider a small simplifying change to the method: we remove the explicit non-negativity constraint on the slack variable and add relaxation. That is, by dropping $s \geq 0$ in~\eqref{eq:SPSL1proj} we have 
\begin{align}
    w^{t+\tfrac{1}{2}}, s^{t+\tfrac{1}{2}} &= \! \underset{w\in\R^d, s \in \R}{ \argmin}  \tfrac{1}{2}\norm{w - w^t}^2 + \tfrac{1}{2}(s-s^t)^2 +\lambda s, \; 
    \mbox{s.t. }  \ell_i(w^t) + \dotprod{\nabla \ell_i(w^t), w - w^t} \leq s \enspace.  \nonumber \\
    [ w^{t+1},s^{t+1} ] &= (1-\gamma)[ w^{t}, s^t] + \gamma [w^{t+\tfrac{1}{2}}, s^{t+\tfrac{1}{2}}] 
    \label{eq:SPSL1noposproj-main}
\end{align}
where the second step is the relaxation and $\gamma \in (0,  1)$ is the relaxation parameter.
The closed-form solution of~\eqref{eq:SPSL1noposproj-main} can be found in detail in Section~\ref{sec:SPSL1nopos}.
This version of \texttt{SPSL1} in~\eqref{eq:SPSL1noposproj-main} is equivalent to~\eqref{eq:SPSL1proj} for non-negative losses that are convex and smooth, since  these assumptions guarantee the non-negativity of the slack variables $s^t$.
We use this property to establish the next theorem.
\begin{theorem} \label{theo:SPSL1smooth}
Fix $(w^\star,s^{\star}) \in \R^{d+1}$ and consider the $(w^t,s^t)$ iterates given by~\eqref{eq:SPSL1noposproj-main}.
Let $\bar{w}^T = \frac{1}{T} \sum_{t=0}^{T-1} w^t$, let $\ell_i$ be non-negative and let $s^0 \geq 0.$
If $\ell_i$ is convex, $L_{\max}$--smooth,  and $\lambda \leq 1/2L_{\max}$, then $s^t \geq 0 $  and
\begin{align}
  \E{ \ell (\bar{w}^T)} 
   & \leq \frac{1}{2 \gamma \lambda T}\frac{\norm{w^{0} -w^\star}^2  + (s^{0} -s^{\star})^2}{(1-\gamma )(1- \lambda L_{\max})}  +  \frac{1}{n} \sum_{i=1}^n\frac{(\ell_i(w^\star)+\lambda)^2 - \lambda^2  }{2\gamma \lambda (1-\gamma) (1- \lambda L_{\max}) } \enspace. \label{eq:SPSL1convsmooth}
\end{align}
Furthermore, if there exists a solution $w^\star$ to the interpolation equations~\eqref{eq:interpolate} and $\lambda = 1/2L_{\max}$ we have that
\begin{align} \label{eq:SPSL1convsmoothLmax}
 \E{ \ell (\bar{w}^T)} 
   & \leq \frac{2L_{\max}}{T}\frac{\norm{w^{0} -w^\star}^2  + (s^{0} -s^{\star})^2}{\gamma(1-\gamma )} \enspace. \quad
\end{align}
\end{theorem}
This convergence in~\eqref{eq:SPSL1convsmooth} gives a rate of $\cO(\frac{1}{T} + \mbox{constant})$, where this constant depends on $w^\star$ and the parameters. Under these same assumptions this is also the convergence rate of  \texttt{SPS}$_{\max}$ (Theorem 3.4~\cite{SPS}) and \texttt{SGD} (Theorem 4.1~\cite{SPS} and Corollary 4.4~\cite{KSLGR2020unifiedsgm}).
 What is different is this constant  term on the right hand side of~\eqref{eq:SPSL1convsmooth}. Though this term is different, under interpolation we have that this constant term is zero giving~\eqref{eq:SPSL1convsmoothLmax}, which excluding the factor of $\frac{1}{\gamma(1-\gamma)}$,  is identical to the convergence of both \texttt{SPS}$_{\max}$ (Theorem 3.4~\cite{SPS}) and \texttt{SGD} (Theorem 4.1~\cite{SPS} and Corollary 4.4~\cite{KSLGR2020unifiedsgm}).


Next we prove  convergence for nonsmooth functions.
\begin{theorem}\label{theo:SPSL1lip}
Let $w^* \in \argmin_{w \in \R^d} \max_{i=1,\ldots, n} \ell_i(w)$ and $s^* = \ell_i(w^*) +  \lambda$.
Consider the $(w^t,s^t)$ iterates given by~\eqref{eq:SPSL1noposproj-main}. 
Let $\bar{w}^T = \frac{1}{T} \sum_{t=0}^{T-1} w^t$ and let $\ell_i$ be non-negative. 
    
      If  $\ell_i$ is $\sigma$--Lipschitz, convex  and $s^0 \geq -\lambda \sigma^2$, then 
\begin{align}
    &\E{\frac{1 }{n} \sum_{i=1}^n \ell_i(\bar{w}^T) } 
    \leq \frac{\norm{w^{0} -w^*}^2 +(s^0-s^*)^2 }{2\gamma (1-\gamma) \lambda T} + \frac{\lambda}{2} \frac{1  + \sigma^2 (1 + \gamma)}{ 1-\gamma }   
    + \frac{\max_{i=1,\ldots, n} \ell_i(w^\star) }{1-\gamma} \; . \label{eq:SPSL1convLip}
\end{align} 
\end{theorem}
In Theorem~\ref{theo:SPSL1lip} we can see that $\lambda$ trades off speed of convergence and the \emph{radius of convergence}.
In some ways, it plays the same role as constant step size for \texttt{SGD}.
On the one hand, when $\lambda$ goes to infinity \texttt{SPSL1} escapes faster from the initial conditions as the convergence rate is ruled by $\cO (1/\lambda T)$.
On the other hand, when $\lambda$ goes to zero, the first term of~\eqref{eq:SPSL1convLip} goes to infinity, whereas the second term, the constant error one, goes to zero. 
Furthermore, the radius of convergence also depends on the remaining term $\max_{i=1,\ldots, n} \ell_i(w^\star)$, which vanishes in the interpolated setting as it is close to zero or can be controlled for $\gamma$ small.

\section{Convergence of \texttt{SPSL2}}
\label{sec:sonv_spspl2_main}



To prove convergence of \texttt{SPSL2} we also introduce a relaxation parameter $\gamma \in [0, 1]$ and a relaxation step analogously to the relaxation we introduced in~\eqref{eq:SPSL1noposproj-main}. The details of this change can be found in Section~\ref{sec:spsl2_convergence_theory} and specifically formula~\eqref{eq:SPSL2SGD}.
With this one additional change, we can  prove the convergence of \texttt{SPSL2} in terms of the average of the square of the losses, as we show in the next theorem.

\begin{theorem} 
\label{theo:SPSDAMfmax_main}
Fix $(w^\star,s^{\star}) \in \R^{d+1}$ and consider the $(w^t,s^t)$ iterates given by~\eqref{eq:SPSL2} with a relaxation parameter $\gamma \in (0,  1)$.
Let $\bar{w}^T = \tfrac{1}{T} \sum_{t=0}^{T-1} w^t$ and let $\ell_i$ be non-negative.

If $\ell_i$ is  $\sigma$--Lipschitz for every $i\in \{1,\ldots, n\}$ then
\begin{align*} 
    \tfrac{1}{n} \sum_{i=1}^n\E{\ell_i(\bar{w}^T)^2}  
    & \leq \frac{1}{C (\hat{\lambda})}\frac{\norm{w^{0} -w^\star}^2+(s^0 -s^{\star})^2}{\gamma(1-\gamma)  T} +  \frac{1}{C (\hat{\lambda})} \frac{1-\hat{\lambda}}{\hat{\lambda}^2 (1-\gamma)} \max_{i=1,\ldots, n} \ell_i(w^\star)^2 \enspace. 
\end{align*}
\begin{equation}
    \label{eq:Capp_main}
    \mbox{where     }  \bar{w}^T = \sum_{t=0}^{T-1} w^k/T, \; \hat{\lambda} = 1/(1+\lambda)\mbox{   and   } 
    C (\hat{\lambda}) \eqdef \frac{1 - \hat{\lambda}}{\hat{\lambda}^2 + (1 - \hat{\lambda}) \left(\hat{\lambda} + L_{\max}^2 \right)} \enspace.
\end{equation}
\end{theorem}

This last result in Theorem~\eqref{theo:SPSDAMfmax_main} shows that \texttt{SPSL2} converges to a minimizer of $\tfrac{1}{n} \sum_{i=0}^n \ell_i(w)^2$ up to an additive constant error proportional to $\max_{i=1,\ldots, n} \ell_i(w^\star)^2.$
This is perhaps not so surprising given that \texttt{SPSL2} is indeed solving the max squared loss~\eqref{eq:slackobjL2}, which in turn is an over-approximation of the average squared loss.
As in Theorems~\ref{theo:SPSL1lip} and~\ref{theo:SPSL1smooth}, the slack parameter $\lambda$ (through $\hat{\lambda}$) trades off the speed of convergence and the constant error term. Furthermore, under interpolation we have again that the constant error term is zero.



In the appendix in Theorem~\ref{theo:SPSDAMregret} we also provide an extension of the regret analysis for \texttt{PA} methods, given in Theorem 5 in~\cite{Crammer06}, from hinge loss over linear models to a more general class of nonlinear convex models.

\section{Numerical Experiments}

We perform several numerical experiments for validating the usefulness of the proposed SPS variants on logistic regression and DNNs. Specifically, we compare: \texttt{ADAM} \cite{ADAM};
 \texttt{ALI-G}, i.e., update \eqref{eq:exact_ALIG};
\texttt{SPS}$_{\max}$, i.e., update \eqref{eq:SPSmaxproj};
\texttt{SPS}$_{dam}$ (proposed), i.e., update \eqref{eq:SPSdam};
\texttt{SPSL1} (proposed), i.e., update \eqref{eq:SPSL1};
\texttt{SPSL2} (proposed), i.e., update \eqref{eq:SPSL2}.
Apart from \texttt{ADAM}, all methods are SPS variants.



\begin{figure*}
\centering\includegraphics[width=\linewidth]{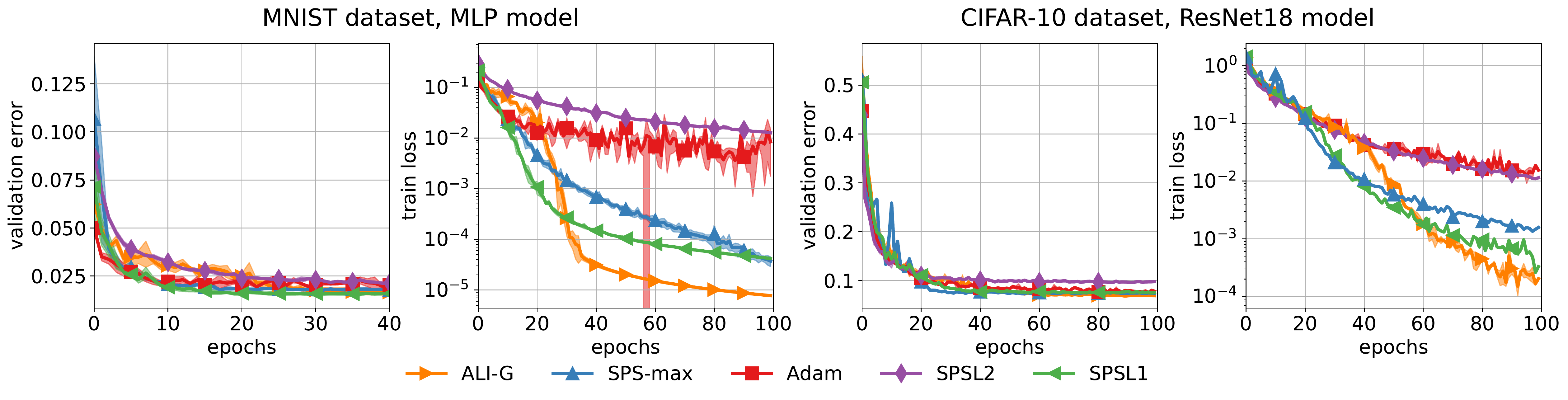}
\centering\includegraphics[width=0.5\linewidth]{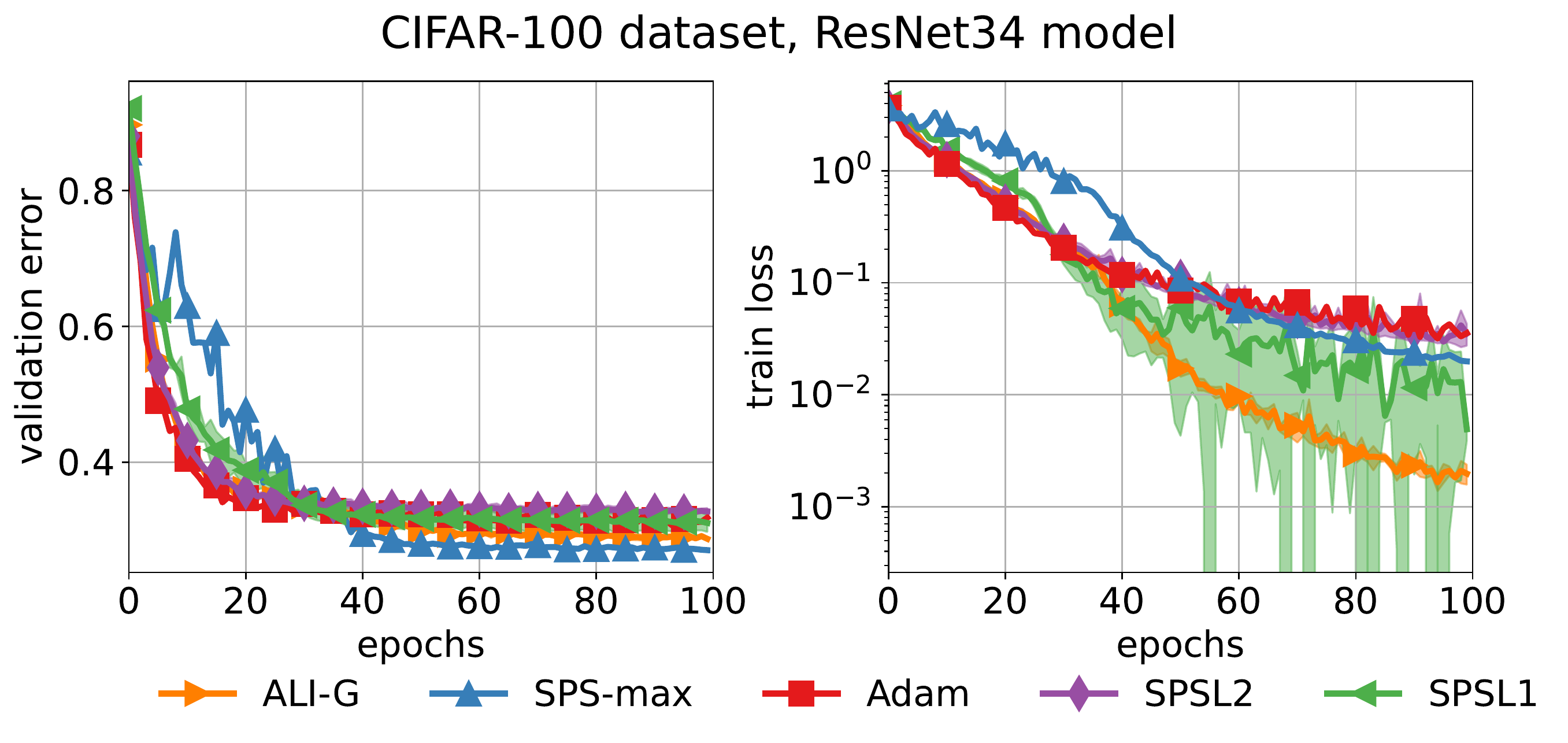}
\includegraphics[width=0.23\linewidth]{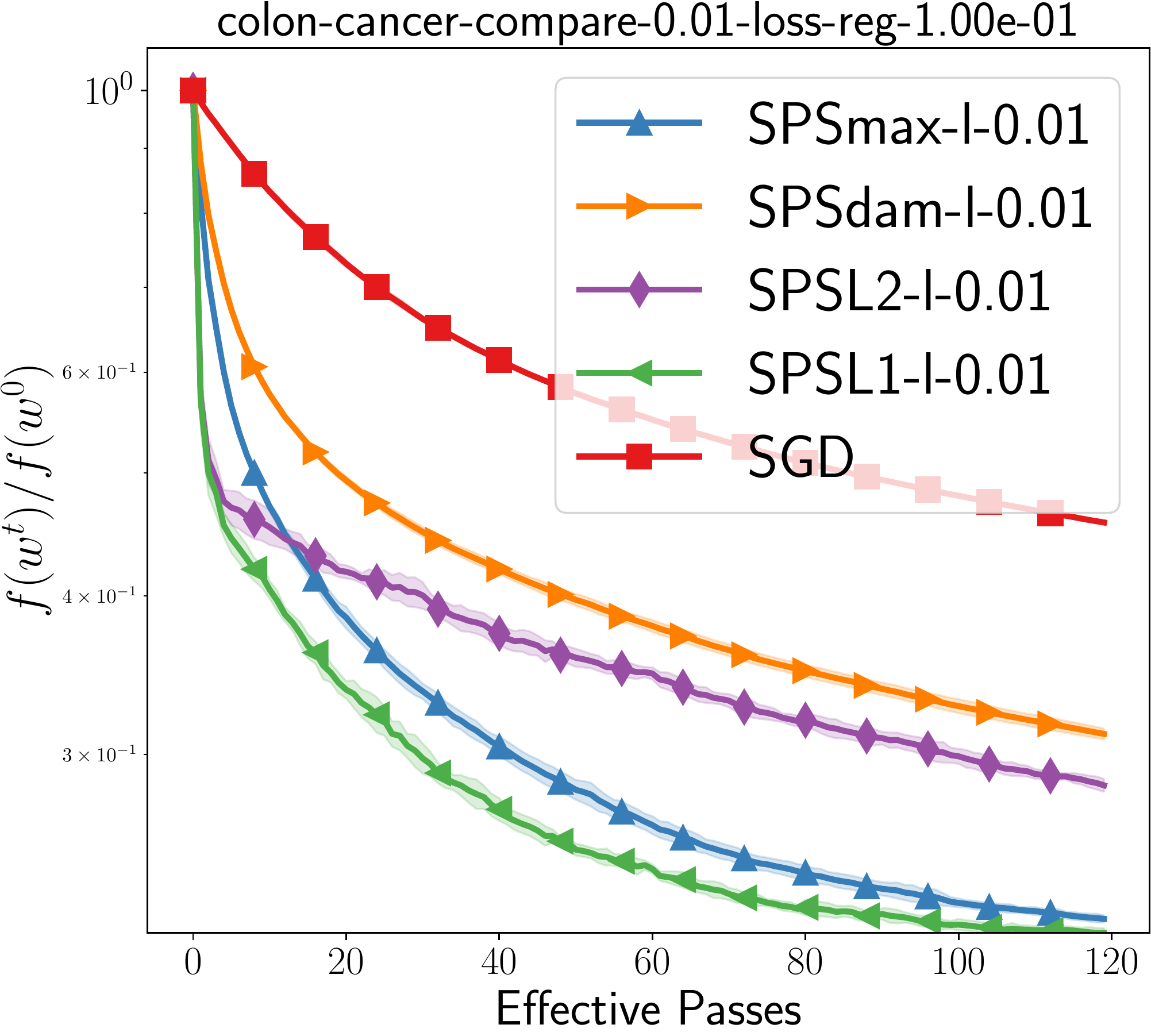} 
 \includegraphics[width=0.23\linewidth]{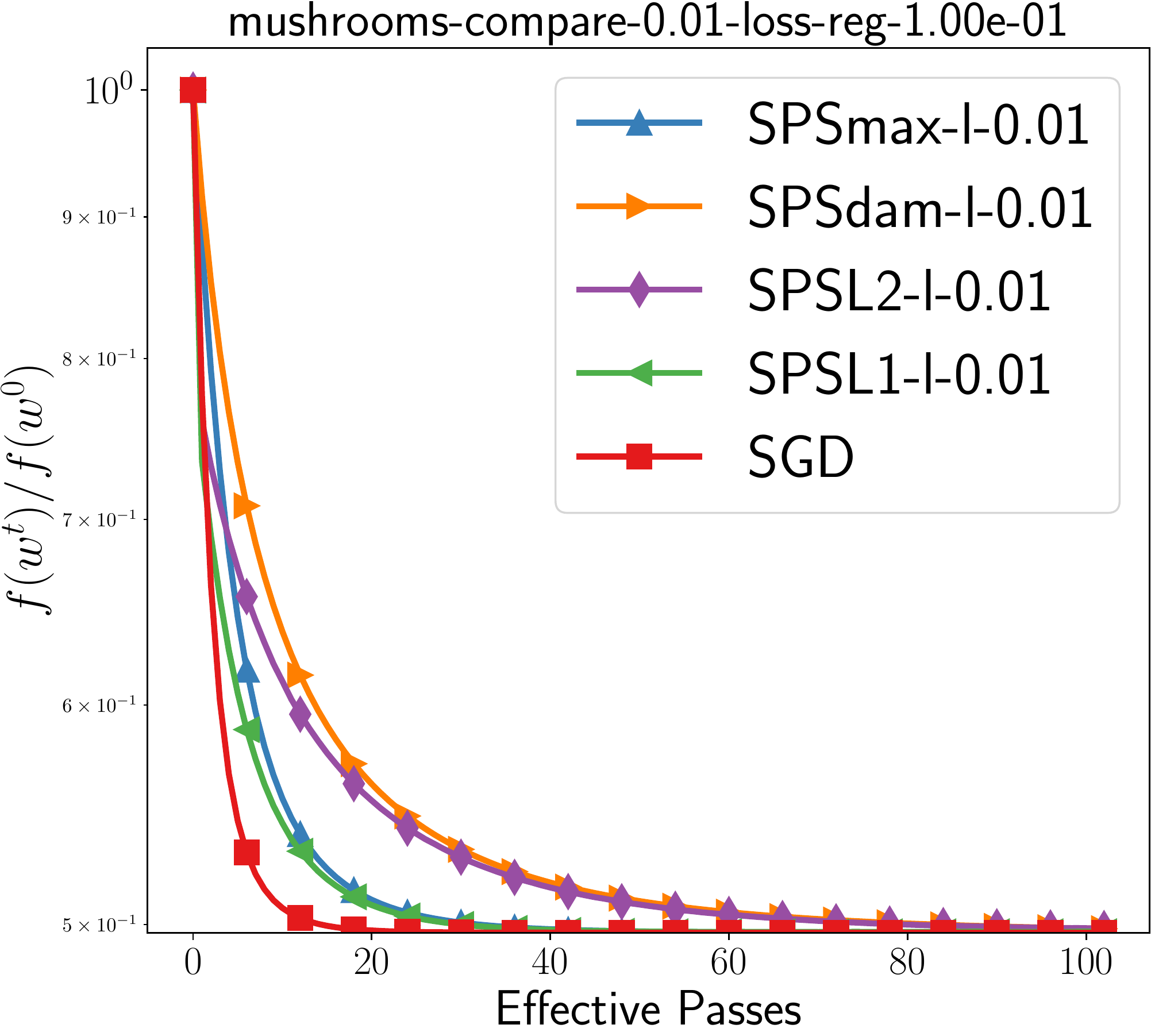}
\caption{Comparison of the proposed and related methods on different computer vision problems in terms of both error in a validation set and loss in the train set. With the same set of parameters, the new slack variants are competitive throughout different datasets and models. In terms of validation error, \texttt{SPSL1} achieves the best result in the MNIST and arrives at a close second on the CIFAR-10.}
\label{figure:deep_results}
\end{figure*}

We compared the aforementioned methods on 4 different deep learning vision problems and one convex logistic regression problem.
The datasets for the vision problem we used are the MNIST dataset (60,000  $28\times28$ images, 10 classes) \cite{lecun2010mnist}, the CIFAR10, and \mbox{CIFAR-100} datasets (60,000 32$\times$32 color images in 10 and 100 classes respectively) \cite{krizhevsky2009learning}  We use 3 different models that are known to work well on these datasets: a multi-layer perceptron (MLP) with one hidden layer on the MNIST dataset, a ResNet18 \cite{he2016deep} on CIFAR-10, and a ResNet34 on the CIFAR-100 dataset.
For our logistic regression problem we used the  colon-cancer (    $d =2001$ and  $n =62$)~\cite{uci} data set.

Because of the importance of momentum in speeding up \texttt{ALI-G} and \texttt{SPS}$_{\max}$ for DNNs~\cite{ALI-G-vs-SPS},  we also use momentum in our implementations of \texttt{SPSL1} and \texttt{SPSL2}.
Formally, we replace $w_{t+1} = w_t - u_t$, where $u_t$ corresponds to the method-specific update, by 
$v_{t+1} = \beta v_t - u_t$ and 
$w_{t+1} = w_t + v_{t+1}$.
We fix the momentum parameter to $\beta =0.5$ in the following experiments.

All methods were implemented in PyTorch. For the \texttt{SPS}$_{\max}$ method, we used the authors' implementation,\footnote{https://github.com/IssamLaradji/sps} which derives from the analyzed method in that it uses a ``step size smoothing'' technique that prevents the step size from increasing too rapidly. The Adam method uses the default PyTorch parameters (lr=0.001, betas=(0.9, 0.999)).
All methods use batch-size of 128 and relaxation parameter $\gamma = 2$. We report the validation error and train loss averaged over 3 runs.


Our results, shown in Figure \ref{figure:deep_results}, confirm that both $\texttt{SPSL1}$ and $\texttt{SPSL2}$ are highly competitive. In terms of validation error, \texttt{SPSL1} achieves the best result in the MNIST dataset (0.0159 vs 0.0165 for the second-best, \texttt{ALI-G}), arrives second on the CIFAR-10 dataset (0.075 vs 0.070 for \texttt{ALI-G}), and third on the CIFAR-100 (0.31 vs 0.27 for \texttt{SPS}$_{\max}$).

\section{Conclusion}

We studied in this paper SPS (SGD with Polyak step size), a family of stochastic algorithms with step size adaptive to the gradient norm and loss value at the sampled points.
Our first contribution was a novel connection with Passive-Aggressive methods, leading to a variational perspective of not only vanilla \texttt{SPS} but also recently-proposed variants \texttt{SPS}$_{\max}$ and \texttt{ALI-G}. Based on this viewpoint, our second contribution was two new SPS variants, \texttt{SPSL1} and \texttt{SPSL2}, that regularize the slack variable, for which we provide a convergence-rate analysis. 
Experiments confirmed that \texttt{SPSL1} and \texttt{SPSL2}, are competitive, both on logistic regression and DNN tasks. 

\bibliographystyle{abbrv}
\bibliography{references}

\begin{thebibliography}{10}

\bibitem{coloncancer}
U.~Alon, N.~Barkai, D.~A. Notterman, K.~Gish, S.~Ybarra, D.~Mack, and A.~J.
  Levine.
\newblock Broad patterns of gene expression revealed by clustering analysis of
  tumor and normal colon tissues probed by oligonucleotide arrays.
\newblock {\em Proceedings of the National Academy of Sciences}, 1999.

\bibitem{asi2019importance}
H.~Asi and J.~C. Duchi.
\newblock The importance of better models in stochastic optimization.
\newblock {\em Proceedings of the National Academy of Sciences}, 2019.

\bibitem{balles2018dissecting}
L.~Balles and P.~Hennig.
\newblock Dissecting adam: The sign, magnitude and variance of stochastic
  gradients.
\newblock In {\em Proceedings of the 35th International Conference on Machine
  Learning (ICML)}, 2018.

\bibitem{belkin2019does}
M.~Belkin, A.~Rakhlin, and A.~B. Tsybakov.
\newblock Does data interpolation contradict statistical optimality?
\newblock In {\em The 22nd International Conference on Artificial Intelligence
  and Statistics}, pages 1611--1619. PMLR, 2019.

\bibitem{ALI-G}
L.~Berrada, A.~Zisserman, and M.~P. Kumar.
\newblock Training neural networks for and by interpolation.
\newblock In {\em Proceedings of the 37th International Conference on Machine
  Learning}, volume 119 of {\em Proceedings of Machine Learning Research},
  pages 799--809, 13--18 Jul 2020.

\bibitem{ALI-G-vs-SPS}
L.~Berrada, A.~Zisserman, and M.~P. Kumar.
\newblock Comment on stochastic polyak step-size: Performance of {ALI-G}.
\newblock {\em CoRR}, abs/2105.10011, 2021.

\bibitem{blitzer_2006}
J.~Blitzer, R.~McDonald, and F.~Pereira.
\newblock Domain adaptation with structural correspondence learning.
\newblock In {\em Proceedings of the 2006 conference on empirical methods in
  natural language processing}, pages 120--128, 2006.

\bibitem{BlondelKU14}
M.~Blondel, Y.~Kubo, and N.~Ueda.
\newblock Online passive-aggressive algorithms for non-negative matrix
  factorization and completion.
\newblock In {\em Proceedings of the Seventeenth International Conference on
  Artificial Intelligence and Statistics, {AISTATS}}, volume~33 of {\em {JMLR}
  Workshop and Conference Proceedings}, pages 96--104, 2014.

\bibitem{Chada-accel-model-2021}
K.~N. Chadha, G.~Cheng, and J.~C. Duchi.
\newblock Accelerated, optimal, and parallel: Some results on model-based
  stochastic optimization.
\newblock {\em CoRR}, abs/2101.02696, 2021.

\bibitem{chechik_2010}
G.~Chechik, V.~Sharma, U.~Shalit, and S.~Bengio.
\newblock Large scale online learning of image similarity through ranking.
\newblock {\em Journal of Machine Learning Research}, 11(3), 2010.

\bibitem{Crammer06}
K.~Crammer, O.~Dekel, J.~Keshet, S.~Shalev-Shwartz, and Y.~Singer.
\newblock Online passive-aggressive algorithms.
\newblock {\em J. Mach. Learn. Res.}, 7:551--585, 2006.

\bibitem{uci}
D.~Dua and C.~Graff.
\newblock {UCI} machine learning repository, 2017.

\bibitem{ADAGRAD}
J.~Duchi, E.~Hazan, and Y.~Singer.
\newblock Adaptive subgradient methods for online learning and stochastic
  optimization.
\newblock {\em J. Mach. Learn. Res.}, 12:2121--2159, July 2011.

\bibitem{gilbert2021fragments}
J.~C. Gilbert.
\newblock Fragments d'optimisation diff{\'e}rentiable-th{\'e}ories et
  algorithmes.
\newblock 2021.

\bibitem{TASPS}
R.~M. Gower, A.~Defazio, and M.~Rabbat.
\newblock Stochastic polyak stepsize with a moving target.
\newblock {\em arXiv:2106.11851}, 2021.

\bibitem{gower2019sgd}
R.~M. Gower, N.~Loizou, X.~Qian, A.~Sailanbayev, E.~Shulgin, and
  P.~Richt{\'a}rik.
\newblock Sgd: General analysis and improved rates.
\newblock In {\em International Conference on Machine Learning}, pages
  5200--5209, 2019.

\bibitem{SGDstruct}
R.~M. Gower, O.~Sebbouh, and N.~Loizou.
\newblock Sgd for structured nonconvex functions: Learning rates, minibatching
  and interpolation.
\newblock {\em arXiv:2006.10311}, 2020.

\bibitem{he2016deep}
K.~He, X.~Zhang, S.~Ren, and J.~Sun.
\newblock Deep residual learning for image recognition.
\newblock In {\em Proceedings of the IEEE conference on computer vision and
  pattern recognition}, pages 770--778, 2016.

\bibitem{Herbster01}
M.~Herbster.
\newblock Learning additive models online with fast evaluating kernels.
\newblock In {\em 14th Annual Conference on Computational Learning Theory,
  {COLT}}, volume 2111 of {\em Lecture Notes in Artificial Intelligence}, pages
  444--460. Springer, 2001.

\bibitem{KSLGR2020unifiedsgm}
A.~Khaled, O.~Sebbouh, N.~Loizou, R.~M. Gower, and P.~Richt{\'a}rik.
\newblock {Unified Analysis of Stochastic Gradient Methods for Composite Convex
  and Smooth Optimization}.
\newblock {\em arXiv preprint arXiv:2006.11573}, 2020.

\bibitem{ADAM}
D.~P. Kingma and J.~Ba.
\newblock Adam: {A} method for stochastic optimization.
\newblock In {\em 3rd International Conference on Learning Representations,
  {ICLR} 2015}, 2015.

\bibitem{krizhevsky2009learning}
A.~Krizhevsky, G.~Hinton, et~al.
\newblock Learning multiple layers of features from tiny images.
\newblock 2009.

\bibitem{lecun2010mnist}
Y.~LeCun, C.~Cortes, and C.~Burges.
\newblock Mnist handwritten digit database. at\&t labs, 2010.

\bibitem{liang2020just}
T.~Liang and A.~Rakhlin.
\newblock Just interpolate: Kernel “ridgeless” regression can generalize.
\newblock {\em The Annals of Statistics}, 48(3):1329--1347, 2020.

\bibitem{SPS}
N.~Loizou, S.~Vaswani, I.~Laradji, and S.~Lacoste-Julien.
\newblock Stochastic polyak step-size for sgd: An adaptive learning rate for
  fast convergence.
\newblock {\em arXiv:2002.10542}, 2020.

\bibitem{MaBB18}
S.~Ma, R.~Bassily, and M.~Belkin.
\newblock The power of interpolation: Understanding the effectiveness of {SGD}
  in modern over-parametrized learning.
\newblock In {\em {ICML}}, {JMLR} Workshop and Conference Proceedings, 2018.

\bibitem{mcdonald_2005}
R.~McDonald, F.~Pereira, K.~Ribarov, and J.~Hajic.
\newblock Non-projective dependency parsing using spanning tree algorithms.
\newblock In {\em Proceedings of human language technology conference and
  conference on empirical methods in natural language processing}, pages
  523--530, 2005.

\bibitem{nesterov2013introductory}
Y.~Nesterov.
\newblock {\em Introductory Lectures on Convex Optimization: A Basic Course},
  volume~87.
\newblock Springer Science \& Business Media, 2013.

\bibitem{Orabona:2019}
F.~Orabona.
\newblock A modern introduction to online learning.
\newblock {\em CoRR}, abs/1912.13213, 2019.

\bibitem{AMSgrad}
S.~J. Reddi, S.~Kale, and S.~Kumar.
\newblock On the convergence of adam and beyond.
\newblock In {\em International Conference on Learning Representations}, 2018.

\bibitem{Saul2021}
L.~K. Saul.
\newblock An online passive-aggressive algorithm for difference-of-squares
  classification.
\newblock {\em Neurips}, 2021.

\bibitem{cod-rna}
A.~V. Uzilov, J.~M. Keegan, and D.~H. Mathews.
\newblock Detection of non-coding rnas on the basis of predicted secondary
  structure formation free energy change.
\newblock {\em {BMC} Bioinform.}, 7:173, 2006.

\bibitem{vaswani2019painless}
S.~Vaswani, A.~Mishkin, I.~Laradji, M.~Schmidt, G.~Gidel, and
  S.~Lacoste-Julien.
\newblock Painless stochastic gradient: Interpolation, line-search, and
  convergence rates.
\newblock {\em arXiv preprint arXiv:1905.09997}, 2019.

\bibitem{SNR}
R.~Yuan, A.~Lazaric, and R.~M. Gower.
\newblock Sketched newton-raphson.
\newblock {\em arXiv:2006.12120, ICML workshop ``Beyond first order methods in
  ML systems''}, 2020.

\bibitem{Zaheer2018}
M.~Zaheer, S.~Reddi, D.~Sachan, S.~Kale, and S.~Kumar.
\newblock Adaptive methods for nonconvex optimization.
\newblock In {\em Advances in Neural Information Processing Systems 31}, pages
  9793--9803, 2018.

\bibitem{ZhangBHRV17}
C.~Zhang, S.~Bengio, M.~Hardt, B.~Recht, and O.~Vinyals.
\newblock Understanding deep learning requires rethinking generalization.
\newblock In {\em 5th International Conference on Learning Representations,
  {ICLR} 2017}, 2017.

\end{thebibliography}

\tableofcontents
\appendix

\section{Auxiliary Lemmas}

\begin{lemma}\label{lem:pos_sum_boud}
For every $a,b\in \R$ and $\beta >0$ we have that
\begin{equation}\label{eq:pos_sum_boud}
    (a+b)_+^2 \leq (1+\beta)a_+^2 + \left( 1+ \frac{1}{\beta}\right)b_+^2 \enspace.
\end{equation} 
\end{lemma}
\begin{proof}
Using that $(a+b)_+ \leq a_+ + b_+$ and expanding the squares we have that
\begin{equation} 
    \label{eq:tempposumbasod}
    (a+b)_+^2 \leq  (a_+ + b_+)^2 \; = \; a_+^2 +2a_+b_++b_+^2 \enspace.
\end{equation}
Now using that
\[\left(\sqrt{\beta}a_+ - \frac{1}{\sqrt{\beta}}b_+\right)^2 \geq 0 \; \implies \; 2a_+b_+ \leq \beta a_+^2 + \frac{1}{\beta} b_+^2 \enspace,\]
in~\eqref{eq:tempposumbasod} gives the result.
\end{proof}

\begin{lemma}[Convex and smooth inequality for non-negative function] \label{lem:convsmoothinter}
Let $\ell_i$ be $L_{\max}$--smooth, convex and suppose that $\ell_i$ is non-negative.
It follows that
\begin{equation}
    \label{eq:after_non-negativity_and_interpolation_of_li}
    \ell_i(w^t) \geq \frac{1}{2L_{\max}} \norm{\nabla \ell_i (w^t)}^2 \enspace.
\end{equation}
\end{lemma}
\begin{proof}
We start from inequality (2.1.7) of Theorem 2.1.5 in~\cite{nesterov2013introductory} for $\ell_i$ which is convex and $L_{\max}$--smooth which states that
\begin{equation}
    \label{eq:nesterov_ineq_217}
    \ell_i(y) \geq \ell_i(x) + \dotprod{\nabla \ell_i (x), y-x} + \frac{1}{2L_{\max}} \norm{\nabla \ell_i (x) - \nabla \ell_i (y)}^2 \enspace.
\end{equation}

Let $w^\star_i \in \argmin_w \ell_i(w)$. Setting  $x \leftarrow w_i^*$ and $y \leftarrow w^t$ gives
\begin{equation}
    \label{eq:temspjso9hs8sh}
    \ell_i(w^t) - \ell_i(w^\star_i)  \geq  \frac{1}{2L_{\max}} \norm{ \nabla \ell_i (w^t)}^2 \enspace.
\end{equation}
Finally the result follows since $\ell_i(w^\star) \geq 0 $ and thus  $\ell_i(w^t) \geq \ell_i(w^t) - \ell_i(w^\star_i) $.
\end{proof}

\section{L2 Projection Lemmas }

Here we establish several projection lemmas used throughout the paper. Recall that
we denote positive part function, otherwise known as the RELU function by 
\[(x)_+ =\begin{cases}  x & \mbox{ if } x \geq 0 \\ 0 & \mbox{otherwise} \end{cases}.\] 

\begin{lemma}\label{lem:least-norm-sol}
Let $w^0\in \R^d$ and let $\mA \in \R^{ d \times b}$ and $c \in \R^b.$
The solution to
\begin{align}
    \label{eq:newtraph}
    w' =& \argmin_{w\in\R^d} \norm{w - w^0}^2 \nonumber \\
    &\, \mbox{s.t. } \mA^\top (w - w^0) + c =0 \enspace.
\end{align}

is given by
\begin{equation}
    w' = w^0 - \mA (\mA^\top  \mA)^\dagger c = w^0 - (\mA^\top)^\dagger c \enspace.
\end{equation}
\end{lemma}
\begin{proof}
Substituting $\hat{w} = w-w^0$ gives 
\[\argmin_{\hat{w} \in \R^d} \norm{\hat{w}}^2 \quad \mbox{s.t. } \mA^\top \hat{w} =-c \enspace. \]
This is a least norm problem, for which the solution is given by the pseudo inverse
\[\hat{w} = - (\mA^\top)^\dagger c = - \mA (\mA^\top  \mA)^\dagger c \enspace. \]
Substituting back $w'-w^0 = \hat{w} $ gives the solution.
\end{proof}

\begin{lemma}[Least norm Unidimensional Inequality solution]\label{lem:least-norm-sol-ineq}
Let $w^0 \in \R^d$, $a \in \R^d \setminus \{0\}$ and $c \in \R.$
The solution to
\begin{align}
    w' =& \argmin_{w\in\R^d} \norm{w -v}^2 \nonumber \\
    &\, \mbox{s.t. } a^\top (w - w^0) + c \geq 0 \enspace.
\label{eq:newtraphineq}
\end{align}
is given by
\begin{equation}\label{eq:least-norm-sol-ineq}
    w' =  v + \frac{\left( a^\top (w^0 -v)-c\right)_+}{ \norm{a}^2} a \enspace.
\end{equation}
\end{lemma}
\begin{proof}
The problem is an L2 projection onto a halfspace. 
The solution depends if the projected vector $v$ is in this halfspace. 
That is, if $v$ satisfies the linear inequality constraint, which holds if 
\[a^\top (v - w^0) + c \geq 0 \enspace,\] 
then the solution is simply $w' = v$.

Alternatively, if the above does not hold, that is
\begin{equation}
    a^\top (v - w^0) + c < 0 \enspace,
\end{equation}
then we need to project $v$ onto the boundary of the halfspace, that is, onto 
\[ \{ w \in \R^d \, | \, a^\top (w - w^0) + c = 0 \} \enspace. \]
In which case the solution, substituting $\hat{w} = w-v$ into~\eqref{eq:newtraphineq} and imposing the equality constraint gives
\begin{align}
    \label{eq:tmeorps8hs8}
    w' =& \argmin_{w\in\R^d} \norm{\hat{w}}^2 \nonumber \\
    &\, \mbox{s.t. } a^\top \hat{w} + a^\top (v - w^0) + c = 0 \enspace.
\end{align}
This is now a least norm problem, to which the solution 
is (following Lemma~\ref{lem:least-norm-sol} with $w^0 \leftarrow 0$, $\mA \leftarrow a$ and $c \leftarrow a^\top (v - w^0) + c$) given by
\begin{align}
\hat{w} &= - (a^\top)^\dagger \left(a^\top (v - w^0) + c \right) \nonumber \\
&= \frac{\left(a^\top (w^0 -v) -c\right)}{\norm{a}^2} a \enspace.
\end{align}
Substituting back $\hat{w} = w-v$ and rearranging gives
\begin{align}
w &= v + a  \frac{\left(a^\top (w^0 -v) -c\right)}{\norm{a}^2} \enspace.
\end{align}

Putting these two cases together we have~\eqref{eq:least-norm-sol-ineq}.
\end{proof}

\section{L2 Projection Lemmas with Slack Variable}

\begin{lemma}[L2 Equality Constraints] \label{lem:slackL2eqconst}
Let $\delta >0$, $w^0\in \R^d$ and let $\mA \in \R^{ d \times b}$ and $s^0, c \in \R^b.$
The closed-form solution to 
\begin{align}
    \label{eq:slackL2eqconstproj}
    w', s' =& \argmin_{w\in\R^d, \, s \in \R^b } \norm{w - w^0}^2 + \delta \norm{s-s^0}^2 \nonumber \\
    &\, \mbox{s.t. } \mA^\top (w-w^0) +c = s \enspace,
\end{align}
is given by 
\begin{align} \label{eq:slackL2eqconstprojsolw}
w' & =  w^0 - \delta \mA (\delta \mA^\top \mA + \mI)^{-1}  (c-s^0) \enspace, \\
s' & = s^0 + (\delta \mA^\top \mA + \mI)^{-1}   (c-s^0) \enspace. \label{eq:slackL2eqconstprojsolb}
\end{align}
\end{lemma}
\begin{proof}
Let $\hat{s} \eqdef \sqrt{\delta} (s-s^0) $ and $\hat{w} \eqdef w-w^0.$ Substituting these new variables  and re-arranging we have that
\begin{align}
    \label{eq:slackL2eqconstprojhats}
&\argmin_{w\in\R^d} \norm{\hat{w }}^2 +  \hat{s}^2 \nonumber \\
    &\, \mbox{s.t. }  \mA^\top \hat{w} -\tfrac{1}{\sqrt{\delta}}\hat{s}= s^0 - c  \enspace.
\end{align}
Thus we need the least norm solution in $\hat{w}$ and $\hat{s}$ which is given by the pseudo-inverse of the system matrix applied to the right-hand side. That is
\begin{align*}
\begin{bmatrix}
\hat{w} \\ \hat{s}
\end{bmatrix}
& = 
\begin{bmatrix}
\mA^\top & -\tfrac{1}{\sqrt{\delta}} \mI_b
\end{bmatrix}^\dagger  (s^0 - c) \\
& =  \begin{bmatrix}
 \mA \\
 -\tfrac{1}{\sqrt{\delta}}\mI_b
\end{bmatrix}
(\mA^\top \mA +\tfrac{1}{\delta}\mI_b)^{-1}   (s^0 - c) \enspace,
\end{align*}
where we used that $\mM^\dagger = \mM ^\top (\mM \mM^\top)^\dagger$ for every matrix $\mM.$ Consequently, substituting back we have that~\eqref{eq:slackL2eqconstprojsolw} and~\eqref{eq:slackL2eqconstprojsolb} is the solution. 
\end{proof}

\begin{lemma}[L2 Unidimensional Inequality Constraint] \label{lem:slackL2ineqconst}
Let $\delta >0, c\in \R$ and $w,w^0,a\in \R^d$ .
The closed-form solution to 
\begin{align}
    \label{eq:slackL2ineqconstproj}
    w',s' =& \argmin_{w\in\R^d, s \in \R^b } \norm{w - w^0}^2 + \delta (s-s^0)^2 \nonumber \\
    &\, \mbox{s.t. } a^\top (w-w^0) +c \leq s \enspace,
\end{align}
is given by 
\begin{align} \label{eq:slackL2ineqconstprojsolw}
w' & =  w^0 - \delta \frac{(c-s^0)_+}{ 1 +\delta \norm{a}^2} a \enspace, \\
s' & = s^0+   \frac{(c-s^0)_+}{ 1 +\delta \norm{a}^2} \enspace. \label{eq:slackL2ineqconstprojsolb}
\end{align}
\end{lemma}
\begin{proof}
The problem~\eqref{eq:slackL2ineqconstproj} is an L2 projection onto a halfspace. 
The solution depends if the projected vector $(w,s) =(w^0,s^0)$ is in the halfspace.

If $w = w^0$ and $s=s^0$ satisfies in the linear inequality constraint, that is if $c \leq s^0$, in which case the solution is simply $w' = w^0$ and $s' = s^0.$

Else, $(w^0,s^0)$ is out of the feasible set, that is $c > s^0$, then we need to project $(w^0,s^0)$ onto the boundary of the halfspace, which means project onto 
\[ \{ (w, s) \in \R^d \times \R^d \, | \, a^\top (w - w^0) + c = s \}  \enspace. \]
In which case the solution is given in Lemma~\ref{lem:slackL2eqconst} (with $w^0 \leftarrow w^0$, $s^0 \leftarrow s^0$, $\delta \leftarrow \delta$, $\mA \leftarrow a$ and $c \leftarrow c$) in~\eqref{eq:slackL2eqconstprojsolw} and~\eqref{eq:slackL2eqconstprojsolb}.
\end{proof}

\begin{lemma}[L2 Unidimensional Equality Constraint with Positivity] \label{lem:slackL2eqconstpos}
Let $\delta >0$, $w^0, a \in \R^d$ and let $s^0, c \in \R.$
Consider the projection problem
\begin{align}
    w', s' =& \argmin_{w\in\R^d, s \in \R^b } \norm{w - w^0}^2 + \delta (s-s^0)^2 \nonumber \\
    &\, \mbox{s.t. }a^\top (w-w^0) +c = s \enspace, \nonumber \\
    &\phantom{\, \mbox{s.t. }} s \geq 0 \enspace. \label{eq:slackL2eqconstposproj}
\end{align}
The solution to~\eqref{eq:slackL2eqconstposproj} 
\begin{align}
w' &= w^0 -c\frac{a}{\norm{a}^2}+ \frac{\left( \delta \norm{a}^2s^0+c \right)_+  }{1+\delta \norm{a}^2} \frac{a}{\norm{a}^2} \enspace, \label{eq:slackL2eqconstposprojsolwpos} \\
s' &= \frac{\left( \delta \norm{a}^2s^0+c \right)_+  }{1+\delta \norm{a}^2} \enspace. \label{eq:slackL2eqconstprojsolbpos}
\end{align}
\end{lemma}
Note that the case given in~\eqref{eq:slackL2eqconstprojsolw}
and~\eqref{eq:slackL2eqconstprojsolbpos} is a direct application of Lemma~\ref{lem:slackL2eqconst}.
\begin{proof}
Substituting out the slack variable $s$ gives
\begin{align}
    & \min_{w\in\R^d, s \in \R^b } \norm{w - w^0}^2 + \delta \norm{a^\top (w-w^0) +c-s^0}^2 \nonumber \\
    &\, \mbox{s.t. } a^\top (w-w^0) +c \geq 0.\label{eq:tempprojsubsirtosin} \\
    &\phantom{\, \mbox{s.t. } } s = a^\top (w - w^0) + c \enspace. \nonumber
\end{align}
Re-arranging the objective function of the above gives
\begin{align*}
    &\norm{w - w^0}^2 + \delta \norm{a^\top (w-w^0) +c-s^0}^2 \\ 
    &= \norm{w-w^0}_{\mI + \delta a a^\top }^2 + 2\delta a^\top (w-w^0)(c-s^0) + \delta\norm{c-s^0}^2\\
    &= \norm{w-w^0 + \delta (c-s^0)(\mI + \delta a a^\top )^{-1}a}_{\mI + \delta a a^\top }^2  \\
    & \quad + (\mbox{constants w.r.t. } w,s) \enspace,
\end{align*}
since $(\mI + \delta a a^\top )^{-1}$ always exists. 

Consequently the solution to~\eqref{eq:tempprojsubsirtosin} is also the solution to
\begin{align}
    & \min_{w\in\R^d, s \in \R^b}\norm{w-w^0 + \delta (c-s^0)(\mI + \delta a a^\top )^{-1}a}_{\mI + \delta a a^\top }^2 \nonumber \\
    &\, \mbox{s.t. } a^\top (w-w^0) +c \geq 0 \label{eq:tempprojsubsirtosin2} \\
    &\phantom{\, \mbox{s.t. } } s = a^\top (w - w^0) + c \enspace. \nonumber
\end{align}
With the variable substitution $\hat{w} = (\mI +\delta a a^\top)^{1/2}w$ and $\hat{w}^0 = (\mI +\delta a a^\top)^{1/2} w^0$ the above is equivalent to solving
\begin{align}
    & \min_{\hat{w} \in \R^d, s \in \R^b} \norm{\hat{w}-\hat{w}^0 + \delta (c-s^0) (\mI +\delta a a^\top)^{-1/2} a}^2\nonumber \\
    &\, \mbox{s.t. } a^\top (\mI +\delta a a^\top)^{-1/2} (\hat{w}-\hat{w}^0) +c \geq 0.\label{eq:tempprojsubsirtosin3} \\
    &\phantom{\, \mbox{s.t. } } s = a^\top (\mI +\delta a a^\top)^{-1/2} (\hat{w} - \hat{w}^0) + c \enspace. \nonumber
\end{align}
Let $\hat{a} \eqdef (\mI +\delta a a^\top)^{-1/2} a$, the problem is equivalent to
\begin{align}
    & \min_{\hat{w} \in \R^d, s \in \R^b} \norm{\hat{w} - \left(\hat{w}^0 - \delta (c-s^0) \hat{a}
    \right)}^2\nonumber \\
    &\, \mbox{s.t. } \hat{a}^\top (\hat{w} - \hat{w}^0) + c \geq 0.\nonumber \\
    &\phantom{\, \mbox{s.t. } } s = \hat{a}^\top (\hat{w} - \hat{w}^0) + c \enspace. \nonumber
\end{align}
This is now a projection onto a half plain for which the solution is given in Lemma~\ref{lem:least-norm-sol-ineq} jointly with the equality $s = \hat{a}^\top (\hat{w} - \hat{w}^0) + c$.
Indeed, applying Lemma~\ref{lem:least-norm-sol-ineq} with $v \leftarrow \hat{w}^0 - \delta  (c-s^0)\hat{a}$, $a \leftarrow \hat{a}$, $w^0 \leftarrow \hat{w}^0$ and $c \leftarrow c$  and gives the solution
\begin{align}
    \hat{w}' 
    & = \hat{w}^0 - \delta  (c-s^0)\hat{a} + \frac{\left( \delta \norm{\hat{a}}^2  (c-s^0)-c\right)_+}{ \norm{\hat{a}}^2}  \hat{a} \nonumber \\
    &=  \hat{w}^0 - \delta  (c-s^0)\hat{a} + \left( \delta (c-s^0)-\frac{c}{\norm{\hat{a}}^2}\right)_+  \hat{a}\label{eq:tempnsoih8zhh4} \\
    &= \hat{w}^0 -\frac{c}{\norm{\hat{a}}^2}\hat{a}  + \left( \delta (s^0-c)+\frac{c}{\norm{\hat{a}}^2} \right)_+ \hat{a} \enspace.
\end{align}
where we used that for $c_1 = \delta (c-s^0) $ and $c_2= c/\norm{\hat{a}}^2\in \R$ we have that
\[-c_1 + (c_1 -c_2)_+ = -c_2 + (c_2 -c_1)_+ \enspace.\]


Before substituting back the $w$ variable, we can simplify by using the Sherman–Morrison formula we have that
\begin{align}
    (\mI + \delta a a^\top )^{-1/2}\hat{a} &= (\mI + \delta a a^\top )^{-1}a \nonumber \\
    & = \left(\mI - \frac{\delta}{1 + \delta \norm{a}^2} a a^\top\right) a  \nonumber \\
    & = \frac{1}{1+\delta \norm{a}^2} a \enspace. \label{eq:tempsop84hozhhz}
\end{align} 
This implies that
\begin{align}
    \norm{\hat{a}}^2 &=  a^\top (\mI + \delta a a^\top )^{-1}a \nonumber \\
    & = \frac{\norm{a}^2}{1+\delta \norm{a}^2} \enspace. \label{eq:tempsop84hozhhzmore}
\end{align}
Thus, to substitute back $\hat{w} = (\mI +\delta a a^\top)^{1/2}w$ we multiply~\eqref{eq:tempnsoih8zhh4} by $(\mI +\delta a a^\top)^{-1/2}$ and using the above gives
\begin{align}
    w' & = w^0 -\frac{c}{\norm{\hat{a}}^2}  (\mI +\delta a a^\top)^{-1/2} \hat{a} + \left( \delta (s^0-c)+\frac{c}{\norm{\hat{a}}^2} \right)_+  (\mI +\delta a a^\top)^{-1/2} \hat{a} \nonumber   \\
     &\overset{\eqref{eq:tempsop84hozhhz}}{=} w^0 -\frac{c}{\norm{\hat{a}}^2}\frac{1}{1+\delta \norm{a}^2} a + \frac{1}{1+\delta \norm{a}^2}\left( \delta (s^0-c)+\frac{c}{\norm{\hat{a}}^2} \right)_+   a \nonumber \\
    &\overset{\eqref{eq:tempsop84hozhhzmore}}{=} w^0 -c\frac{a}{\norm{a}^2}+ \frac{\left( \delta \norm{a}^2s^0+c \right)_+  }{1+\delta \norm{a}^2} \frac{a}{\norm{a}^2} \enspace. \nonumber \\
\end{align}
Using the equality constraint we have that
\begin{align}
    s' &=  a^\top(w'-w^0) +c \nonumber \\
    &\overset{\eqref{eq:slackL2eqconstposprojsolwpos}}{=} 
    -c\frac{\norm{a}^2}{\norm{a}^2}+ \frac{\left( \delta \norm{a}^2s^0+c \right)_+  }{1+\delta \norm{a}^2}+c  \nonumber \\
    & =  \frac{\left( \delta \norm{a}^2s^0+c \right)_+  }{1+\delta \norm{a}^2} \enspace.
\end{align}
\end{proof}

\begin{lemma}[L2 Unidimensional Inequality Constraints and Positivity] \label{lem:slackL2ineqconstpos}
Let $\delta > 0$ and $w^0,a\in \R^d.$ 
Furthermore, let $c \geq 0$ and $s^0 \in \R.$
The closed-form solution to 
\begin{align}
    w',s' =& \argmin_{w\in\R^d, s \in \R } \norm{w - w^0}^2 + \delta (s-s^0)^2 \nonumber \\
    &\, \mbox{s.t. } a^\top (w-w^0) +c \leq s \enspace, \nonumber \\
    & \phantom{\, \mbox{s.t. }} s \geq 0, \label{eq:slackL2ineqconstposproj}
\end{align}
is given by 
    \begin{align} \label{eq:ineqposslackw3}
    w' &= w^0 -\min\left\{\delta \frac{(c-s^0)_+}{1+ \delta\norm{a}^2} , \; \frac{c}{\norm{a}^2} \right\} a \enspace, \\
    s' &= \left( s_0 +\frac{(c-s^0)_+}{1+ \delta\norm{a}^2} \right)_+ \enspace. \label{eq:ineqposslacks3}
\end{align}
\end{lemma}
\begin{proof}

The problem~\eqref{eq:slackL2ineqconstposproj} is an L2 projection onto the intersection of two halfspaces. Consequently the solution may be inside the intersection (case $(III)$), on the boundary  with $a^\top (w-w^0) +c =s$ (case $(I)$ and $(II)$) or on the boundary with $s=0$ (case $(IV)$). We will examine each of these possibilities and we will soon show that they depend on where  $s^0$ is on the following interval
\begin{figure}[!h]
    \centering
 \begin{tikzpicture}[x=100]
        \draw (-0.5,0) --  node[above] {$(I)$} (0,0);
       \draw (0.0,0) -- node[above] {$(II)$} (1.0,0);
       \draw (1.0,0) -- node[above] {$(III)$}  (1.5,0);
       \draw (0, 0) node[below=7pt] {$-\frac{c}{\delta \norm{a}^2}$};
       \draw[] (0,-0.1) -- (0,0.1);
       \draw (1, 0) node[below=7pt] {$c$};       
       \draw[] (1,-0.1) --  (1,0.1);
   \end{tikzpicture}
   \caption{The real line of possible values for $s_0$ divided into three segments.}
   \label{fig:s0intervals}
\end{figure}
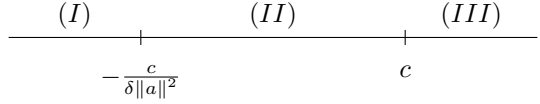

[{\bf Case $(III)$:}] First we check if the solution is inside the constraint set. We can verify this by testing if  $w =w^0$ and $s=s^0$ satisfies the constraints
\begin{equation}
    s^0 \geq c \quad \mbox{and} \quad s^0\geq 0 \quad \Leftrightarrow \quad s^0 \geq c \enspace. \label{eq:tempconditionasdasd}
\end{equation}
If $s^0 \geq c$ holds then the solution is simply $(w',s') = (w^0,s^0).$ This case corresponds to $(III)$ in Figure~\ref{fig:s0intervals}.
Alternatively if~\eqref{eq:tempconditionasdasd} does not hold then necessarily
\begin{equation}
    s^0 \leq c
\end{equation}
and at least one of the inequality constraints must be active at the optimal point as the problem~\eqref{eq:slackL2ineqconstposproj} is an L2 projection onto the intersection of two halfspaces.
Let us consider what happens when one or the other constraint is active.

[{\bf Case $(II)$:}] 
If the constraint $a^\top (w-w^0) +c =s$ is active, then our problem reduces to that of Lemma~\ref{lem:slackL2eqconstpos} where we have that the solution is given by~\eqref{eq:slackL2eqconstposprojsolwpos} and~\eqref{eq:slackL2eqconstprojsolbpos}. According to Lemma~\ref{lem:slackL2eqconstpos} the solution depends on the condition 
\begin{equation}\label{eq:tempcond03a9hn9psh4}
     \delta \norm{a}^2 s^0 + c \geq 0.
\end{equation}

If~\eqref{eq:tempcond03a9hn9psh4} holds, then $-\frac{c}{\delta \norm{a}^2} \leq s^0$ and since $s^0 \leq c$ we are in line segment $(II)$ of Figure~\ref{fig:s0intervals}. 
In this case Lemma~\ref{lem:slackL2eqconstpos} implies that
    \begin{align}
        w' &\overset{\eqref{eq:slackL2eqconstposprojsolwpos}}{=} w^0 - \frac{\delta(c - s^0)}{1 + \delta \norm{a}^2}  a  \enspace, \nonumber\\
        s' &\overset{\eqref{eq:slackL2eqconstprojsolbpos}}{=}  s^0 + \frac{c - s^0}{1 + \delta \norm{a}^2} \enspace. \label{eq:ssoltemp1}
    \end{align}

[{\bf Case $(I)$:}] 
Alternatively, if condition~\eqref{eq:tempcond03a9hn9psh4} does not hold, then $s^0 < -\frac{c}{\delta \norm{a}^2}$, and so we are in line segment $(I)$ of Figure~\ref{fig:s0intervals}. In this case, still by Lemma~\ref{lem:slackL2eqconstpos} we have that
\begin{align*}
    w' &\overset{\eqref{eq:slackL2eqconstposprojsolwpos}}{=}  w^0 - c\frac{a}{\norm{a}^2} \enspace, \\
    s' &\overset{\eqref{eq:slackL2eqconstprojsolbpos}}{=} 0 \enspace.
\end{align*}

[{\bf Case $(IV)$:}] 
If the constraint $s' =0$ is active then our problem reduces to 
\begin{align}
  & \min_{w\in\R^d} \norm{w - w^0}^2 \nonumber \\
    &\, \mbox{s.t. } -a^\top (w-w^0) -c \geq 0 \enspace.  \label{eq:tempo9zho9eap9hez4}
\end{align}
This is a projection onto a halfspace, for which the solution in Lemma~\ref{lem:least-norm-sol-ineq} with $a \leftarrow -a$, $v \leftarrow w^0$ and $c \leftarrow -c$ gives
\begin{equation}\label{eq:templzie8hze}
    w' =  w^0 - c\frac{a}{\norm{a}^2} \enspace.
\end{equation}
Furthermore, plugging in this solution gives that $a^\top(w-w^0) +c = 0 =s$, thus this is simply a special case of {\bf Case II \& III}, and we need not consider it separately.

Now that we have enumerated our possible cases,
the above observations lead to the following case analysis
 \begin{align} \label{eq:KKTcasesolution}
   \mbox{Case III: If $s^0 \geq c$, then} \qquad   & \begin{cases}
           w'  = w^0  \\
            s'  = s^0  & 
       \end{cases}  \\
    \mbox{Case II:  If   $s^0 \leq c \quad \mbox{and} \quad c + \delta \norm{a}^2 s^0 \geq 0$ then} \qquad &
     \begin{cases}
           w'  =  \displaystyle w^0  -\delta \frac{(c-s^0)_+}{1+ \delta\norm{a}^2} a  \\
            s'  =  \displaystyle s^0  + \frac{(c-s^0)_+}{1+ \delta\norm{a}^2}  & 
       \end{cases} 
       \\
     \mbox{Case I, If $c + \delta \norm{a}^2 s^0 \leq 0 $ } \qquad & \begin{cases}
          \displaystyle  w'  =  w^0  -\frac{c}{\norm{a}^2} a  \\
            s'  = 0 & 
       \end{cases} 
\end{align}
The above cases can be written condensely as~\eqref{eq:ineqposslackw3} and~\eqref{eq:ineqposslacks3}.

\end{proof} 

\paragraph{Proof of Lemma~\ref{lem:slackL2ineqconstpos} with KKT}
Because our first proof for Lemma is a bit unconventional, here we provide a different based on verifying the KKT equations. The two proofs (fortunately) arrive at the same conclusion.
\begin{proof}
    Let $\delta > 0$, $c \geq 0$ and $w^0,a\in \R^d.$ and $s^0 \in \R$.
    Let  $w \in \R^d$ and $s \in \R$  and let $z  \eqdef 
    \begin{pmatrix}
        w \\
        s
    \end{pmatrix}$. Let the objective function in~\eqref{eq:slackL2ineqconstposproj} be rewritten as 
    \begin{equation}
        \label{eq:g_obj_PI}
        g : 
        \begin{cases} 
            \R^{d+1} \rightarrow \R \\
          z =   (w,s) \mapsto \norm{w - w^0}^2 + \delta (s-s^0)^2 \enspace,
        \end{cases}
    \end{equation}
     and let the inequality constraints auxiliary function be reformulated as
    \begin{equation}
        \label{eq:C_I_PI}
        C_I : 
        \begin{cases} 
            \R^{d+1} \rightarrow \R^2 \\
            z \mapsto 
            \begin{pmatrix}
                C_1 (z) \\
                C_2 (z)
            \end{pmatrix} \eqdef
            \begin{pmatrix}
                a^\top (w - w^0) + c - s \\
                - s
            \end{pmatrix} 
            \enspace,
        \end{cases}
    \end{equation}
    where $C_1 (z), C_2 (z) \in \R$.
    These constraints are qualified as they are linear in $z$.
  With this notation we can re-write Then the optimization problem given in~\eqref{eq:slackL2ineqconstposproj}, that is, 
    \begin{align*}
        w',s' =& \argmin_{w\in\R^d, s \in \R } \norm{w - w^0}^2 + \delta (s-s^0)^2 \\
        &\, \mbox{s.t. } a^\top (w-w^0) +c \leq s \enspace, \\
        & \phantom{\, \mbox{s.t. }} s \geq 0,
    \end{align*}
    is equivalent to
    \begin{align}
        \label{eq:PI}
        z' =& \argmin_{z \in \R^{d+1}} g(z) \\
        &\, \mbox{s.t. } C_I (z) \leq_{\R^2} 0_2 \enspace.
    \end{align}
    In the above constraints, $0_2$ denotes the zero vector of size $2$ and $\leq_{\R^2}$ represent the inequality sign for each row of vectors in $\R^2$.  
    We called the above problem (PI) for ``Problem with Inequalities''.
    
    The problem~\eqref{eq:PI} is convex because the objective $g$ and constraints $C_I$ are all convex and differentiable. Thus according to Theorem 4.33 of~\cite{gilbert2021fragments}, if there exist a Lagrangian multiplier $\lambda^* \in \R^2$ satisfying the Karush-Kuhn-Tucker (KKT) conditions then $z^*$ is a global minimum of (PI) in~\eqref{eq:PI}. 
    We now look at the KKT conditions to see when they are satisfied :
    
    Let $z^*, \lambda^* \in \R^{d+1} \times \R^2$ such that
    \begin{align}
        & \nabla g (z^*) + C_I' (z^*)^\top \lambda^* = 0_{d+1} \label{eq:KKT_a} \\
        & 0_2 \leq_{\R^2} \lambda^* \; \bot \; C_I (z^*) \leq_{\R^2} 0_2 \enspace, \label{eq:KKT_b} 
    \end{align}
    where the second condition means that the inequalities must be satisfied, the Lagrangian multipliers must be non-negative and their scalar product must be null.

    The gradient of the objective function~\eqref{eq:g_obj_PI} is given by
    \begin{equation}
        \nabla g(z) = 2 
        \begin{pmatrix}
            w - w^0 \\
            \delta (s - s^0)
        \end{pmatrix} \in \R^{d+1} \enspace,
    \end{equation}
    and the derivative of the constraints equals
    \begin{equation}
        C_I' (z) = 
        \begin{pmatrix}
            \horzbar \; a^\top \; \horzbar & -1  \\
            \horzbar \; 0_d^\top \; \horzbar & -1
        \end{pmatrix} \in \R^{2 \times d+1} \enspace,
    \end{equation}
    since $(C_I' (z))_{ij} = \frac{\partial c_i}{\partial z_j} (z)$ for all $i \in \{1, 2\}$ and $ j \in \{1, \ldots, d+1\}$.
    The first equality of the KKT system~\eqref{eq:KKT_a} thus becomes
    \begin{align}
        \nabla g (z^*) + C_I' (z^*)^\top \lambda^2 = 0_{d+1} 
        &\iff 2 
        \begin{pmatrix}
            w^* - w^0 \\
            \delta (s^* - s^0)
        \end{pmatrix} +
        \begin{pmatrix}
            \vertbar & \vertbar \\[.5em]
            a & 0_d \\[.5em]
            \vertbar & \vertbar \\
            -1 & -1
        \end{pmatrix}
        \begin{pmatrix}
            \lambda_1^* \\
            \lambda_2^* \\
        \end{pmatrix}
        = 
        \begin{pmatrix}
            0_d \\
            0 \\
        \end{pmatrix} \nonumber \\
        \iff&  
        \begin{pmatrix}
            2 (w^* - w^0) + \lambda_1^* a \\
            2 \delta (s^* - s^0) - \lambda_1^* - \lambda_2^*
        \end{pmatrix}
        =
        \begin{pmatrix}
            0_d \\
            0 \\
        \end{pmatrix} \label{eq:KKT_a_developped}
    \end{align}
    And the second part of the KKT conditions~\eqref{eq:KKT_b}, enforcing the complementarity of the inequality constraints, becomes
    \begin{align}
        & 0_2 \leq_{\R^2} \lambda^* \; \bot \; C_I (z^*) \leq_{\R^2} 0_2
        \iff 
        \begin{cases}
            & \lambda_1^* \geq 0 \hfill (i)\\
            & \lambda_2^* \geq 0 \hfill (ii)\\
            & a^\top (w^* - w^0) + c - s^* \leq 0 \hfill (iii)\\
            & - s^* \leq 0 \hfill (iv) \\
            & \lambda_1^* (a^\top (w^* - w^0) + c - s^*) - \lambda_2^* s^* = 0 \quad \hfill (v)
        \end{cases} \label{eq:KKT_b_developped}
    \end{align}

    We will now explore four cases depending on which the inequality constraints are activated (\ie equals $0$) and check which are feasible and under which conditions.
    
    \paragraph{Case I) $\lambda_1^* = \lambda_2^* = 0$.}\mbox{}\\
    In this case, no constraint is activated, and~\eqref{eq:KKT_a_developped} implies that 
    \begin{equation}
        \label{eq:case_I_solution}
        z^* = 
        \begin{pmatrix}
            w^* \\
            s^*
        \end{pmatrix} =
        \begin{pmatrix}
            w^0 \\
            s^0
        \end{pmatrix} \enspace.
    \end{equation}
    Since $\lambda_1^* = \lambda_2^* = 0$, $(i), (ii)$ and $(v)$ of~\eqref{eq:KKT_b_developped} are verified.
    According to $(iii)$ and $(iv)$ of~\eqref{eq:KKT_b_developped}, this solution is feasible if $s^0 \geq 0$ and $s^0 \geq c$.
    Since we assume here that $c \geq 0$\footnote{As in our applications it corresponds to non-negative losses $f_i (w^t)$.} the only condition to be satisfied for the solution in~\eqref{eq:case_I_solution} to be feasible is to have
    \begin{equation}
        \label{eq:case_I_condition}
        s^0 \geq c \enspace.
    \end{equation}
    
    \paragraph{Case II) $\lambda_1^* = 0$ and $s^* = 0$ ($\iff C_2 (z^*) = 0$).}\mbox{}\\
    In this case, only the second constraint is activated, and the upper part of the system in~\eqref{eq:KKT_a_developped} implies that 
    \begin{equation}
        \label{eq:case_II_solution}
        z^* = 
        \begin{pmatrix}
            w^* \\
            s^*
        \end{pmatrix} =
        \begin{pmatrix}
            w^0 \\
            0
        \end{pmatrix} \enspace.
    \end{equation}
    Moreover, the last line of~\eqref{eq:KKT_a_developped} gives that
    \begin{equation}
        \label{eq:lambda_2_case_II}
        \lambda_2^* = - 2 \delta s^0 \enspace.
    \end{equation}
    We now need to check if all conditions in~\eqref{eq:KKT_b_developped} are satisfied.
    Inequality $(ii)$, jointly with~\eqref{eq:lambda_2_case_II} implies that we need to have
    \begin{equation*}
        s^0 \leq 0 \enspace.
    \end{equation*}
    Also, one can directly see that $(iii)$ of~\eqref{eq:KKT_b_developped} implies that $c \leq 0$. 
    Yet, we initially assumed that $c \geq 0$, which implies that we must have $c = 0$ for this solution to be feasible. 
    To summarize, the solution in~\eqref{eq:case_II_solution} is feasible if
    \begin{equation}
        \label{eq:case_II_condition}
        c = 0 \quad \mbox{and} \quad s^0 \leq 0 \enspace.
    \end{equation}
    
    \paragraph{Case III) $a^\top (w^* - w^0) + c - s^* = 0$ ($\iff C_1 (z^*) = 0$) and $\lambda_2^* = 0$.}\mbox{}\\
    
    Here, last line of equation~\eqref{eq:KKT_a_developped} implies that
    \begin{equation}
        \label{eq:case_III_lambda_1}
        \lambda_1^* = 2 \delta (s^* - s^0) \enspace,
    \end{equation}
    which, after being injected into the first equation in~\eqref{eq:KKT_a_developped} leads to
    \begin{equation}
        \label{eq:case_III_intermediate_eq}
        w^* - w^0 = - \delta (s^* - s^0) a \enspace.
    \end{equation}
    Now by using the assumption of Case III stating that the first inequality is activated, we have that 
    \begin{align*}
        & a^\top (w^* - w^0) + c - s^* = 0 \\
        \iff& s^* = a^\top (w^* - w^0) + c \enspace.
    \end{align*}
    By injecting~\eqref{eq:case_III_intermediate_eq} into this later equation, we finally get
    \begin{align}
        &s^* = - \delta \norm{a}^2 (s^* - s^0) + c \nonumber \\
        \iff& s^* = \frac{1}{1 + \delta \norm{a}^2} (c + \delta \norm{a}^2 s^0) \nonumber \\
        \iff& s^* = s^0 + \frac{c-s^0}{1 + \delta \norm{a}^2}  \label{eq:case_III_z_dplus1}
        \enspace.
    \end{align}
    Finally, by putting back this value into~\eqref{eq:case_III_intermediate_eq} we get that
    \begin{equation}
        \label{eq:case_III_z_d}
        w^* = w^0 - \frac{\delta (c - s^0)}{1 + \delta \norm{a}^2} a \enspace.
    \end{equation}
    So in this case, the solution is given by
    \begin{equation}
        \label{eq:case_III_solution}
        z^* = 
        \begin{pmatrix}
            w^* \\
            s^*
        \end{pmatrix} =
        \begin{pmatrix}
            w^0 - \dfrac{\delta (c - s^0)}{1 + \delta \norm{a}^2} a \\
            s^0 + \dfrac{c-s^0}{1 + \delta \norm{a}^2} 
        \end{pmatrix} \enspace.
    \end{equation}
    Let us now check when are all the conditions in~\eqref{eq:KKT_b_developped} are verified.
    We just need to check whether $(i)$ and $(iv)$ are true.
    From~\eqref{eq:case_III_lambda_1}, $(i)$ implies that we must have
    \begin{align*}
        s^* - s^0 \geq 0 & \overset{\eqref{eq:case_III_z_dplus1}}{\iff} \frac{1}{1 + \delta \norm{a}^2} (c + \delta \norm{a}^2 s^0) - s^0 \geq 0 \\
        & \iff c \geq s^0 \enspace.
    \end{align*}
    Moreover, $(iv)$ implies that we required the following to hold
    \begin{align*}
        s^* \geq 0 & \overset{\eqref{eq:case_III_z_dplus1}}{\iff} c + \delta \norm{a}^2 s^0 \geq 0 \enspace.
    \end{align*}
    To summarize, the solution in~\eqref{eq:case_III_solution} is feasible if
    \begin{equation}
        \label{eq:case_III_condition}
        c \geq s^0 \quad \mbox{and} \quad c + \delta \norm{a}^2 s^0 \geq 0 \enspace.
    \end{equation}

    \paragraph{Case IV) $a^\top (w^* - w^0) + c - s^* = 0$ ($\iff C_1 (z^*) = 0$) and $s^* = 0$ ($\iff C_2 (z^*) = 0$).}\mbox{}\\
    
    The first equation in~\eqref{eq:KKT_a_developped} implies that
    \begin{align}
        \lambda_1^* a &= - 2 (w^* - w^0) \label{eq:case_IV_intermediate_computations_w} \\
        \implies \lambda_1^* &= - \frac{2}{\norm{a}^2} a^\top (w^* - w^0) \nonumber \\
        \iff \lambda_1^* &\overset{C_1 (z^*) = 0}{=} \frac{2}{\norm{a}^2} c \enspace. \nonumber
    \end{align}
    Then, by injecting this in the last line of~\eqref{eq:KKT_a_developped} we get
    \begin{align*}
        \lambda_2^* &= - \lambda_1^* - 2 \delta s^0 \\
        &= - \frac{2}{\norm{a}^2} (c + \delta \norm{a}^2 s^0) \enspace.
    \end{align*}
Also re-arranging~\eqref{eq:case_IV_intermediate_computations_w} gives
    \begin{align*}
        w^* &= w^0 - \frac{\lambda_1^*}{2} a \\
        &= w^0 - \frac{c}{\norm{a}^2} a \enspace.
    \end{align*}
    The solution in this case is thus given by
    \begin{equation}
        \label{eq:case_IV_solution}
        z^* = 
        \begin{pmatrix}
            w^* \\
            s^*
        \end{pmatrix} =
        \begin{pmatrix}
            w^0 - \dfrac{c}{\norm{a}^2} a \\
            0
        \end{pmatrix} \enspace.
    \end{equation}
    Finally, the point $(i)$ of~\eqref{eq:KKT_b_developped} is satisfied as we assumed that $c$ is non-negative and the point $(ii)$ is satisfied, and so the solution in~\eqref{eq:case_IV_solution} is feasible, as long as 
    \begin{equation}
        \label{eq:case_IV_condition}
        c + \delta \norm{a}^2 s^0 \leq 0 \enspace.
    \end{equation}
    
    Putting all of the above together we have that
    \begin{align} \label{eq:KKTcasesolution}
   \mbox{Case I: If $s^0 \geq c$, then} \qquad   & \begin{cases}
           w'  = w^0  \\
            s'  = s^0  & 
       \end{cases}  \\
     \mbox{Case II: If $s^0 \leq c =0$ , then} \qquad   & \begin{cases}
           w'  = w^0  \\
            s'  = 0 & 
       \end{cases}  \\     
    \mbox{Case III:  If   $s^0 \leq c \quad \mbox{and} \quad c + \delta \norm{a}^2 s^0 \geq 0$ then} \qquad &
     \begin{cases}
           w'  =  \displaystyle w^0  -\delta \frac{(c-s^0)_+}{1+ \delta\norm{a}^2} a  \\
            s'  =  \displaystyle s^0  + \frac{(c-s^0)_+}{1+ \delta\norm{a}^2}  & 
       \end{cases} 
       \\
     \mbox{Case IV, If $c + \delta \norm{a}^2 s^0 \leq 0 $ } \qquad & \begin{cases}
          \displaystyle  w'  =  w^0  -\frac{c}{\norm{a}^2} a  \\
            s'  = 0 & 
       \end{cases} 
\end{align}
    All of the above cases can be written in a condensed form given by~\eqref{eq:ineqposslackw3} and~\eqref{eq:ineqposslacks3}, that is
        \begin{align} \label{eq:ineqposslackw3rr}
    w' &= w^0 -\min\left\{\delta \frac{(c-s^0)_+}{1+ \delta\norm{a}^2} , \; \frac{c}{\norm{a}^2} \right\} a \enspace, \\
    s' &= \left( s_0 +\frac{(c-s^0)_+}{1+ \delta\norm{a}^2} \right)_+ \enspace. \label{eq:ineqposslacks3rr}
\end{align}
    To see this equivalence, one need only check all of the condition of each case and the resulting solution. Most cases follow straightforwardly, perhaps the least obvious case is that of Case IV. To verify Case IV it helps to note that
    \[ c + \delta \norm{a}^2 s^0 \leq 0 \implies \left( s_0 +\frac{(c-s^0)_+}{1+ \delta\norm{a}^2} \right)_+=0. \]
\end{proof}

\section{L1 Projection Lemmas with Slack Variable}



\begin{lemma}[L1 Equality Constraints and Positivity]
\label{lem:slackL1eqconstpos}
Let $w^0,a\in \R^d$,   $\lambda >0$ and $c \in \R$. 
The closed-form solution to 
\begin{align}
    w', s' =& \argmin_{w\in\R^d, s\in \R}\tfrac{1}{2} \norm{w - w^0}^2 + \lambda |s| \nonumber \\
    &\, \mbox{s.t. } a^\top (w-w^0) +c = s \enspace, \nonumber \\
    &\,\phantom{ \mbox{s.t. }} s \geq 0   \label{eq:slackL1eqconstposproj}
\end{align}
is given by 
\begin{align}
    w' &= w^0-\lambda a  + \frac{a}{\norm{a}^2} \left( \lambda\norm{a}^2 -c\right)_+ \label{eq:slackL1eqconstposprojsol} \\
    &= w^0 - \min\left\{ \frac{c}{\norm{a}^2}, \lambda \right\} a \enspace,
\end{align}
and in $s$ is given by
\begin{align}
    s' = \left( c-\lambda\norm{a}^2 \right)_+ \enspace. \label{eq:slackL1eqconstposprojsols}
\end{align}
\end{lemma}
\begin{proof}
Substituting out the $s$ variable and using the positivity constraint gives
\begin{align}
    w' =& \argmin_{w\in\R^d}\tfrac{1}{2} \norm{w - w^0}^2 + \lambda(a^\top (w-w^0) +c ) \nonumber \\
    &\, \mbox{s.t. }  a^\top (w-w^0) +c \geq 0 \enspace.   \label{eq:newtraphslacklinearconstnobeta}
\end{align}
Re-writing the objective function and dropping constants independent of $w$ we have that
\begin{equation}
    \tfrac{1}{2} \norm{w - w^0}^2 + \lambda a^\top (w-w^0) = \tfrac{1}{2} \|w - w^0 + \lambda a\|^2 - \underbrace{\frac{\lambda^2}{2}  \| a\|^2}_{\text{independent of } w}.
\end{equation}
Consequently solving~\eqref{eq:newtraphslacklinearconstnobeta} is equivalent to solving
\begin{align}
    w' =& \argmin_{w\in\R^d}\tfrac{1}{2} \|w - w^0 +\lambda   a \|^2 \nonumber \\
    &\, \mbox{s.t. }   a^\top (w-w^0) +c \geq 0 \enspace. \nonumber
\end{align}
We now have a projection onto a half-space for which the solution is given in Lemma~\ref{lem:least-norm-sol-ineq}. Indeed applying Lemma~\ref{lem:least-norm-sol-ineq}
with $v \leftarrow w^0-\lambda a  $
gives~\eqref{eq:slackL1eqconstposprojsol}.
Finally to get the solution $s'$ we substitute $w'$ from~\eqref{eq:slackL1eqconstposprojsol} in the constraint which gives
\begin{align}
    s' &= a^\top (w'-w^0) +c \nonumber \\
    &=c -\lambda \norm{a}^2  +  \left( \lambda\norm{a}^2 -c\right)_+ \nonumber \\
    & = \left( c-\lambda\norm{a}^2 \right)_+\enspace .
\end{align}
\end{proof}


\begin{lemma}[L1 Inequality Constraints and Positivity]
\label{lem:slackL1ineqconst}
Let $w^0,a\in \R^d$,  $\lambda >0$ and $c >0.$ 
The closed-form solution to 
\begin{align}
    w', s' =& \argmin_{w\in\R^d, \ s\in\R^b}\tfrac{1}{2} \norm{w - w^0}^2 + \lambda |s| \nonumber \\
    &\, \mbox{s.t. } a^\top (w-w^0) +c \leq s \enspace, \label{eq:slackL1ineqconstproj} \\
    &\,\phantom{ \mbox{s.t. }} s \geq 0 \nonumber
\end{align}
is given by
\begin{align}
w' &=w^0-\lambda a  + \frac{a}{\norm{a}^2} \left( \lambda\norm{a}^2 -c\right)_+ \label{eq:slackL1ineqconstsolw1} \\
& = w^0 - a \min\left\{ \frac{c}{\norm{a}^2}, \lambda \right\} \enspace, \nonumber \\
s' &=  \left( c-\lambda\norm{a}^2 \right)_+. \label{eq:slackL1ineqconstsols1}
\end{align}
\end{lemma}
\begin{remark}
    By applying the above lemma with the dimension of $s$ set to $b=1$ and
    \begin{equation*}
        w^0 \leftarrow w^t, \quad a \leftarrow \nabla \ell_i (w^t), \quad c \leftarrow \ell_i (w^t) \enspace,
    \end{equation*}
    we show that the closed-form formula of the variational formulation in~\eqref{eq:SPSmaxproj}, is given by the \texttt{SPS}$_{\max}$ updates~\eqref{eq:SPSmax}. 
\end{remark}
\begin{proof}
First we check if $(w,s) = (w^0,0)$ is in  the constraint set, that is if $c \leq 0.$ If yes, then the solution is given by $(w,s) = (w^0,0)$. Otherwise one of the constraint must be active at the solution. 

\begin{enumerate}
    \item Consider the case that the $a^\top (w-w^0) +c = s$ is active. In this case,  our problem is equivalent to that of Lemma~\ref{lem:slackL1eqconstpos}. The solution given by  Lemma~\ref{lem:slackL1eqconstpos} is in~\eqref{eq:slackL1ineqconstsolw1} and~\eqref{eq:slackL1ineqconstsols1}, and consequently have that
    \begin{align}
        \tfrac{1}{2} \norm{w - w^0}^2 + \lambda |s| & = 
        \tfrac{1}{2} \norm{a}^2\left(\lambda   - \frac{\left( \lambda\norm{a}^2 -c\right)_+}{\norm{a}^2}  \right)^2 + \lambda  \left( c-\lambda\norm{a}^2 \right)_+ \enspace. \label{eq:objfunoption1}
    \end{align}
    \item If the constraint $s=0$ is active then we have that the solution is given by
    
\begin{align}
    w' =& \argmin_{w\in\R^d} \norm{w -w^0}^2 \nonumber \\
    &\, \mbox{s.t. } -a^\top (w - w^0) - c \geq 0 \enspace.  \label{eq:newtraphineqtemp}
\end{align}
    This now fits the setting of Lemma~\ref{lem:least-norm-sol-ineq}
and thus the solution
is given by
\begin{align}
    w' &=  w^0 - \frac{c}{ \norm{a}^2} a \enspace, \label{eq:slackL1ineqconstsolw2} \\
    s' &=  0 \enspace. \label{eq:slackL1ineqconstsols2}
\end{align}
Consequently have that the objective function is given by
\begin{align}
     \tfrac{1}{2} \norm{w - w^0}^2 + \lambda |s|  & =   \tfrac{1}{2} \frac{c^2}{\norm{a}^2} \enspace. \label{eq:objfunoption2}
\end{align}
\end{enumerate}
With our two options for the solution, we now need to compare the two resulting objective functions~\eqref{eq:objfunoption1} and~\eqref{eq:objfunoption2} to know when each option is active.  To compare these two objective functions, we divide into two cases depending on the condition
\begin{equation}\label{eq:condL1ineqconst}
    \lambda\norm{a}^2 -c \geq 0 \enspace.
\end{equation}
If~\eqref{eq:condL1ineqconst} holds then it is easy to show that~\eqref{eq:slackL1ineqconstsolw1} and~\eqref{eq:slackL1ineqconstsols1} is equivalent to~\eqref{eq:slackL1ineqconstsolw2} and~\eqref{eq:slackL1ineqconstsols2}. 
Alternatively if~\eqref{eq:condL1ineqconst} does not hold, then~\eqref{eq:objfunoption1} is smaller or equal to~\eqref{eq:objfunoption2}. Indeed this follows since in this case we have
\begin{align*}
    \tfrac{1}{2} \norm{w - w^0}^2 + \lambda |s| &\overset{\eqref{eq:objfunoption1}}{=} \tfrac{1}{2} \lambda^2 \norm{a}^2  + \lambda c - \lambda^2\norm{a}^2 \\
    & = \lambda c - \frac{\lambda^2}{2}\norm{a}^2 \enspace.
\end{align*}
Consequently $\eqref{eq:objfunoption1} \leq \eqref{eq:objfunoption2} $ if and only if
\begin{align*}
    &c - \frac{\lambda}{2}\norm{a}^2 \leq  \frac{c^2}{2\lambda\norm{a}^2} \\
    \iff& \tfrac{1}{2} \left( \frac{c}{\sqrt{\lambda}\norm{a}} - \sqrt{\lambda} \norm{a}\right) \geq 0 \enspace,
\end{align*}
where the last inequality always holds, and thus $\eqref{eq:objfunoption1} \leq \eqref{eq:objfunoption2} $ and the solution is again given by ~\eqref{eq:slackL1ineqconstsolw1} and~\eqref{eq:slackL1ineqconstsols1}.
\end{proof}

\section{Missing Proofs of Lemmas}

\subsection{Proof of Lemma~\ref{lem:ProjL1pludL2}}
\begin{proof}
We can re-write the slack part of the objective function in~\eqref{eq:SPSL1proj} as 
\begin{align*}
    \lambda s + \tfrac{1}{2} (s-s^t)^2 &= \tfrac{1}{2} \left(s-s^t+\lambda\right)^2 + \mbox{constants w.r.t. $w$ and $s$} \enspace. 
\end{align*} 
Dropping constants independent of $s$ and $w$ we have that~\eqref{eq:SPSL1proj} is equivalent to solving
\begin{align}
   & \min_{w\in\R^d, s \geq 0} \norm{w - w^t}^2 + \left(s-s^t+\lambda\right)^2 \nonumber \\
 &\, \mbox{s.t. } \ell_i(w^t) + \dotprod{\nabla \ell_i(w^t), w - w^t} \leq s \enspace. \nonumber 
\end{align}
This projection problem now fits the format of Lemma~\ref{lem:slackL2ineqconstpos} which when applied with $s_0 \leftarrow s^t-  \lambda$ and $\delta \leftarrow 1,$ $c \leftarrow \ell_i(w^t)$ and $a \leftarrow \nabla \ell_i(w^t)$ gives 

\begin{align}
w^{t+1} & = w^t -  \min\left\{ \frac{(\ell_i(w^t)-s^t+\lambda)_+}{1+ \norm{\nabla \ell_i(w^t)}^2}, \; \frac{ \ell_i(w^t)}{\norm{\nabla \ell_i(w^t)}^2}  \right\} \nabla \ell_i(w^t). \nonumber \\
s^{t+1} & =   \left( s^t-\lambda + \frac{(\ell_i(w^t)-s^t+\lambda)_+}{1+ \norm{\nabla \ell_i(w^t)}^2}  \right)_+\label{eq:SPSL1-app}  
\end{align}
which is equivalent to the solution~\eqref{eq:SPSL1}.

\end{proof}






\subsection{Proof of Lemma~\ref{lem:SPSdam}}

\begin{proof}

First note that the slacks variables in the objective function of~\eqref{eq:SPSL2proj} can be re-written as
\begin{align*}
 (s-s^t)^2 + \lambda s^2 & = \; (s^t)^2 - 2s s^t + \frac{1}{\hat{\lambda}} s^2 \\
& = \frac{1}{\hat{\lambda}} (s - \hat{\lambda} s^t)^2  + \mbox{constants w.r.t. } s \enspace. 
\end{align*}
Consequently, after dropping constants, solving~\eqref{eq:SPSL2proj} is equivalent to 
\begin{align}
    w^{t+1}, s^{t+1} =& \underset{w \in \R^d, s \in \R}{\argmin} \; \norm{w - w^t}^2 + \frac{1}{\hat{\lambda}} \left(s - \hat{\lambda} s^t \right)^2\nonumber \\
    &\, \mbox{s.t. } \dotprod{\nabla \ell_i(w^t), w - w^t} + \ell_i(w^t) \leq s \enspace.   \label{eq:newtraphslackrelaxtewmp}
\end{align}
By Lemma~\ref{lem:slackL2ineqconst} with $a \leftarrow \nabla \ell_i(w^t)$, $c \leftarrow \ell_i(w^t)$, $s^0 \leftarrow \hat{\lambda} s^t$ and $\delta \leftarrow 1/\hat{\lambda}$ we have that the solution is given by ~\eqref{eq:SPSL2}.
\end{proof}

\section{Convergence theory of \texttt{SPSL1}}
\label{sec:spsl1_app}

As a proxy of \texttt{SPSL1} we will analyse a close variant where we do not impose a positivity constraint on the slack variable in~\eqref{eq:SPSL1proj}. We will show that for the convex and smooth setting this constraint makes no difference since, as we will show, the slack variable is positive.

\subsection{\texttt{SPSL1} without Positivity Constraint} \label{sec:SPSL1nopos}
Here we consider \texttt{SPSL1} but without an explicit non-negativity constraint on the slack variable. That is, by dropping $s \geq 0$ in~\eqref{eq:SPSL1proj} we have 
\begin{align}
    w^{t+1}, s^{t+1} =& \argmin_{w\in\R^d, s \in \R} \tfrac{1}{2}\norm{w - w^t}^2 + \tfrac{1}{2}(s-s^t)^2 +\lambda s \nonumber \\
    &\, \mbox{subject to } \ell_i(w^t) + \dotprod{\nabla \ell_i(w^t), w - w^t} \leq s \enspace.   \label{eq:SPSL1noposproj}
\end{align}
\begin{lemma}
The closed-form solution to~\eqref{eq:SPSL1noposproj} is given by
\begin{align} 
w^{t+1} & = w^t - \frac{(\ell_i(w^t)-s^t+\lambda)_+}{1+ \norm{\nabla \ell_i(w^t)}^2}\nabla \ell_i(w^t) \enspace, \nonumber \\
s^{t+1} & =   s^t-\lambda + \frac{(\ell_i(w^t)-s^t+\lambda)_+}{1+ \norm{\nabla \ell_i(w^t)}^2} \enspace. \label{eq:SPSL1nopos} 
\end{align}
\end{lemma}
\begin{proof}
First we rewrite the objective function as 
\begin{equation*}
    \tfrac{1}{2}\norm{w - w^t}^2 + \tfrac{1}{2}(s-s^t)^2 +\lambda s = \tfrac{1}{2}\norm{w - w^t}^2 + \tfrac{1}{2}(s-s^t + \lambda)^2 + \text{constant terms w.r.t.} \, w \, \text{and} \, s 
\end{equation*}
and then apply Lemma~\ref{lem:slackL2ineqconst} (with $w^0 \eqdef w^t$, $s^0 \eqdef s^t - \lambda$, $a \eqdef \nabla \ell_i (w^t)$, $c \eqdef \ell_i (w^t)$ and $\delta \eqdef 1$). The solution to the above is simply~\eqref{eq:SPSL1nopos}.
\end{proof}

\subsection{\texttt{SGD} viewpoint}

In order to study the convergence of the variant of \texttt{SPSL1} without positivity constraint, we first reinterpret its updates~\eqref{eq:SPSL1noposproj} as SGD updates of a particular objective function. 
Then, involving classical SGD convergence proof technique, we are able to prove the convergence of both iterates $w^t$ and $s^t$.
Indeed, the method in~\eqref{eq:SPSL1nopos} can be interpreted as an online \texttt{SGD} method applied to
\begin{equation}\label{eq:SPSL1SGDobj}
    G_t(w,s) \eqdef \frac{1}{n} \sum_{i=1}^n G_{t,i}(w,s) \eqdef 
    \frac{1}{n} \sum_{i=1}^n\tfrac{1}{2}\frac{(\ell_i(w)-s+\lambda)_+^2}{1+ \norm{\nabla \ell_i(w^t)}^2} +\lambda s \enspace.
\end{equation}
Indeed by applying a \texttt{SGD}  step with step size $\gamma$ on this objective function we recover \texttt{SPSL1}  \textit{without positivity constraint} updates given in~\eqref{eq:SPSL1nopos}
\begin{align}
w^{t+1} &= w^t - \gamma \nabla_w  G_{t,i}(w,s)  \\
&= w^t - \gamma \frac{(\ell_i(w^t)-s^t+\lambda)_+}{1+ \norm{\nabla \ell_i(w^t)}^2}\nabla \ell_i(w^t) \enspace,\label{eq:SPSL1noposwup}\\
s^{t+1} & = s^t -\gamma  \nabla_s  G_{t,i}(w,s) \\
&=   s^t - \gamma \lambda + \gamma  \frac{(\ell_i(w^t)-s^t+\lambda)_+}{1+ \norm{\nabla \ell_i(w^t)}^2} \enspace. \label{eq:SPSL1nopossup}
\end{align}
Furthermore, this way of adding a step size $\gamma$ is equivalent to the relaxation step given in~\eqref{eq:SPSL1noposproj-main}. 

\subsection{Properties of \texttt{SPSL1}}

Before exposing our main convergence lemma, we need to prove properties on the online objective function of our \texttt{SGD} reformulation.

\subsubsection{Growth Property}

We start by presenting the following growth property. 
We refer to it as a growth property as it relates the variation of the individual objective functions to their value. 
\begin{lemma}[Growth SPSL1] \label{lem:growthSPSL1}
Let $(w^t, s^t)_t$ be a sequence of iterates given by \texttt{SPSL1} method without positivity constraint on the slack variable given in~\eqref{eq:SPSL1nopos}. 
Then, for every $t \in \N$ and $i \in [n] \eqdef \{1, \ldots, n\}$ we have
\begin{eqnarray}
      \norm{ \nabla_w  G_{t,i}(w^t,s^t)}^2 + \norm{ \nabla_s  G_{t,i}(w^t,s^t)}^2 \; = \;  2G_{t,i}(w^t,s^t)
      -2\lambda \hat{s}^{t+1}- \lambda^2, 
\end{eqnarray}
where
\begin{equation}\label{eq:shat}
 \hat{s}^{t+1} \eqdef  s^t-\lambda +\frac{(\ell_i(w^t)-s^t+\lambda)_+}{1+ \norm{\nabla \ell_i(w^t)}^2}.
\end{equation}
\end{lemma}
\begin{proof}
Direct computations lead to the desired equality
\begin{align*}
    \norm{ \nabla_w  G_{t,i}(w^t,s^t)}^2 + \norm{ \nabla_s          G_{t,i}(w^t,s^t)}^2 
    &=\frac{(\ell_i(w^t)-s^t+\lambda)_+^2}{(1+ \norm{\nabla \ell_i(w^t)}^2)^2}\norm{\nabla \ell_i(w^t)}^2 
    + \left(\lambda - \frac{(\ell_i(w^t)-s^t+\lambda)_+}{1+ \norm{\nabla \ell_i(w^t)}^2} \right)^2 \\
    &=\frac{(\ell_i(w^t)-s^t+\lambda)_+^2}{(1+ \norm{\nabla \ell_i(w^t)}^2)^2}(1 + \norm{\nabla \ell_i(w^t)}^2) 
    -2\lambda  \frac{(\ell_i(w^t)-s^t+\lambda)_+}{1+ \norm{\nabla \ell_i(w^t)}^2} 
    + \lambda^2 \\
    &=\frac{(\ell_i(w^t)-s^t+\lambda)_+^2}{1+ \norm{\nabla \ell_i(w^t)}^2}
    -2\lambda\frac{(\ell_i(w^t)-s^t+\lambda)_+}{1+ \norm{\nabla \ell_i(w^t)}^2} 
    + \lambda^2 \\
  &=\frac{(\ell_i(w^t)-s^t+\lambda)_+^2}{1+ \norm{\nabla \ell_i(w^t)}^2}+2\lambda s^t-2\lambda\left(s^t-\lambda +\frac{(\ell_i(w^t)-s^t+\lambda)_+}{1+ \norm{\nabla \ell_i(w^t)}^2}\right) - \lambda^2\\
  &=2G_{t,i}(w^t,s^t) -2\lambda \hat{s}^{t+1}- \lambda^2 \enspace,
\end{align*}
where in the last equality we used~\eqref{eq:SPSL1nopossup}.
\end{proof}

\subsubsection{Inherited Convexity}

Then, we show that the convexity of the losses is transferred to the objective function. 

\begin{lemma}\label{lem:convexSPSL1}
If $\ell_i$ is convex for $i=1,\ldots, n$ then the function $G_t(w,s)$ defined in~\eqref{eq:SPSL1SGDobj} is convex.
\end{lemma}
\begin{proof}
Computing the
gradient of $ (\ell_i(w)-s +\lambda)_+^2$  we have that
\[  \nabla (\ell_i(w)-s +\lambda)_+^2 = 2\begin{bmatrix}
\nabla \ell_i(w) \\
-1
\end{bmatrix}(\ell_i(w)-s +\lambda)_+ \enspace.\]
Let $\ones_{\R^+} (.)$ be the indicator function over the set of positive real numbers.
Computing the Hessian gives
\begin{align}
 \nabla^2 (\ell_i(w)-s +\lambda)_+^2 &=
  2 \ones_{\R^+} ((\ell_i(w)-s +\lambda)_+) \begin{bmatrix} 
\nabla \ell_i(w) \\
-1
\end{bmatrix}\begin{bmatrix}
\nabla \ell_i(w) ^\top & - 1
\end{bmatrix}
+
2\begin{bmatrix}
\nabla^2 \ell_i(w) & 0\\
0 & 0
\end{bmatrix}(\ell_i(w)-s +\lambda)_+ \nonumber \\
& =  
2 \ones_{\R^+} ((\ell_i(w)-s +\lambda)_+) \begin{bmatrix} 
\nabla \ell_i(w)\nabla \ell_i(w) ^\top & -\nabla \ell_i(w) \\
-\nabla \ell_i(w) ^\top & 1
\end{bmatrix} \nonumber \\
& \phantom{=} +
2\begin{bmatrix}
\nabla^2 \ell_i(w) & 0\\
0 & 0
\end{bmatrix}(\ell_i(w)-s +\lambda)_+ \enspace. \label{eq:fialphai2hessSPSL1}
\end{align}

Now let $\mI_n \in \R^{n\times n}$ be the identity matrix in $\R^{n\times n}$, let
\begin{align}
  \mD_t(w,s) & \eqdef \; 
  \diag{\frac{\ones_{\R^+} ((\ell_1(w)-s +\lambda)_+)}{\norm{\nabla \ell_1(w^t)}^2+1}, \ldots, \frac{\ones_{\R^+} ((\ell_n(w)-s +\lambda)_+)}{\norm{\nabla \ell_n(w^t)}^2+1}} \in \R^{n \times n} \nonumber \\
  \mH_t(w,s) & \eqdef \;  \sum_{i=1}^{n+1} \nabla^2 \ell_i(w)  \frac{(\ell_i(w)-s +\lambda)_+}{\norm{\nabla \ell_i(w^t)}^2+1} \in \R^{d \times d}
\end{align}
and let
\[D F(w) \; \eqdef
 \; \begin{bmatrix}\nabla \ell_1(w), \ldots, \nabla \ell_n(w)\end{bmatrix} \in \R^{d\times n} \enspace.\]

Let $\ones_{n} \in \R^n$ be the vector full of ones of size $n$. Using~\eqref{eq:fialphai2hessSPSL1} and by the definition of $G_t$ in~\eqref{eq:SPSL1SGDobj} we have that
\begin{equation}\label{eq:ht2hessstarSPSL1}
 \nabla^2 F_t(w, s) \; = \; 
 \frac{1}{n} \underbrace{ \begin{bmatrix}
 DF(w) \mD_t(w,s) DF(w)^\top & -DF(w)\mD_t(w,s) \ones_{n}  \\
-(DF(w)\mD_t(w,s) \ones_{n})^\top & \ones_{n}^\top\mD_t(w,s)\ones_{n},
\end{bmatrix} }_{\eqdef \mM_t(w,s)}
+
\begin{bmatrix}
\mH_t(w,s)  & 0  \\
0 & 0
\end{bmatrix} 
\in \R^{(d+1) \times (d+1)} \enspace.
\end{equation}
where we used that $\nabla^2 \tfrac{1}{2}s^2 =1$ .
Thus the matrix~\eqref{eq:ht2hessstarSPSL1} is a sum of two terms. 

The second matrix is symmetric positive semi-definite because the losses $\ell_i$ are convex. 
So we will show that the first matrix $\mM_t(w,s)$
is  symmetric positive definite. Indeed, 
left and right multiplying the above by $[x,\,a] \in \R^{d+1}$ where $x \in \R^d$ and $a \in \R$ gives
\begin{eqnarray*}
 \begin{bmatrix}
x & a
\end{bmatrix}^\top 
\mM_t(w,s)
\begin{bmatrix}
x \\ a
\end{bmatrix} 
&\overset{\eqref{eq:ht2hessstarSPSL1}}{ =} &
 \begin{bmatrix}
x & a
\end{bmatrix}^\top 
\begin{bmatrix}
 D F(w)  \mD_t(w,s) D F(w) ^\top x - a D F(w) \mD_t(w,s) \ones_{n}  \\
-(D F(w) \mD_t(w,s) \ones_{n})^\top x + a \ones_{n}^\top\mD_t(w,s)\ones_{n}
\end{bmatrix}  \\
&= &
\norm{\mD_t(w,s)^{1/2}D F(w) ^\top x }^2 -2a (D F(w) \mD_t(w,s) \ones_{n})^\top x+ a^2 \ones_{n}^\top\mD_t(w,s)\ones_{n} \\
& = & \norm{\mD_t(w,s)^{1/2}(D F(w) ^\top x - a \ones_{n}) }^2 - a^2 \norm{\mD_t(w,s)^{1/2} \ones_{n} }^2 + a^2 \ones_{n}^\top\mD_t(w,s)\ones_{n} \\
&= & \norm{\mD_t(w,s)^{1/2}(D F(w) ^\top x - a \ones_{n}) }^2\geq 0 \enspace.
\end{eqnarray*}
\end{proof}

\subsubsection{Lower bounds}

The last property we need to provide a  convergence result are bounds over the auxiliary objective function. This will allow us to translate the convergence of $G_t(w,s)$ to the convergence of $\ell(w).$

\begin{lemma}\label{lem:SPSL1lowerbound}
We have that
\begin{equation}\label{eq:SPSL1lowerboundgen}
  \frac{\lambda}{n} \sum_{i=1}^n \ell_i(w) +\frac{\lambda^2}{2} \left( 1 - \frac{1}{n}\sum_{i=1}^n \norm{\nabla \ell_i(w^t)}^2\right) \leq G_t(w,s) \enspace, \quad \forall w, s.
\end{equation}
Consequently if $\ell_i$ is $\sigma$--Lipschitz then
\begin{equation}\label{eq:SPSL1lowerboundlip}
  \frac{\lambda }{n} \sum_{i=1}^n \ell_i(w) +\frac{\lambda^2}{2} \left( 1 - \sigma^2\right)  \leq G_t(w,s) \enspace, \quad \forall w, s.
\end{equation}
Alternatively if $\ell_i$ is $L_{\max}$--smooth, convex and non-negative then
\begin{equation}\label{eq:SPSL1lowerboundinter}
  \frac{\lambda(1- \lambda L_{\max}) }{n} \sum_{i=1}^n \ell_i(w) +\frac{\lambda^2}{2} \leq G_t(w,s) \enspace, \quad \forall w, s.
\end{equation}

Furthermore, for a fixed $w^\star\in \R^d$ and $s^{\star} = \max_{i=1,\ldots, n} \ell_i(w^\star) +\lambda $ we have that
\begin{equation}
      G_t(w^\star,s^{\star}) = \lambda \max_{i=1,\ldots, n} \ell_i(w^\star) +\lambda^2 \enspace.
\end{equation}
\end{lemma}

\begin{proof}
To prove the lower bound in~\eqref{eq:SPSL1lowerboundgen} that does not depend on $s$, we will minimize $G_t(w,s)$ in $s$. We will also make use of the fact that
\begin{align}
    \min_s  G_t(w,s) &= \min_s  \frac{1}{n} \sum_{i=1}^n \tfrac{1}{2} \frac{(\ell_i(w)-s+\lambda)_+^2}{1+ \norm{\nabla \ell_i(w^t)}^2}  +\lambda s \nonumber \\
    & \geq
      \frac{1}{n} \sum_{i=1}^n \min_s \left( \tfrac{1}{2}\frac{(\ell_i(w)-s+\lambda)_+^2}{1+ \norm{\nabla \ell_i(w^t)}^2}  +\lambda s\right)  \nonumber \\
    &=  \frac{1}{n} \sum_{i=1}^n \min_s G_{t,i}(w,s) \enspace.
    \label{eq:tmeps9ohs484}
\end{align}

Now it remains to find the minima in $s$ of $G_{t,i}(w,s).$ Differentiating in $s$ and setting to zero we have that

\[ \frac{(\ell_i(w)-s+\lambda)_+}{1+\norm{\nabla \ell_i(w^t)}^2}  = \lambda \enspace.\]
Clearly $\ell_i(w)-s+\lambda >0$ otherwise the above equality would read $0 = \lambda$, which is not possible.
Consequently we have that

\[ \frac{\ell_i(w)-s+\lambda}{1+\norm{\nabla \ell_i(w^t)}^2}  = \lambda \enspace.\]
Re-arranging we have that
\[\frac{s}{1+ \norm{\nabla \ell_i(w^t)}^2} = \frac{\ell_i(w)+\lambda}{1+ \norm{\nabla \ell_i(w^t)}^2}  -\lambda =  \frac{\ell_i(w) - \lambda \norm{\nabla \ell_i(w^t)}^2}{1+ \norm{\nabla \ell_i(w^t)}^2} \enspace. \]
Isolating $s$ then gives the minimizer of $s \mapsto G_{t,i}(w, s)$
\[s_i^* \eqdef \ell_i(w) - \lambda \norm{\nabla \ell_i(w^t)}^2 \enspace.\]
It is worth noting that even if we have dropped the positivity constraint in our optimization method~\eqref{eq:SPSL1nopos}, if $1/2L_ {\max} \geq \lambda$, then Lemma~\ref{lem:convsmoothinter} implies that $s_i^* \geq 0$.

We can now see that this is indeed the minima for $s$ even when including the positive part since
\[ \ell_i(w)- s_i^* +\lambda = \lambda \left( 1 + \norm{\nabla \ell_i(w^t)}^2 \right) \geq 0 \enspace.\]
Inserting this minimizer back into $G_{t,i}(w,s)$ gives
\begin{eqnarray}
G_{t,i}(w, s_i^* ) &=&
      \tfrac{1}{2}\frac{\lambda^2 \left(1+\norm{\nabla \ell_i(w^t)}^2\right)^2}{1+ \norm{\nabla \ell_i(w^t)}^2} +\lambda \left(\ell_i(w) - \lambda \norm{\nabla \ell_i(w^t)}^2\right) \nonumber \\
      &=& \frac{\lambda^2}{2} \left(1+\norm{\nabla \ell_i(w^t)}^2\right) + \lambda \left(\ell_i(w) - \lambda \norm{\nabla \ell_i(w^t)}^2\right) \nonumber \\
      &=& \frac{\lambda^2}{2} \left(1 - \norm{\nabla \ell_i(w^t)}^2\right) + \lambda \ell_i(w) \enspace.
    \label{eq:Gts_sistar}
\end{eqnarray}
Using the above in~\eqref{eq:tmeps9ohs484} gives
\begin{align}
    \min_s  G_t(w,s) & \geq
      \frac{1}{n} \sum_{i=1}^n G_{t,i}(w, s_i^*) \\
      &=  \frac{\lambda }{n} \sum_{i=1}^n \ell_i(w) +\frac{\lambda^2}{2} \left( 1 - \frac{1}{n}\sum_{i=1}^n \norm{\nabla \ell_i(w^t)}^2\right) \enspace,
\end{align}
which concludes the proof of~\eqref{eq:SPSL1lowerboundgen}.

Now if $\ell_i$ is $\sigma$--Lipschitz then $\norm{\nabla \ell_i(w^t)} \leq \sigma$, thus~\eqref{eq:SPSL1lowerboundlip} follows from~\eqref{eq:SPSL1lowerboundgen}.

Alternatively if $\ell_i$ is convex, smooth and non-negative, then by Lemma~\ref{lem:convsmoothinter} we have that $2 L_{\max}\ell_i(w^t) \geq \norm{\nabla \ell_i(w^t)}^2. $
Consequently
\begin{align*}
     \lambda \ell_i(w) -\frac{\lambda^2}{2}\norm{\nabla \ell_i(w^t)}^2 &\geq \lambda \ell_i(w) -\lambda^2  L_{\max}\ell_i(w^t) \\
     & =\lambda(1- \lambda L_{\max})\ell_i(w^t) \enspace.
\end{align*}   
Thus~\eqref{eq:SPSL1lowerboundinter} follows from the above.

Finally, by plugging in $s(w) = \max_{i=1,\ldots, n} \ell_i(w) +\lambda$ we have that
\begin{eqnarray}
      G_t(w,s(w)) &=& \tfrac{1}{2} \frac{\left(\ell_i(w)- \max_{i=1,\ldots, n} \ell_i(w)\right)_+^2}{1+ \norm{\nabla \ell_i(w^t)}^2} + \lambda \left( \max_{i=1,\ldots, n} \ell_i(w) + \lambda \right) \nonumber\\
      &=& \lambda  \max_{i=1,\ldots, n} \ell_i(w) + \lambda^2 \enspace.
\end{eqnarray}


\end{proof}

\subsubsection{Lower bounds on $s^t$}

Here we provide lower bounds of the sequence of slack iterates across iterations of \texttt{SPSL1} for two different settings.

\paragraph{non-negativity.}

\begin{lemma}[non-negativity of $s^t$]
    \label{lem:SPSL1snon-negativesmooth}
    Let $\ell_i$ be non-negative, convex, and $L_i$--smooth. Let $s^0 = 0$ and $1/2L_{\max} \geq \lambda$. It follows that
\begin{equation}
    \label{eq:SPSL1snon-negative}
    s^t \geq 0, \quad \mbox{and}\quad \hat{s}^{t+1} \geq0, \quad \mbox{for all }t\geq 0 \quad,
\end{equation}
where $\hat{s}^{t+1}$ is defined in~\eqref{eq:shat} which is the slack variable updated with $\gamma =1.$
\end{lemma}
\begin{proof}
The proof follows by induction.\\

1) \textbf{Initialization}: $s^0 = 0$ verifies~\eqref{eq:SPSL1snon-negative}.\\
2) \textbf{Induction}: let us assume that for a fixed $t \in \N$, $s^t \geq 0$.
Let us prove that the induction property is also true for the next slack iterate $s^{t+1}$.

First, let 
\[\hat{s}^{t+1} = s^t - \lambda + \frac{(\ell_i(w^t) - s^t + \lambda)_+}{1 + \norm{\nabla \ell_i (w^t)}^2} \]
and thus 
\[s^{t+1} = (1-\lambda) s^t + \lambda \hat{s}^{t+1},\]
due to the relaxation step in~\eqref{eq:SPSL1nopossup}.
Using Lemma~\ref{lem:convsmoothinter} we have that
\begin{align}
    \ell_i(w^t) &\geq \frac{1}{2L_{\max}} \norm{\nabla \ell_i (w^t)}^2 \nonumber \\
    &\geq \lambda \norm{\nabla \ell_i (w^t)}^2  \nonumber \\
    &\geq (\lambda - s^t) \norm{\nabla \ell_i (w^t)}^2 \enspace, \label{eq:after_L_max}
\end{align}
where we apply the assumed upper bound on $\lambda$ and then use the induction hypothesis stating that $s^t \geq 0$.

By adding $\lambda - s^t$ on both sides of~\eqref{eq:after_L_max} and rearranging the different terms, we obtain the following inequality:
\begin{equation*}
    s^t - \lambda + \frac{\ell_i(w^t) - s^t + \lambda}{1 + \norm{\nabla \ell_i (w^t)}^2} \geq 0 \enspace.
\end{equation*}
Finally we can upper bound the numerator of the right-hand side term using the fact that $(x)_+ \geq x$ for all $x \in \R$.
From this we conclude that
\begin{equation}
    \label{eq:after_rearranging}
    \hat{s}^{t+1} \eqdef s^t - \lambda + \frac{(\ell_i(w^t) - s^t + \lambda)_+}{1 + \norm{\nabla \ell_i (w^t)}^2} \geq 0 \enspace.
\end{equation}
Thus finally $s^{t+1}$ since it is a convex combination between $\hat{s}^{t+1}$ and $s^t,$ both of which are positive.
This shows that the induction property is also true at $t+1$ which concludes the proof.
\end{proof}

\paragraph{Lower bound.}

\begin{lemma}[Lower bound on $s^t$]\label{lem:SPSL1slower}
Let $\gamma <1$, $\ell_i$ be non-negative and $\sigma$--Lipschitz. Let $s^0 \geq -\lambda \sigma^2.$ It follows that
\begin{equation}
    s^t \geq -\lambda \sigma^2, \quad \mbox{for all }t\geq 0.
\end{equation}
\end{lemma}
\begin{proof}

The proof follows by induction. We also chose to show a \emph{constructive} proof. That is, we aim to determine for what $C \in \R $ do we have $s^{t+1} \geq C+\lambda$ for all $t$. Let $s^t \geq C+ \lambda$.

First, let us suppose that $s^t \geq \ell_i(w^t) + \lambda$ then
\begin{align*}
    s^{t+1} &= s^t - \gamma\lambda +    \gamma\frac{(\ell_i(w^t)-s^t+\lambda)_+}{1+ \norm{\nabla \ell_i(w^t)}^2} \\
    &= s^t -  \gamma\lambda \\ 
    &\geq \ell_i(w^t) + \lambda(1-\gamma) \geq 0\enspace,
\end{align*}
where we used twice that $s^t \geq \ell_i(w^t) + \lambda$ and then the positivity of $\gamma$ and $\lambda$. The same argument shows that $\hat{s}^{t+1} \geq \ell_i(w^t) \geq 0.$


Alternatively, if $s^t \leq \ell_i(w^t) + \lambda$ then we have that
\begin{align}
    s^{t+1} &= s^t - \gamma \lambda +    \gamma\frac{\ell_i(w^t)-s^t+\lambda}{1+ \norm{\nabla \ell_i(w^t)}^2} \nonumber \\
    &= \gamma\frac{\ell_i(w^t)}{1+ \norm{\nabla \ell_i(w^t)}^2}+
(s^t-\lambda) \left( 1 - \frac{\gamma}{1+ \norm{\nabla \ell_i(w^t)}^2}  \right) 
+\lambda (1-\gamma)\\
    &\geq 
    \frac{\gamma \ell_i(w^t)+C\left(1-\gamma+\norm{\nabla \ell_i(w^t)}^2\right)}{1+ \norm{\nabla \ell_i(w^t)}^2}  +\lambda (1-\gamma) \enspace, \label{eq:tempz98oh8h4}
\end{align} 
where we used the induction hypothesis.

Now for $s^{t+1} \geq C+\lambda$ to hold, we need that
\begin{align}
    & \frac{\gamma \ell_i(w^t)+C\left(1-\gamma+\norm{\nabla \ell_i(w^t)}^2\right)}{1+ \norm{\nabla \ell_i(w^t)}^2}  +\lambda (1-\gamma) \geq C+\lambda \nonumber \\
    \iff & \gamma \ell_i(w^t)+C\left(1-\gamma+\norm{\nabla \ell_i(w^t)}^2\right) \geq (C+\lambda\gamma)(1+ \norm{\nabla \ell_i(w^t)}^2) \\
    \iff & -\gamma C \geq \lambda\gamma(1+ \norm{\nabla \ell_i(w^t)}^2) -\gamma \ell_i(w^t) \nonumber \\
    \iff & C \leq \ell_i(w^t) -\lambda(1+ \norm{\nabla \ell_i(w^t)}^2) \enspace. \label{eq:condition_on_C}     
\end{align}
Since the loss $\ell_i$ is non-negative and $\sigma$--Lipschitz, \ie $\norm{\nabla \ell_i(w^t)} \leq \sigma$ for all $w^t \in \R^d$.
Thus,~\eqref{eq:condition_on_C}  holds for $C \eqdef -\lambda (1 + \sigma^2)$. 

Finally, this proves that for this constant $C$ we have, by induction,  that for all $t\geq 0$
\begin{equation}
    s^t \geq C + \lambda = - \lambda \sigma^2 \enspace.
\end{equation}
Finally for $\hat{s}^{t+1} $ we have that~\eqref{eq:tempz98oh8h4} holds with $\gamma =1$ that is
\begin{align*}
     \hat{s}^{t+1} &\geq 
    \frac{ \ell_i(w^t)+C\left(\norm{\nabla \ell_i(w^t)}^2\right)}{1+ \norm{\nabla \ell_i(w^t)}^2}   \;=  \frac{ \ell_i(w^t)-\lambda (1 + \sigma^2)\left(\norm{\nabla \ell_i(w^t)}^2\right)}{1+ \norm{\nabla \ell_i(w^t)}^2}  \\
    & \geq -\frac{\lambda (1 + \sigma^2)\left(\norm{\nabla \ell_i(w^t)}^2\right)}{1+ \norm{\nabla \ell_i(w^t)}^2} \; \geq -\lambda (1 + \sigma^2),
\end{align*}
where we used that $\norm{\nabla \ell_i(w^t)}^2$ is non-negative.
\end{proof}

\subsection{\texttt{SGD} style convergence results}

Here we present first a general \textit{\texttt{SGD} style} proof for the \texttt{SPSL1} method without positivity constraint.
\begin{lemma}\label{lem:convSPSL1}
Let $\ell_i$ be convex for $i=1,\ldots, n$. 
It follows that for iterates $z^t \eqdef (w^t, s^t) \in \R^{d+1}$ generated by~\eqref{eq:SPSL1nopossup} and let $z^* \in \R^{d+1}$ we have
\begin{align}\label{eq:convSPSL1}
 \frac{2(1-\gamma )}{T} \sum_{t=0}^{T-1}\E{G_{t}(z^t)}
   & \leq \frac{1}{\gamma T}\norm{z^{0} -z^*}^2  +\frac{2}{T}  \sum_{t=0}^{T-1} \E{G_{t}(z^*)-\gamma\lambda \hat{s}^{t+1}} -\gamma \lambda^2 \enspace.
\end{align} 
\end{lemma}
\begin{proof}
Let $z^* \in \R^{d+1}$ and let $z^t \eqdef (w^t, s^t) \in \R^{d+1}$ be iterates generated by~\eqref{eq:SPSL1nopos}.
Expanding the squares and using Lemma~\ref{lem:growthSPSL1} we have that
\begin{align}
    \norm{z^{t+1} -z^*}^2  &=   \norm{z^{t} -z^*}^2  - 2\gamma \dotprod{\nabla G_{t,i}(z^t), z^t-z^*} +\gamma^2 \norm{\nabla G_{t,i}(z)}^2 \nonumber \\
    &= \norm{z^{t} -z^*}^2  - 2\gamma \dotprod{\nabla G_{t,i}(z^t), z^t-z^*} +\gamma^2\left(2G_{t,i}(w^t,s^t)-2\lambda  \hat{s}^{t+1} - \lambda^2\right) \enspace.
\end{align} 
Since the losses $(\ell_i)_{i \in [n]}$ are convex t follows from Lemma~\ref{lem:convexSPSL1} that $G_t$ is convex.
Taking expectation conditioned on $t$ and convexity of $G_t$, which follows from Lemma~\ref{lem:convexSPSL1},  gives
\begin{align*}
   \EE{t}{ \norm{z^{t+1} -z^*}^2 } &\leq  
    \norm{z^{t} -z^*}^2 - 2\gamma(G_{t}(z^t) -  G_{t}(z^*) )+\gamma^2 \left(2G_{t}(z^t)-2\lambda \EE{t}{ \hat{s}^{t+1}} - \lambda^2\right) \\
    &=  
    \norm{z^{t} -z^*}^2 - 2\gamma(1-\gamma) G_{t}(z^t)  +2\gamma   (G_{t}(z^*)-\gamma \lambda \EE{t}{  \hat{s}^{t+1}}) -\gamma^2 \lambda^2 \enspace.
\end{align*}

Taking expectation,  and re-arranging gives
\begin{align*}
 2\gamma(1-\gamma )\E{G_{t}(z^t)}
   & \leq  \E{\norm{z^{t} -z^*}^2} -  \EE{}{ \norm{z^{t+1} -z^*}^2 }   +2\gamma  \E{G_{t}(z^*)-\gamma\lambda  \hat{s}^{t+1}} -\gamma^2 \lambda^2 \enspace.
\end{align*} 
Summing up from $t=0,\ldots, T-1$ and using telescopic cancellation gives
\begin{align*}
 \frac{2\gamma(1-\gamma )}{T} \sum_{t=0}^{T-1}\E{G_{t}(z^t)}
   & \leq \frac{1}{T}\norm{z^{0} -z^*}^2  +\frac{2\gamma}{T}  \sum_{t=0}^{T-1} \E{G_{t}(z^*)-\gamma\lambda  \hat{s}^{t+1}} -\gamma^2 \lambda^2 \enspace.
\end{align*} 
Dividing through by $\gamma$ gives the result.
\end{proof}

\subsection{Proof of Theorem~\ref{theo:SPSL1lip}}

\begin{theorem}\label{theo:SPSL1lip-ap}
Let $\bar{w}^T = \frac{1}{T} \sum_{t=0}^{T-1} w^t$ and let $\ell_i$ be non-negative.  If  $\ell_i$ is $\sigma$--Lipschitz, convex  and $s^0 \geq -\lambda \sigma^2$, then 
\begin{align}\label{eq:SPSL1convLip-ap}
 (1-\gamma ) \E{\frac{1 }{n} \sum_{i=1}^n \ell_i(\bar{w}^T) } 
   & \leq \frac{1}{2\gamma\lambda T}\norm{z^{0} -z^*}^2  + \max_{i=1,\ldots, n} \ell_i(w^\star)+   \frac{\lambda}{2} \left(1  + \sigma^2 (1 + \gamma)\right) \enspace.
\end{align} 
\end{theorem}
\begin{proof}
The result follows by applying Lemma~\ref{lem:convSPSL1}, Lemmas~\ref{lem:SPSL1lowerbound} and~\ref{lem:SPSL1slower}.
First, 
by fixing $w^\star\in \R^d$ and $s^{\star} = \max_{i=1,\ldots, n} \ell_i(w^\star) +\lambda $ so that
\begin{eqnarray}
       G_t(w^\star,s^{\star}) = \lambda \max_{i=1,\ldots, n} \ell_i(w^\star) +\lambda^2 \enspace. 
\end{eqnarray}
Thus by Lemma~\ref{lem:convSPSL1} with this $z^* \eqdef (w^\star, s^{\star})$, we have after re-arranging that
\begin{align}
    & \frac{2(1-\gamma )}{T} \sum_{t=0}^{T-1}\E{G_{t}(z^t)}
   \leq \frac{1}{\gamma T}\norm{z^{0} -z^*}^2  +\frac{2\lambda}{T}  \sum_{t=0}^{T-1} \E{ \max_{i=1,\ldots, n} \ell_i(w^\star) -\gamma  \hat{s}^{t+1}} + \lambda^2(2-\gamma) \nonumber \\
   \iff& \frac{2(1-\gamma )}{T} \sum_{t=0}^{T-1}\E{G_{t}(z^t)}
   \leq \frac{1}{\gamma T}\norm{z^{0} -z^*}^2  
   + 2\lambda \max_{i=1,\ldots, n} \ell_i(w^\star)  
   - \frac{2 \gamma \lambda}{T} \sum_{t=0}^{T-1} \E{ \hat{s}^{t+1}}
   + \lambda^2(2-\gamma) \enspace. \label{eq:tempreuslnrso84}
\end{align}

Now consider the setting where $\ell_i$ is $\sigma$--Lipshitz. From~\eqref{eq:SPSL1lowerboundlip} 
in Lemma~\ref{lem:SPSL1lowerbound} we get that
\begin{align*}
    & \frac{2(1-\gamma )}{T} \sum_{t=0}^{T-1}\E{\frac{\lambda }{n} \sum_{i=1}^n \ell_i(w^t) +\frac{\lambda^2}{2} \left( 1 - \sigma^2\right)}
    \leq \frac{1}{\gamma T}\norm{z^{0} -z^*}^2  
    + 2\lambda \max_{i=1,\ldots, n} \ell_i(w^\star)  
     \\
    &\phantom{\frac{2(1-\gamma )}{T} \sum_{t=0}^{T-1}\E{\frac{\lambda }{n} \sum_{i=1}^n \ell_i(w^t) +\frac{\lambda^2}{2} \left( 1 - \sigma^2\right)}
    \leq} - \frac{2 \gamma \lambda}{T} \sum_{t=0}^{T-1} \E{ \hat{s}^{t+1}}+ \lambda^2 (2-\gamma) \\
    \iff& \frac{1-\gamma}{T} \sum_{t=0}^{T-1} \E{\frac{1}{n} \sum_{i=1}^n \ell_i(w^t)}
    + \frac{\lambda}{2} (1-\gamma) \left(1 - \sigma^2\right) 
    \leq \frac{1}{2 \gamma \lambda T}\norm{z^{0} -z^*}^2 
    + \max_{i=1,\ldots, n} \ell_i(w^\star)  
    \\
    & \phantom{\frac{1-\gamma}{T} \sum_{t=0}^{T-1} \E{\frac{1}{n} \sum_{i=1}^n \ell_i(w^t)}
    + \frac{\lambda}{2} (1-\gamma) \left(1 - \sigma^2\right) 
    \leq} - \frac{\gamma}{T} \sum_{t=0}^{T-1} \E{\hat{s}^{t+1}} + \frac{\lambda}{2} (2-\gamma) \\
    \iff& (1-\gamma) \E{\frac{1}{n} \sum_{i=1}^n \frac{1}{T} \sum_{t=0}^{T-1} \ell_i(w^t)}
    \leq \frac{1}{2 \gamma \lambda T}\norm{z^{0} -z^*}^2 
    + \max_{i=1,\ldots, n} \ell_i(w^\star)  
   \\
    & \phantom{(1-\gamma) \E{\frac{1}{n} \sum_{i=1}^n \frac{1}{T} \sum_{t=0}^{T-1} \ell_i(w^t)}
    \leq} - \frac{\gamma}{T} \sum_{t=0}^{T-1} \E{\hat{s}^{t+1}}  + \frac{\lambda}{2} \left(1 + (1 - \gamma) \sigma^2 \right) \enspace.
\end{align*} 
Then applying Jensen's inequality in the previous inequality 
with the average iterate $\bar{w}^T = \frac{1}{T} \sum_{t=0}^{T-1} w^t$ gives
\begin{equation} \label{eq:tempo8shj4}
    (1-\gamma ) \E{\frac{1 }{n} \sum_{i=1}^n \ell_i \left(\bar{w}^T \right)} 
    \leq \frac{1}{2 \gamma \lambda T}\norm{z^{0} -z^*}^2 
    + \max_{i=1,\ldots, n} \ell_i(w^\star)  
    - \frac{\gamma}{T} \sum_{t=0}^{T-1} \E{\hat{s}^{t+1}}
    + \frac{\lambda}{2} \left(1 + \sigma^2 - \gamma \sigma^2 \right) \enspace.
\end{equation}
Finally applying Lemma~\ref{lem:SPSL1slower}
where we have that
\begin{eqnarray}
       \hat{s}^t \geq -\lambda \sigma^2 \enspace, \quad \mbox{for all }t\geq 0.
\end{eqnarray}
Using the above in~\eqref{eq:tempo8shj4} gives
\begin{equation} 
    (1-\gamma ) \E{\frac{1 }{n} \sum_{i=1}^n \ell_i \left(\bar{w}^T \right)} 
    \leq \frac{1}{2 \gamma \lambda T}\norm{z^{0} -z^*}^2 
    + \max_{i=1,\ldots, n} \ell_i(w^\star)  
    + \frac{\lambda}{2} \left(1  + \sigma^2 + \gamma \sigma^2\right) \enspace.
\end{equation}
\end{proof}

\subsection{Proof of Theorem~\ref{theo:SPSL1smooth}}

\begin{theorem} \label{theo:SPSL1smooth-ap}
Let $\bar{w}^T = \frac{1}{T} \sum_{t=0}^{T-1} w^t$ .
If $\ell_i$ is convex, $L_{\max}$--smooth, non-negative  and $\lambda \leq 1/2L_{\max}$, then
\begin{align}\label{eq:SPSL1convsmooth-ap}
  \E{ \ell (\bar{w}^t)} 
   & \leq \frac{1}{2 \gamma \lambda T}\frac{\norm{z^{0} -z^*}^2}{(1-\gamma )(1- \lambda L_{\max})} +  \frac{1}{n} \sum_{i=1}^n\frac{(\ell_i(w^\star)+\lambda)_+^2 - \lambda^2  }{2\gamma \lambda (1-\gamma) (1- \lambda L_{\max}) }\enspace.
\end{align}
Furthermore, if the interpolation condition holds~\eqref{eq:interpolate} and  $\lambda = 1/2L_{\max}$ we have that
\begin{align}\label{eq:SPSL1convsmoothLmax_app}
 \E{ \ell (\bar{w}^t)} 
   & \leq \frac{2L_{\max}}{  T}\frac{\norm{z^{0} -z^*}^2}{\gamma(1-\gamma )} \enspace.
\end{align}
This is the same rate of convergence of both \texttt{SPS}$_{\max}$ (Theorem 3.4~\cite{SPS}) and \texttt{SGD} (Theorem 4.1~\cite{SPS} and Corollary 4.4~\cite{KSLGR2020unifiedsgm}).
\end{theorem}
\begin{proof}
Here we prove the convergence of \texttt{SPSL1} for convex and smooth losses $\ell_i$ with interpolation.

Again, from Lemma~\ref{lem:convSPSL1} we have that
\begin{align*} 
 \frac{2(1-\gamma )}{T} \sum_{t=0}^{T-1}\E{G_{t}(z^t)}
   & \leq \frac{1}{\gamma T}\norm{z^{0} -z^*}^2  +\frac{2}{T}  \sum_{t=0}^{T-1} \E{G_{t}(z^*)-\gamma\lambda s^{t+1}} -\gamma \lambda^2 \enspace.
\end{align*} 
Then, applying~\eqref{eq:SPSL1lowerboundinter} from Lemma~\ref{lem:SPSL1lowerbound} gives
\begin{align*} 
    & \frac{2(1-\gamma )}{T} \sum_{t=0}^{T-1} \frac{\lambda(1- \lambda L_{\max}) }{n} \sum_{i=1}^n\E{ \ell_i(w^t)} + (1-\gamma)\lambda^2
   \leq \frac{1}{\gamma T}\norm{z^{0} -z^*}^2  +\frac{2}{T}  \sum_{t=0}^{T-1} \E{G_{t}(z^*)-\gamma\lambda s^{t+1}} -\gamma \lambda^2\\ 
   \iff & 2 \lambda (1-\gamma) (1- \lambda L_{\max}) \frac{1}{T} \sum_{t=0}^{T-1} \E{\ell (w^t)}
   \leq \frac{1}{\gamma T}\norm{z^{0} -z^*}^2 - \frac{2\gamma\lambda}{T} \sum_{t=0}^{T-1} \E{s^{t+1}} + \frac{2}{T}  \sum_{t=0}^{T-1} \E{G_{t}(z^*)} - \lambda^2 
   \enspace.
\end{align*}
Using that $s^{t+1}\geq 0$ from Lemma~\ref{lem:SPSL1snon-negativesmooth} gives 

\begin{align*} 
2\gamma \lambda (1-\gamma) (1- \lambda L_{\max}) \frac{1}{T} \sum_{t=0}^{T-1}\E{ \ell(w^t)} 
  &  
   \leq \frac{1}{\gamma T}\norm{z^{0} -z^*}^2  + \frac{2}{T}  \sum_{t=0}^{T-1} \E{G_{t}(z^*)} - \lambda^2 
   \enspace.
\end{align*}

Now fix $w^\star$ and using that 
\[G_t(w^\star,0) \leq  \frac{1}{2n} \sum_{i=1}^n(\ell_i(w^\star)+\lambda)^2 \] 
and re-arranging gives
\begin{align*} 
\frac{1}{T} \sum_{t=0}^{T-1}\E{ \ell(w^t)} 
  &  
   \leq \frac{1}{\gamma T}\frac{\norm{z^{0} -z^*}^2}{2\gamma \lambda (1-\gamma) (1- \lambda L_{\max}) }  +  \frac{1}{n} \sum_{i=1}^n\frac{(\ell_i(w^\star)+\lambda)^2 - \lambda^2  }{2\gamma \lambda (1-\gamma) (1- \lambda L_{\max}) }
   \enspace.
\end{align*}

The result~\eqref{eq:SPSL1convsmooth-ap} now follows by applying Jensen's inequality.



Finally, if we assume there is a solution $w^\star$ to the interpolation equations in~\eqref{eq:interpolate} then $\ell_i(w^\star)$ and consequently the constant term on the right hand side of~\eqref{eq:SPSL1convsmooth-ap} is zero because 
\[(\ell_i(w^\star)+\lambda)^2 - \lambda^2  = \lambda^2 - \lambda^2 =0 \enspace.\]

\end{proof}

\section{Convergence theory of \texttt{SPSL2}}
\label{sec:spsl2_convergence_theory}

\subsection{\texttt{SGD} Viewpoint}

At the $t$th iteration, the \texttt{SPSL2} method can be interpreted as an online \texttt{SGD} applied to minimizing
an \emph{auxiliary} objective function
\begin{align}
    \min_{w,s} F_t(w,s) &\eqdef \frac{1}{n} \sum_{i=1}^n F_{t,i}(w,s) 
    \eqdef \frac{1}{n}\sum_{i=1}^n \tfrac{1}{2}\frac{(\ell_i(w) -\hat{\lambda} s)_+^2}{\hat{\lambda} + \norm{\nabla \ell_i(w^t)}^2} + (1 - \hat{\lambda}) \frac{s^2}{2} \enspace. \label{eq:objSPSdam}
\end{align}

Indeed, by sampling the $i$th term in the summation and taking a step of \texttt{SGD} (with subgradients) we have that
\begin{align}
    w^{t+1} &= w^t - \gamma \nabla_w\left. \left(  \tfrac{1}{2}\frac{(\ell_i(w) -\hat{\lambda} s)_+^2}{\hat{\lambda} + \norm{\nabla \ell_i(w^t)}^2} + (1 - \hat{\lambda}) \frac{s^2}{2}   \right)\right|_{t} \nonumber \\
    &= w^t -  \gamma\frac{(\ell_i(w^t) -\hat{\lambda} s^t)_+}{\hat{\lambda} + \norm{\nabla \ell_i(w^t)}^2} \nabla \ell_i (w^t) \enspace. \nonumber  \\
    s^{t+1} &= s^t -  \gamma \nabla_s \left.\left( \tfrac{1}{2}\frac{(\ell_i(w) -\hat{\lambda} s)_+^2}{\hat{\lambda} + \norm{\nabla \ell_i(w^t)}^2} + (1 - \hat{\lambda}) \frac{s^2}{2} \right)\right|_{t} \nonumber \\
    &= \left(1 - \gamma (1 - \hat{\lambda})\right)  s^t + \gamma \hat{\lambda} \frac{(\ell_i(w^t) -\hat{\lambda}  s^t)_+}{\hat{\lambda} + \norm{\nabla \ell_i(w^t)}^2} \enspace, \label{eq:SPSL2SGD}
\end{align}
where $\gamma >0$ is the learning rate. By choosing $\gamma =1$  the above is equivalent to the \texttt{SPSL2} update given in~\eqref{eq:SPSL2}. This way of introducing a learning rate $\gamma$ is equivalent to using a relaxation step as explained in Section~\ref{sec:sonv_spspl1_main}, thus~\eqref{eq:SPSL2SGD} is equivalent to \texttt{SPSL2} with a relaxation step.
Here we will prove convergence of~\eqref{eq:SPSL2SGD}.
First, we note that $s^t$ is non-negative.
\begin{lemma}\label{lem:sposSPSL2}
Let $(w^t,s^t)$ be given by~\eqref{eq:SPSL2SGD}. If $s^0 \geq 0$ then for every $\gamma \in [0,  1]$ we have that $s^t \geq 0$
\end{lemma}
\begin{proof}
First recall that $\hat{\lambda} =\frac{1}{1+\lambda},$ where $\lambda >0$, which implies that $\hat{\lambda} \in [0, 1].$ 
Consequently
\[\left(1 - \gamma (1 - \hat{\lambda})\right)  \geq 0 \enspace,\]
as $\gamma \in [0, 1]$.
Finally, by induction if $s^t \geq 0$ then
\begin{align*}
    s^{t+1} &= \left(1 - \gamma (1 - \hat{\lambda})\right)  s^t + \gamma \hat{\lambda} \frac{(\ell_i(w^t) -\hat{\lambda}  s^t)_+}{\hat{\lambda} + \norm{\nabla \ell_i(w^t)}^2} \;\geq 0 \enspace.
\end{align*} 
\end{proof}

This \texttt{SGD} interpretation of \texttt{SPSL2} allows us to prove the convergence of \texttt{SPSL2} by leveraging proof techniques from the literature of \texttt{SGD}. But first, we establish some bounds on the auxiliary objective function~\eqref{eq:objSPSdam}.

\begin{lemma}\label{lem:objSPSdambound}
Let $\ell_i$ be non-negative and $\sigma$--Lipschitz for $i=1,\ldots, n.$ Let
\begin{equation}
    \label{eq:C} 
    C (\hat{\lambda}) \; \eqdef\; 
 \frac{1 - \hat{\lambda}}{\hat{\lambda}^2 + (1 - \hat{\lambda}) \left(\hat{\lambda} + \sigma^2 \right)} \enspace. 
\end{equation}
It follows that
\begin{equation}\label{eq:Ftbound}
 \frac{C(\hat{\lambda})}{2n}\sum_{i=1}^n \ell_i(w)^2 \leq 
F_t(w,s) \enspace.
\end{equation}
Furthermore for every $w\in \R^d$ and 
$$s(w) \eqdef \frac{1}{\hat{\lambda}} \max_{i=1,\ldots, n} \ell_i(w) \enspace,$$ we have that
\begin{equation}
    \label{eq:FstarSPSdam_appdx}
    F_t(w,s(w)) \; = \; \frac{1 - \hat{\lambda}}{2\hat{\lambda}^2} \max_{i=1,\ldots, n} \ell_i(w)^2 \enspace. 
\end{equation}
\end{lemma}

\begin{proof}
We begin the proof by minimizing in $s$ the function
\begin{equation}
\label{eq:knzi8h48z4z4}
F_{t,i}(w,s) =\tfrac{1}{2}\frac{\left(\ell_i(w) -\hat{\lambda} s\right)_+^2}{\hat{\lambda} + \norm{\nabla \ell_i(w^t)}^2} + (1 - \hat{\lambda}) \frac{s^2}{2} \enspace,
\end{equation}
which is the $i$th term in  the objective in~\eqref{eq:objSPSdam}.
Ignoring the positive part (for now), we have that
\[\nabla_s F_{t,i}(w,s) = - \hat{\lambda} \frac{\ell_i(w) - \hat{\lambda} s}{\hat{\lambda} + \norm{\nabla \ell_i(w^t)}^2} + (1 - \hat{\lambda}) s = 0 \enspace.\]
Isolating $s$ we have that
\[ s^{\star}_i = \frac{\hat{\lambda} \ell_i(w)}{\hat{\lambda}^2 + (1 - \hat{\lambda}) \left(\hat{\lambda} + \norm{\nabla \ell_i(w^t)}^2 \right)} \enspace. \]
We can now see that $\hat{\lambda} s_i^{\star} \leq f_i(w)$, thus consequently this $s_i^{\star}$ is also the minimum of~\eqref{eq:knzi8h48z4z4} even when including the positive part.

 Inserted $s_i^*$ into~\eqref{eq:knzi8h48z4z4}  gives
\begin{align} \label{eq:knzi8h48z4z42} 
F_{t,i}(w,s^{\star}_i) & = \tfrac{1}{2}\frac{\left(\ell_i(w) - \frac{\hat{\lambda}^2 \ell_i(w)}{\hat{\lambda}^2 + (1 - \hat{\lambda}) \left(\hat{\lambda} + \norm{\nabla \ell_i(w^t)}^2 \right)} \right)_+^2}{ \hat{\lambda} + \norm{\nabla \ell_i(w^t)}^2 } 
+ \tfrac{1}{2} \frac{(1 - \hat{\lambda}) \hat{\lambda}^2 \ell_i(w)^2}{\left[\hat{\lambda}^2 + (1 - \hat{\lambda}) \left(\hat{\lambda} + \norm{\nabla \ell_i(w^t)}^2 \right)\right]^2} \nonumber \\
&= \tfrac{1}{2} \frac{(1 - \hat{\lambda})^2 \ell_i(w)^2 \left( \hat{\lambda} + \norm{\nabla \ell_i(w^t)}^2 \right)}{\left[\hat{\lambda}^2 + (1 - \hat{\lambda}) \left( \hat{\lambda} + \norm{\nabla \ell_i(w^t)}^2 \right)\right]^2} 
+ \tfrac{1}{2} \frac{(1 - \hat{\lambda}) \hat{\lambda}^2 \ell_i(w)^2}{\left[\hat{\lambda}^2 + (1 - \hat{\lambda}) \left(\hat{\lambda} + \norm{\nabla \ell_i(w^t)}^2 \right)\right]^2} \nonumber \\
&= \frac{\ell_i(w)^2}{2} \frac{(1 - \hat{\lambda})}{\hat{\lambda}^2 + (1 - \hat{\lambda}) \left(\hat{\lambda}^2 + \norm{\nabla \ell_i(w^t)}^2 \right)} \nonumber \enspace.
\end{align}
Consequently the lower bound~\eqref{eq:Ftbound} follows since
\begin{align}
  F_t(w,s) \geq   \min_s F_t(w,s) &\geq \frac{1}{n} \sum_{i=1}^n   \min_s F_{t,i}(w,s) \\
    & \geq  \frac{1}{n} \sum_{i=1}^n  \frac{\ell_i(w)^2}{2} \frac{1 - \hat{\lambda}}{\hat{\lambda}^2 + (1 - \hat{\lambda}) \left(\hat{\lambda} + \norm{\nabla \ell_i(w^t)}^2 \right)} \enspace.
\end{align}

Finally, if we set $\displaystyle s(w) \eqdef \frac{1}{\hat{\lambda}} \max_{i=1,\ldots, n} \ell_i(w)$ then direct computations lead to a cancellation of the first term in 
\begin{align*}
    F_t(w,s(w)) &= \frac{1}{n} \sum_{i=1}^n \tfrac{1}{2}\frac{\left(\ell_i(w) - \max_{i=1,\ldots, n} \ell_i(w) \right)_+^2}{\hat{\lambda} + \norm{\nabla \ell_i(w^t)}^2} + \frac{1 - \hat{\lambda}}{\hat{\lambda}^2} \frac{(\max_{i=1,\ldots, n} \ell_i(w))^2}{2} \\
    &= \frac{1 - \hat{\lambda}}{2\hat{\lambda}^2} \max_{i=1,\ldots, n} \ell_i(w)^2 \enspace.
\end{align*}
\end{proof}


The above lemma shows that $F_t$ is an upper bound for both the robust objective~\eqref{eq:maxbound} and the average squared loss. Thus driving $F_t$ to zero, will drive~\eqref{eq:maxbound} and the average squared loss to zero. We will use this fact after establishing the convergence of $F_t$.

\subsection{Properties of \texttt{SPSL2}}
To prove the convergence of \texttt{SPSL2}, we will use two properties that leverage the \texttt{SGD} interpretation. The first is the following growth property.

\subsubsection{Growth Property}
\begin{lemma}[Growth]
    \label{lem:SPSdamgrowth} 
    Let $\gamma \in [0,  1]$. 
    For every $(w^t,s^t)$ we have that
    \begin{equation}
        \norm{\nabla  F_{t,i}(w^t,s^t)}^2 \leq 2F_{t,i}(w^t,s^t) \enspace.
    \end{equation}
\end{lemma}
\begin{proof}
We recall our notation : $\hat{\lambda} \eqdef 1/(1 + \lambda)$, so that $1-\hat{\lambda} = \lambda/(1+\lambda) \geq 0$. Furthermore, as $\gamma \in [0,  1]$, from Lemma~\ref{lem:sposSPSL2} we have that $s^t \geq 0$ for every $t$. Consequently
\begin{align*}
    \norm{\nabla F_{t,i}(w^t,s^t)}^2 
    &= \norm{\nabla_w  F_{t,i}(w^t,s^t)}^2 +\norm{\nabla_s F_{t,i}(w^t,s^t)}^2 \\
    &= \left( \frac{(\ell_i(w^t) -\hat{\lambda} s^t)_+}{\hat{\lambda}  + \norm{\nabla \ell_i(w^t)}^2}\right)^2 \norm{\nabla \ell_i (w^t)}^2 
    + \left[- \hat{\lambda} \frac{(\ell_i(w^t) -\hat{\lambda} s^t)_+}{\hat{\lambda} + \norm{\nabla \ell_i(w^t)}^2} \norm{\nabla \ell_i (w^t)} + (1-\hat{\lambda}) s^t \right]^2 \\
    &\leq \left( \frac{(\ell_i(w^t) -\hat{\lambda} s^t)_+}{\hat{\lambda}  + \norm{\nabla \ell_i(w^t)}^2}\right)^2 \norm{\nabla \ell_i (w^t)}^2
    + \hat{\lambda}^2 \left( \frac{(\ell_i(w^t) -\hat{\lambda} s^t)_+}{\hat{\lambda}  + \norm{\nabla \ell_i(w^t)}^2}\right)^2 \norm{\nabla \ell_i (w^t)}^2 
    + (1-\hat{\lambda})^2 (s^t)^2 \\
    &\leq \left( \frac{(\ell_i(w^t) -\hat{\lambda} s^t)_+}{\hat{\lambda}  + \norm{\nabla \ell_i(w^t)}^2}\right)^2\left( \norm{\nabla \ell_i(w^t)}^2 + \hat{\lambda}^2 \right) 
    + (1-\hat{\lambda})^2 (s^t)^2 \\
    &\leq \frac{(\ell_i(w^t) -\hat{\lambda} s^t)_+^2}{\hat{\lambda}  + \norm{\nabla \ell_i(w^t)}^2} + (1-\hat{\lambda}) (s^t)^2  \\
    & = 2 F_{t,i}(w^t,s^t) \enspace,
\end{align*}
where in the first inequality we use that $(a-b)^2 \leq a^2 + b^2$ when $a$ and $b$ are non-negative and for the two last inequalities we use that $\hat{\lambda} \in [0, 1]$.
\end{proof}

\subsubsection{Inherited Convexity}
The second property, is that the auxiliary objective is convex whenever the $\ell_i$'s are convex.
\begin{lemma}[Inherited Convexity]\label{lem:SPSdamconvex}
Consider the iterates $(w^t,s^t) \in \R^{d+1}$ given by~\eqref{eq:SPSL2}. 
If $\ell_i$ is convex for $i=1,\ldots, n$ then $F_t$ in~\eqref{eq:SPSL2} is  convex.
\end{lemma}

\begin{proof}
Computing the gradient of $ (\ell_i(w)- \hat{\lambda} s)_+^2$  we have that
\[  \nabla \left( \ell_i(w)- \hat{\lambda} s \right)_+^2 = 2\begin{bmatrix}
\nabla \ell_i(w) \\
-1
\end{bmatrix} \left(\ell_i(w)- \hat{\lambda} s \right)_+ \enspace, \]
where we recall that $\hat{\lambda} \eqdef 1/(1 + \lambda)$.
We recall that $\ones_{\R^+} (.)$ denotes the indicator function over the set of positive real numbers.
Computing the Hessian gives
\begin{align}
 \nabla^2 \left(\ell_i(w)- \hat{\lambda} s \right)_+^2 &=
  2 \ones_{\R^+} (\ell_i(w)- \hat{\lambda} s) \begin{bmatrix}
\nabla \ell_i(w) \\
-1
\end{bmatrix}\begin{bmatrix}
\nabla \ell_i(w) ^\top & - 1
\end{bmatrix}
+
2\begin{bmatrix}
\nabla^2 \ell_i(w) & 0\\
0 & 0
\end{bmatrix}(\ell_i(w)- \hat{\lambda} s)_+ \nonumber \\
& =  
2 \ones_{\R^+} (\ell_i(w)- \hat{\lambda} s) \begin{bmatrix}
\nabla \ell_i(w)\nabla \ell_i(w) ^\top & -\nabla \ell_i(w) \\
-\nabla \ell_i(w) ^\top & 1
\end{bmatrix}
+
2\begin{bmatrix}
\nabla^2 \ell_i(w) & 0\\
0 & 0
\end{bmatrix}(\ell_i(w)- \hat{\lambda} s)_+ \enspace. \label{eq:fialphai2hess}
\end{align}

Now let $\mI_n \in \R^{n\times n}$ be the identity matrix in $\R^{n\times n}$, let
\begin{align}
  \mD_t(w,s) & \eqdef \; 
\diag{\frac{\ones_{\R^+} (\ell_1(w)- \hat{\lambda} s)}{\hat{\lambda} + \norm{\nabla \ell_1(w^t)}^2}, \ldots, \frac{\ones_{\R^+} (\ell_n(w)- \hat{\lambda} s)}{\hat{\lambda} + \norm{\nabla \ell_n(w^t)}^2}} \in \R^{n \times n} \nonumber \\
  \mH_t(w,s) & \eqdef \;  \sum_{i=1}^{n} \nabla^2 \ell_i(w)  \frac{(\ell_i(w)- \hat{\lambda} s)_+}{\hat{\lambda} + \norm{\nabla \ell_i(w^t)}^2}
  \in \R^{d \times d}
\end{align}
and let
\[D F(w) \; \eqdef
 \; \begin{bmatrix}\nabla \ell_1(w), \ldots, \nabla \ell_n(w)\end{bmatrix} \in \R^{d\times n} \enspace.\]

Let $\ones_{n} \in \R^n$ be the vector full of ones of size $n$.
Using~\eqref{eq:fialphai2hess} and by the definition of $F_t$ in~\eqref{eq:objSPSdam} we have that
\begin{equation}\label{eq:ht2hessstar}
 \nabla^2 F_t(w, s) \; = \; 
 \frac{1}{n} \underbrace{ \begin{bmatrix}
 DF(w) \mD_t(w,s) DF(w)^\top & -DF(w)\mD_t(w,s) \ones_{n} \\
-(DF(w)\mD_t(w,s) \ones_{n})^\top & \ones_{n}^\top\mD_t(w,s) \ones_{n} + n (1 - \hat{\lambda})
\end{bmatrix} }_{\eqdef \mM_t(w,s)}
+
\begin{bmatrix}
\mH_t(w,s)  & 0  \\
0 & 0
\end{bmatrix} 
\in \R^{(d+1) \times (d+1)} \enspace,
\end{equation}
where we used that $\nabla^2 \tfrac{1}{2} s^2 =1$ .
Thus the matrix~\eqref{eq:ht2hessstar} is a sum of two terms. 

The second matrix is symmetric positive semi-definite because the losses $\ell_i$ are convex. 
So we just need to show that the first matrix $\mM_t(w,s)$
is  symmetric positive definite. 
Indeed, 
left and right multiplying the above by $[x,\,a]  \in \R^{d+1}$ gives
\begin{eqnarray*}
 \begin{bmatrix}
x & a
\end{bmatrix}^\top 
\mM_t(w,s)
\begin{bmatrix}
x \\ a
\end{bmatrix} 
&\overset{\eqref{eq:ht2hessstar}}{ =} &
 \begin{bmatrix}
x & a
\end{bmatrix}^\top 
\begin{bmatrix}
 D F(w)  \mD_t(w,s) D F(w) ^\top x  - a D F(w) \mD_t(w,s) \ones_{n} \\
-(D F(w) \mD_t(w,s) \ones_{n})^\top x + a (\ones_{n}^\top\mD_t(w,s)\ones_{n} + n (1 - \hat{\lambda}) ).
\end{bmatrix}  \\
&= &
\norm{\mD_t(w,s)^{1/2}D F(w) ^\top x }^2 -2a (D F(w) \mD_t(w,s) \ones_{n})^\top x \\
& & + \; a^2 (\ones_{n}^\top\mD_t(w,s)\ones_{n} + n (1 - \hat{\lambda}) ) \\
& = & \norm{\mD_t(w,s)^{1/2}(D F(w) ^\top x -\ones_{n} a) }^2 - a^2 \norm{\mD_t(w,s)^{1/2} \ones_{n} }^2 \\
& & + \; a^2 (\ones_{n}^\top\mD_t(w,s)\ones_{n} + n (1 - \hat{\lambda}) ) \\
&= & \norm{\mD_t(w,s)^{1/2}(D F(w) ^\top x -\ones_{n} a) }^2 + n (1 - \hat{\lambda}) a^2 \geq 0 \enspace.
\end{eqnarray*}

Thus the Hessian~\eqref{eq:ht2hessstar} is non-negative for every $(x,a)$ from which we conclude that the Hessian $\nabla^2  F_t(w,s)  $ in~\eqref{eq:ht2hessstar} is positive semi-definite. 
\end{proof}

\subsection{\texttt{SGD}  style convergence results}
These two proceeding properties allow us to establish an \texttt{SGD} style proof with different learning rates $\gamma_t$ at each iteration.
\begin{lemma} \label{lem:convergence}
Let $z^* =(w^\star,s^{\star}) \in \R^{d+1}$ be given  and let $\ell_i$ be convex for every $i=1,\ldots, n.$ Let $\gamma_t \in [0, 
1]$ be the learning rate used at iteration $t$ and consider the iterates $z^t = (w^t,s^t)$ given by~\eqref{eq:SPSL2SGD}. It follows that
\begin{align}
     \frac{\sum_{t=0}^{T-1}\gamma_t(1-\gamma_t)\E{F_t(z^t) - F_t(z^*)}}{\sum_{t=0}^{T-1}\gamma_t(1-\gamma_t)} & \leq \tfrac{1}{2}\frac{\norm{z^{0} -z^*}^2}{\sum_{t=0}^{T-1}\gamma_t(1-\gamma_t)} +\frac{\sum_{t=0}^{T-1}\gamma_t^2 \E{F_{t}(z^*)}}{\sum_{t=0}^{T-1}\gamma_t(1-\gamma_t)} \enspace. \nonumber 
\end{align}
Or equivalently by re-arranging the above gives
\begin{align}
       \frac{\sum_{t=0}^{T-1}\gamma_t(1-\gamma_t)\E{F_t(z^t) }}{\sum_{t=0}^{T-1}\gamma_t(1-\gamma_t)} &\leq \tfrac{1}{2}\frac{\norm{z^{0} -z^*}^2}{\sum_{t=0}^{T-1}\gamma_t(1-\gamma_t)} +\frac{\sum_{t=0}^{T-1}\gamma_t\E{F_t(z^*)} }{\sum_{t=0}^{T-1}\gamma_t(1-\gamma_t) } \enspace. \nonumber 
\end{align}
And finally, if $\gamma_t \eqdef \gamma \in [0, 
1]$, it simplifies in
\begin{equation}
    \label{eq:convSPSdam}
    \frac{\sum_{t=0}^{T-1}\E{F_t(z^t) }}{T} \leq \frac{1}{2 T}\frac{\norm{z^{0} -z^*}^2}{\gamma(1-\gamma)} + \frac{\sum_{t=0}^{T-1}\E{F_t(z^*)} }{T(1-\gamma) } \enspace. 
\end{equation}
\end{lemma}
\begin{proof}
Let $z = (w,s)$. Expanding the squares we have that
\begin{align}
    \norm{z^{t+1} -z^*}^2  &\leq   \norm{z^{t} -z^*}^2  - 2\gamma_t \dotprod{\nabla F_{t,i}(z^t), z^t-z^*} +\gamma_t^2 \norm{\nabla F_{t,i}(z^t)}^2 \nonumber \\
    &\overset{\mbox{Lemma~\ref{lem:SPSdamgrowth}}}{\leq} \norm{z^{t} -z^*}^2  - 2\gamma_t \dotprod{\nabla F_{t,i}(z^t), z^t-z^*} +2\gamma_t^2F_{t,i}(z^t) \enspace.
\end{align} 
Taking expectation conditioned on $t$ gives
\begin{align}
   \EE{t}{ \norm{z^{t+1} -z^*}^2 } &\leq  
    \norm{z^{t} -z^*}^2 - 2\gamma_t \dotprod{\nabla F_{t}(z^t), z^t-z^*} +2\gamma_t^2F_{t}(z^t) \enspace.
\end{align} 
Using Lemma~\ref{lem:SPSdamconvex} we have that 
\[F_t(z^*) \geq F_t(z^t)+ \dotprod{\nabla F_t(z^t), z^* - z^t} \enspace,\]
which plugged into the above gives

\begin{align}
   \EE{t}{ \norm{z^{t+1} -z^*}^2 } &\leq  
    \norm{z^{t} -z^*}^2 - 2\gamma_t (F_t(z^t) - F_t(z^*)) +2\gamma_t^2F_{t}(z^t) \nonumber \\
        & =     \norm{z^{t} -z^*}^2 - 2\gamma_t (F_t(z^t) - F_t(z^*)) +2\gamma_t^2(F_{t}(z^t)-F_{t}(z^*)) +2\gamma_t^2 F_{t}(z^*)\nonumber \\ 
    & =     \norm{z^{t} -z^*}^2 - 2\gamma_t(1-\gamma_t) (F_t(z^t) - F_t(z^*)) +2\gamma_t^2 F_{t}(z^*) \enspace.
\end{align} 

Taking expectation and re-arranging gives
\begin{align*}
    2\gamma_t(1-\gamma_t) \E{F_t(z^t) - F_t(z^*)}
     & \leq \E{\norm{z^{t} -z^*}^2} - \EE{}{ \norm{z^{t+1} -z^*}^2 }+2\gamma_t^2 \E{F_{t}(z^*)} \enspace.
\end{align*}
Summing up from $t=0, \ldots, T-1 $, using telescopic cancellation and dividing by $\sum_{t=0}^{T-1}\gamma_t(1-\gamma_t)$ gives
\begin{align*}
    2 \frac{\sum_{t=0}^{T-1}\gamma_t(1-\gamma_t)\E{F(z^t) - F_t(z^*)}}{\sum_{t=0}^{T-1}\gamma_t(1-\gamma_t)} & \leq \frac{\norm{z^{0} -z^*}^2}{\sum_{t=0}^{T-1}\gamma_t(1-\gamma_t)}+2\frac{\sum_{t=0}^{T-1}\gamma_t^2 \E{F_{t}(z^*)}}{\sum_{t=0}^{T-1}\gamma_t(1-\gamma_t)} \enspace.
\end{align*}
\end{proof}

\subsection{Extra convergence theorem of \texttt{SPSL2} controlled by the average loss}

We now specialize Lemma~\ref{lem:convergence} into the two theorems presented in the main text. 

\begin{theorem}
\label{theo:SPSDAMregret}
Consider the setting of Lemma~\ref{lem:convergence}. 
Let 
\begin{equation}
    \label{eq:first_choice_w_star_spsl2}
    w^\star \in \argmin_{w\in \R^d} \ell (w) \eqdef \frac{1}{n} \sum_{i=1}^n \ell_i(w)^2
\end{equation}
and let $z^* = (w^\star,0)$.
If $\ell_i$ is non-negative and $\sigma$--Lipschitz for every $i\in \{1,\ldots, n\}$ then
\begin{align} \label{eq:convSPSdamv2}
    \E{\sum_{i=1}^n \ell_i(\bar{w}^T)^2 } 
    & \leq \frac{1}{ C (\hat{\lambda})T}\frac{\norm{w^{0} -w^\star}^2+(s^0-s^{\star})^2}{\gamma(1-\gamma)} 
    + \frac{1}{C (\hat{\lambda})}\frac{1}{\hat{\lambda} (1-\gamma)} \frac{1}{n}\sum_{i=1}^n\ell_i(w^\star)^2 \enspace. 
\end{align}
where $ \bar{w}^T = \sum_{t=0}^{T-1} w^t/T $ and $C (\hat{\lambda})$ is defined in~\eqref{eq:C}.
\end{theorem}
\begin{proof}
Let first $z^* = (w^\star, 0)$ where $w^\star$ is defined in~\eqref{eq:first_choice_w_star_spsl2} by
\begin{equation*}
    w^\star \in \argmin_{w\in \R^d} \frac{1}{n} \sum_{i=1}^n \ell_i(w)^2 \enspace.
\end{equation*}
From the definition of $F_t$ in~\eqref{eq:objSPSdam} we have that
\begin{equation} \label{eq:tempsos8h4s4}
    F_t(w^\star,0) \;=\; \frac{1}{2n}\sum_{i=1}^n\frac{(\ell_i(w^\star))_+^2}{ \hat{\lambda} + \norm{\nabla \ell_i(w^t)}^2 }  \leq 
    \frac{1}{2n\hat{\lambda}}\sum_{i=1}^n \ell_i(w^\star)^2 \enspace.
\end{equation}
Furthermore from Lemma~\ref{lem:objSPSdambound} we have that 
\begin{equation}\label{eq:tempsos8h4s44}
    \frac{C(\hat{\lambda})}{n}\sum_{i=1}^n \ell_i(w)^2 \leq 2 F_t(w,s) \enspace.
\end{equation}
Plugging $\gamma_t \eqdef \gamma$ and the above two bounds into~\eqref{eq:convSPSdam}, using that $\ell_i(w)^2$ is convex because $\ell_i(w)$ is positive and convex, together with Jensen's inequality with respect to the average  $\displaystyle \bar{w}^T \eqdef \frac{1}{T} \sum_{t=0}^{T-1} w^t$ over time gives
\begin{align}
\frac{C(\hat{\lambda})}{n}\sum_{i=1}^n \ell_i(\bar{w}^T)^2 
    &\overset{\mbox{Jensen}}{\leq} \; \frac{C(\hat{\lambda})}{n}\sum_{i=1}^n \frac{1}{T} \sum_{t=0}^{T-1} \ell_i(w^t)^2  \; \overset{\eqref{eq:tempsos8h4s44}}{\leq} \; \frac{2}{T} \sum_{t=0}^{T-1}F_t(z^t) \enspace.
\end{align}

Now taking expectation gives
   \begin{align} 
    \frac{C(\hat{\lambda})}{n}\sum_{i=1}^n \E{\ell_i(\bar{w}^T)^2} & \leq 
    \frac{2}{T} \sum_{t=0}^{T-1}\E{F_t(z^t)} \nonumber \\ 
    &\overset{\eqref{eq:convSPSdam}}{\leq} \frac{1}{T} \frac{\norm{z^{0} -z^*}^2}{\gamma(1-\gamma)}
    + \frac{2}{1-\gamma} \frac{1}{T} \sum_{t=0}^{T-1}\E{F_t(z^*)} \label{eq:intermediate_line_conv_proof_spsl2} \\
    &\overset{\eqref{eq:tempsos8h4s4}}{\leq} \frac{1}{ T}\frac{\norm{z^{0} -z^*}^2}{\gamma(1-\gamma)} + \frac{1}{\hat{\lambda} (1-\gamma)} \frac{1}{n}\sum_{i=1}^n\ell_i(w^\star)^2 \enspace. \nonumber 
\end{align}

Thus dividing by $C(\hat{\lambda})$ and subtracting $ \frac{1}{n}\sum_{i=1}^n \ell_i(w^\star)^2$ from both sides gives
\begin{align*}
    \frac{1}{n}\sum_{i=1}^n( \ell_i(\bar{w}^T)^2- \ell_i(w^\star)^2) & \leq 
\frac{1}{ C(\hat{\lambda})T}\frac{\norm{z^{0} -z^*}^2}{\gamma(1-\gamma)}+  \left( \frac{1}{C(\hat{\lambda})} \frac{1}{\hat{\lambda} (1-\gamma)} -1\right)\frac{1}{n}\sum_{i=1}^n\ell_i(w^\star)^2 \enspace.
\end{align*}
\end{proof}

This Theorem~\ref{theo:SPSDAMregret} can be seen as an extension of the regret analysis for \texttt{PA} methods~\cite{Crammer06}. Indeed, if each $\ell_i(w^t)$ is instead an online loss $\ell_t(w^t)$ which could be chosen adversarially, using the same proof for establishing Theorem~\ref{theo:SPSDAMregret} we can show that
%
%
\begin{equation*}
    \frac{1}{T}\sum_{t=0}^{T-1} \E{\ell_t(\bar{w}^T)^2} 
    \leq \frac{1}{C (\hat{\lambda}) T}\frac{\norm{z^{0} - z^*}^2}{\gamma(1-\gamma)} 
    + \frac{1}{C (\hat{\lambda})} \frac{1}{\hat{\lambda} (1-\gamma) }\frac{1}{T}\sum_{t=0}^{T-1} \ell_t(w^\star)^2 \enspace,
\end{equation*}
which, up to constants, is a regret proportional to the optimal regret. In this sense, our result is an extension of Theorem 5 in~\cite{Crammer06} from hinge-loss over linear models, to general nonlinear convex models.



\subsection{Proof of Theorem~\ref{theo:SPSDAMfmax_main}}

Here we prove Theorem~\ref{theo:SPSDAMfmax_main} and also give some additional details.
\begin{theorem}[Theorem~\ref{theo:SPSDAMfmax_main}]
\label{theo:SPSDAMfmax}
Consider the setting of Lemma~\ref{lem:convergence}. 
Let 
\begin{align}
    w^\star &\in \argmin_{w\in \R^n} \max_{i=1,\ldots, n} \ell_i(w)^2 \enspace, \nonumber \\
    s^{\star} &= \frac{1}{\hat{\lambda}} \max_{i=1,\ldots, n} \ell_i(w^\star)^2 \enspace. 
    \label{eq:2nd_choice_zstar_spsL1}
\end{align}
and let $z^* = (w^\star,s^{\star})$.
If $\ell_i$ is non-negative and $\sigma$--Lipschitz for every $i\in \{1,\ldots, n\}$ then
\begin{align*} 
    \frac{1}{n}\sum_{i=1}^n \E{\ell_i(\bar{w}^T)^2}  
    & \leq \frac{\norm{z^{0} -z^*}^2}{ \gamma(1-\gamma)}\frac{1}{C (\hat{\lambda}) T} +  \frac{1}{C (\hat{\lambda})} \frac{1-\hat{\lambda}}{\hat{\lambda}^2 (1-\gamma)} \max_{i=1,\ldots, n} \ell_i(w^\star)^2 \enspace. 
\end{align*}
where $ \bar{w}^T = \sum_{t=0}^{T-1} w^t/T$ and $C (\hat{\lambda})$ is defined in~\eqref{eq:C}. 
\end{theorem}
\begin{proof}
Using $z^* = (w^\star, s^{\star})$ defined in~\eqref{eq:2nd_choice_zstar_spsL1} 
, the result in~\eqref{eq:FstarSPSdam_appdx} gives
\begin{equation*}
    F_t (z^*) = \frac{1 - \hat{\lambda}}{2\hat{\lambda}^2} \max_{i=1,\ldots, n} \ell_i(w^\star)^2 \enspace.
\end{equation*}
Following the proof of Theorem~\ref{theo:SPSDAMregret}
up to equation~\eqref{eq:intermediate_line_conv_proof_spsl2} and using the above gives us
\begin{align*}
    \frac{C(\hat{\lambda})}{n}\sum_{i=1}^n 
    \E{\ell_i(\bar{w}^T)^2} &\leq \frac{1}{T} \frac{\norm{z^{0} -z^*}^2}{\gamma(1-\gamma)} 
    + \frac{2}{1-\gamma} \frac{1}{T} \sum_{t=0}^{T-1}\E{F_t(z^*)} \\
    &= \frac{1}{T} \frac{\norm{z^{0} -z^*}^2}{\gamma(1-\gamma)} 
    + \frac{1 - \hat{\lambda}}{\hat{\lambda}^2 (1-\gamma)} \max_{i=1,\ldots, n} \ell_i(w^\star)^2 \enspace, 
\end{align*}
which, after taking expectation, concludes the proof.
\end{proof}

\section{\texttt{ALI-G} and \texttt{SPS}$_{\max}$} \label{sec:Ap-ALIG-SPS-default}
Here we give further details on the \texttt{ALI-G} and \texttt{SPS}$_{\max}$ methods, and their respective parameter choices.
Both methods are very closely related and fit the format 
\begin{equation}
     w^{t+1} \;= \; w^t -\gamma \min\left\{\frac{\ell_i(w^t) - \ell_i^*}{\norm{\nabla \ell_i(w^t)}^2+\varepsilon}, \; \lambda \right\} \nabla \ell_i(w^t) \enspace,
\end{equation}
where $ \ell_i^* >0$, $\varepsilon \geq 0$ and $\lambda >0$ are the different parameters. 
The only significant distinction between \texttt{ALI-G} and \texttt{SPS}$_{\max}$ is that a nonzero $\ell_i^*$ was used and introduced in \texttt{SPS}$_{\max}$ ~\cite{SPS}. 

In~\cite{SPS} the authors provide default parameter settings. In~\cite{ALI-G} the parameter $\lambda$ is tuned for each set of experiments.

\paragraph{Default settings for \texttt{ALI-G}:} $\varepsilon =10^{-5}$,  $\lambda=0.1$, $\gamma =1.$
\paragraph{Default settings for \texttt{SPS}$_{\max}$:} For convex problems $\gamma =0.5$,  $\lambda=100$, $\varepsilon =10^{-8}$ and $\ell_i^* = \min_{w} \ell_i(w).$ For  DNNs  $\gamma =0.2$.

For most models we have that $\ell_i^* =0$ and thus the two methods are in practice the same method. Indeed, in all of our experiments we text problems where $\ell_i^* =0,$ thus the only difference between the methods are the parameter settings.

\section{Experiments Details and Additional Experiments}

\label{sec:ap-experiments}

Here we study the four slack methods  \texttt{SPS}$_{\max}$~\eqref{eq:SPSmax}, \texttt{SPS}$_{dam}$~\eqref{eq:SPSdam},  \texttt{SPSL2}~\eqref{eq:SPSL2} and \texttt{SPSL1}~\eqref{eq:SPSL1} in the simplified setting of logistic regression.
That is, we use binary classification using  logistic regression
\begin{equation} \label{eq:fgenlin}
\ell (w) = \frac{1}{n} \sum \limits_{i=1}^n \phi (x_i^\top w) +\frac{reg}{2}\norm{w}_2^2 \;
\end{equation}
where 
 $
 \phi_i(t) \, =\, \ln\left(1+ e^{-y_i t} \right),
$ 
$(x_i,y_i)\in \R^{d+1}$ are the features and labels for $i=1, \ldots, n$, and $reg >0$ is the regularization parameter.

The data sets we used were 
colon-cancer~\cite{coloncancer}, \texttt{mushrooms}~\cite{uci}, \texttt{phishing}~\cite{uci} and \texttt{cod-rna}~\cite{cod-rna}.
The dimensions and properties of these datasets a can be found in Table~\ref{tab:datasetsconv}. We chose five data sets with varying degrees of over-parametrization and interpolation. By varying degrees of interpolation, we refer to the size of $f(w^\star)$ which is in the last column in Table~\ref{tab:datasetsconv}.

Logistic regression interpolates the first two 
 datasets \texttt{colon-cancer} and \texttt{mushrooms},see Table~\ref{tab:datasetsconv}.  The main difference between these two data is that the resulting model is over-parametrized for \texttt{colon-cancer}  and under-parametrized  for \texttt{mushrooms}. 
 The remaining two data sets 

\begin{table}[!ht]
\centering
\begin{tabular}{c|ccc|c}
   \toprule
 dataset     & $d$  & $n$ & $L_{\max}$  &  $\ell^*$ \\
 \midrule
 \texttt{colon-cancer} &   $2001$ &  $62$ & $137.8$    &0.0  \\ 
 \texttt{mushrooms}   & $112$            &  $8124$ &    $5.5$    &  0.0  \\ 
  \texttt{phishing}    & $68$             & $11055$   &   $7.75$  &    0.142 \\
  \texttt{cod-rna} & $8$             & $59535$   &   $7.75$  &    0.186 \\
 \bottomrule
\end{tabular}
\caption{Binary datasets used in the logistic regression experiments where $\ell^* = \min_w  \frac{1}{n} \sum \limits_{i=1}^n \phi (x_i^\top w) .$}
\label{tab:datasetsconv}
\end{table}

\subsection{Logistic regression experiments}


To better understand the advantages of regularizing the slack, as we have done in the design of \texttt{SPSL1} and \texttt{SPSL2}, we compare these two methods in a series of experiments to their unregularized counterparts \texttt{SPS}$_{\max}$ and \texttt{SPS}$_{dam}$. Here, we use the simpler setting of logistic regression, the details of which can be found in the appendix in Section~\ref{sec:ap-experiments}. 

\newcommand\smallfigsize{3.2cm}

\begin{figure}
    \centering
\begin{minipage}{0.24\linewidth}
\centering
\includegraphics[width=\linewidth]{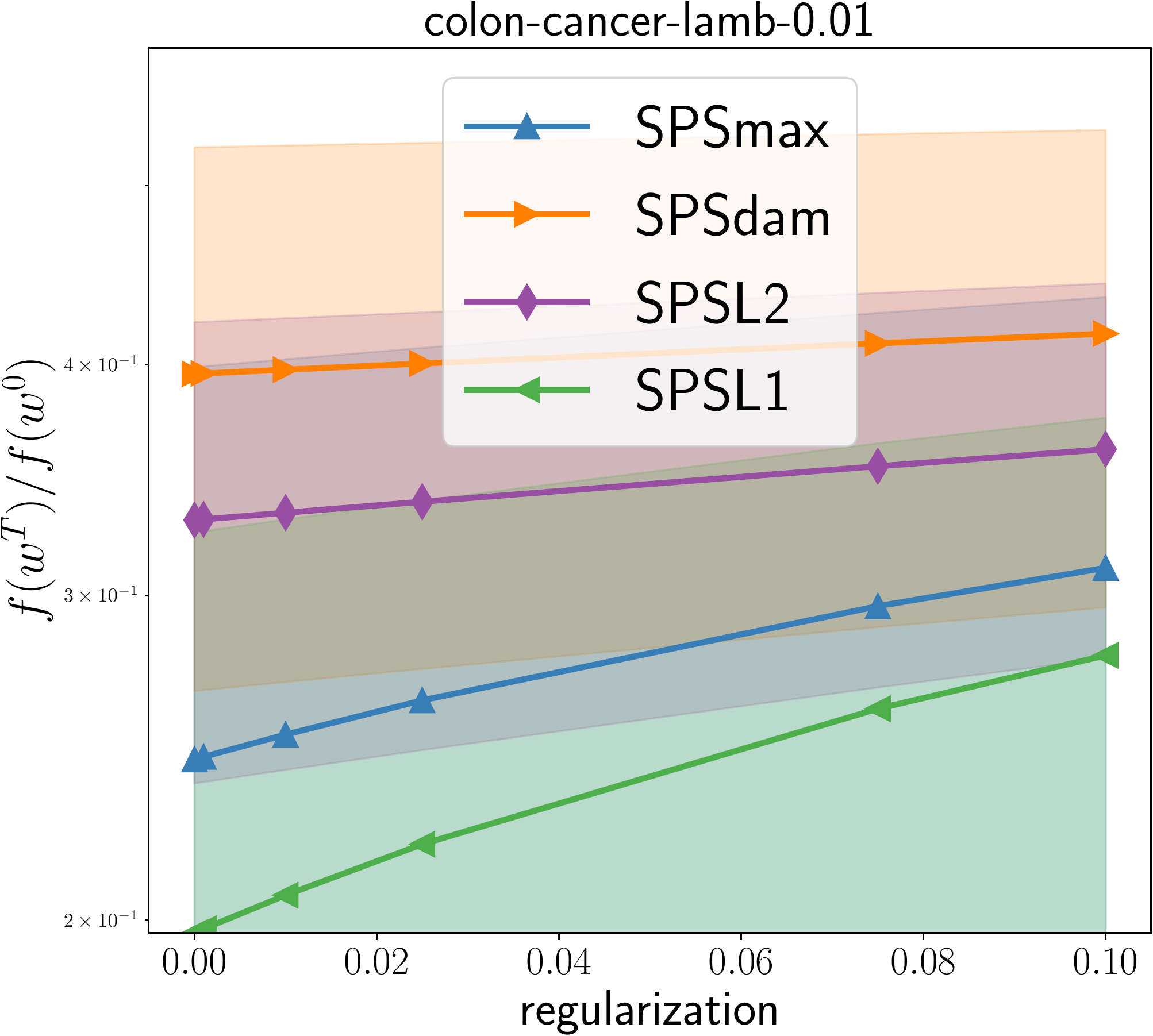}
\centerline{\footnotesize{(a) colon-cancer $\lambda = 0.01$}}
\end{minipage}
\begin{minipage}{0.24\linewidth}
\centering
\includegraphics[width=\linewidth]{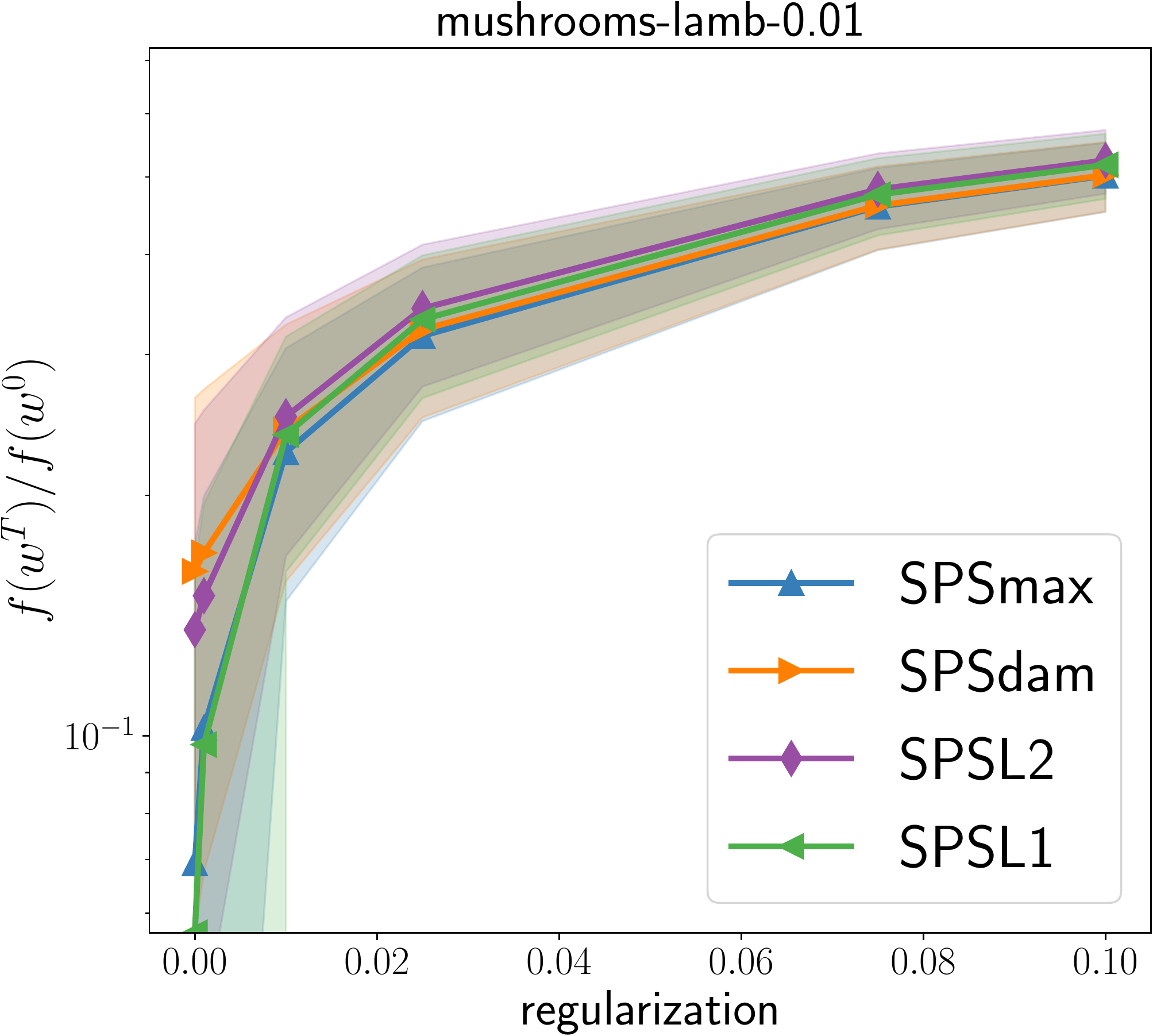}
\centerline{\footnotesize{(b) mushrooms $\lambda = 0.01$}}
\end{minipage}
   \begin{minipage}{0.24\linewidth}
\centering
\includegraphics[width=\linewidth]{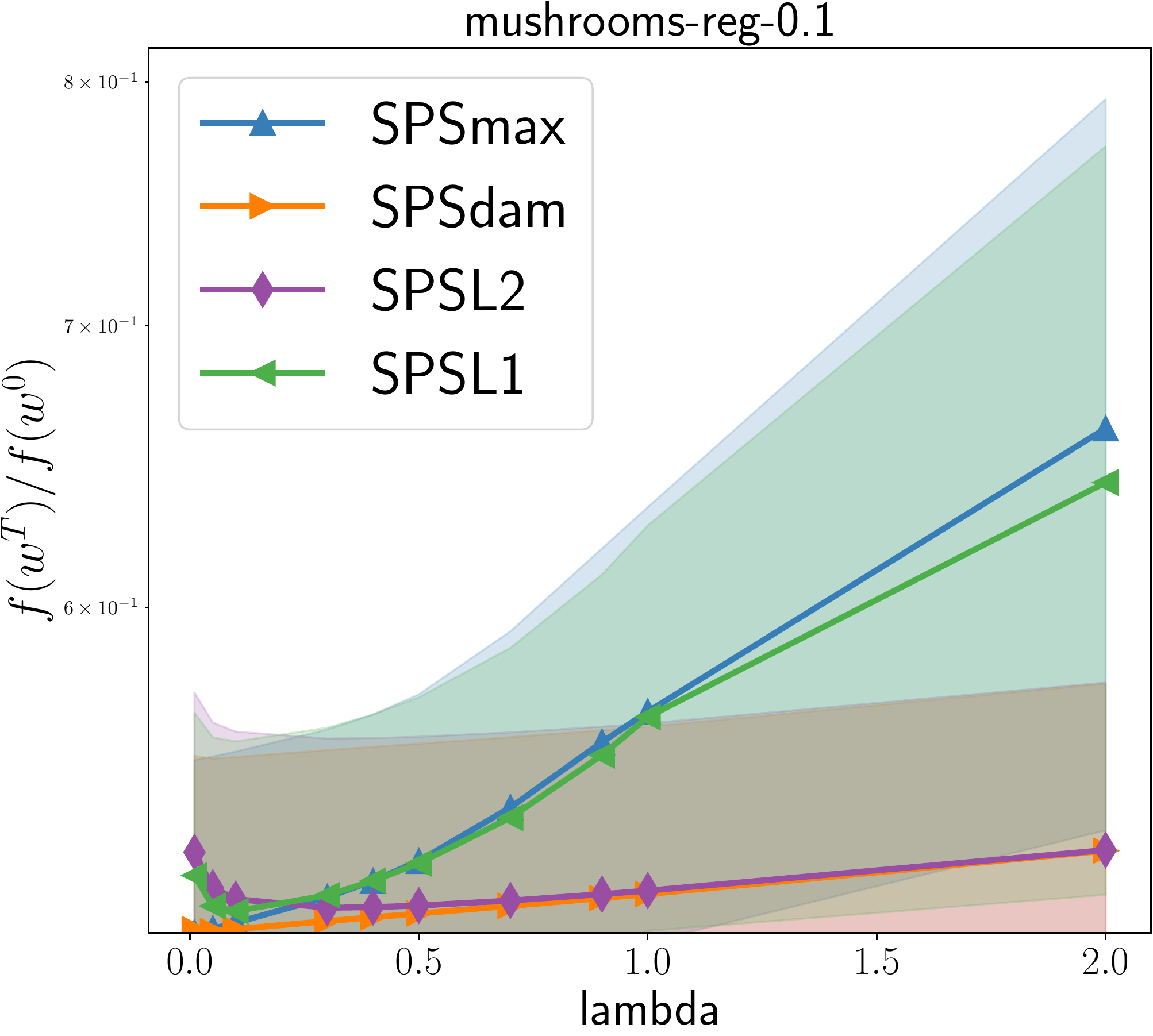}
\centerline{\footnotesize{(c) mushrooms $reg = 0.1$}}
\end{minipage} 
   \begin{minipage}{0.24\linewidth}
\centering
\includegraphics[width=\linewidth]{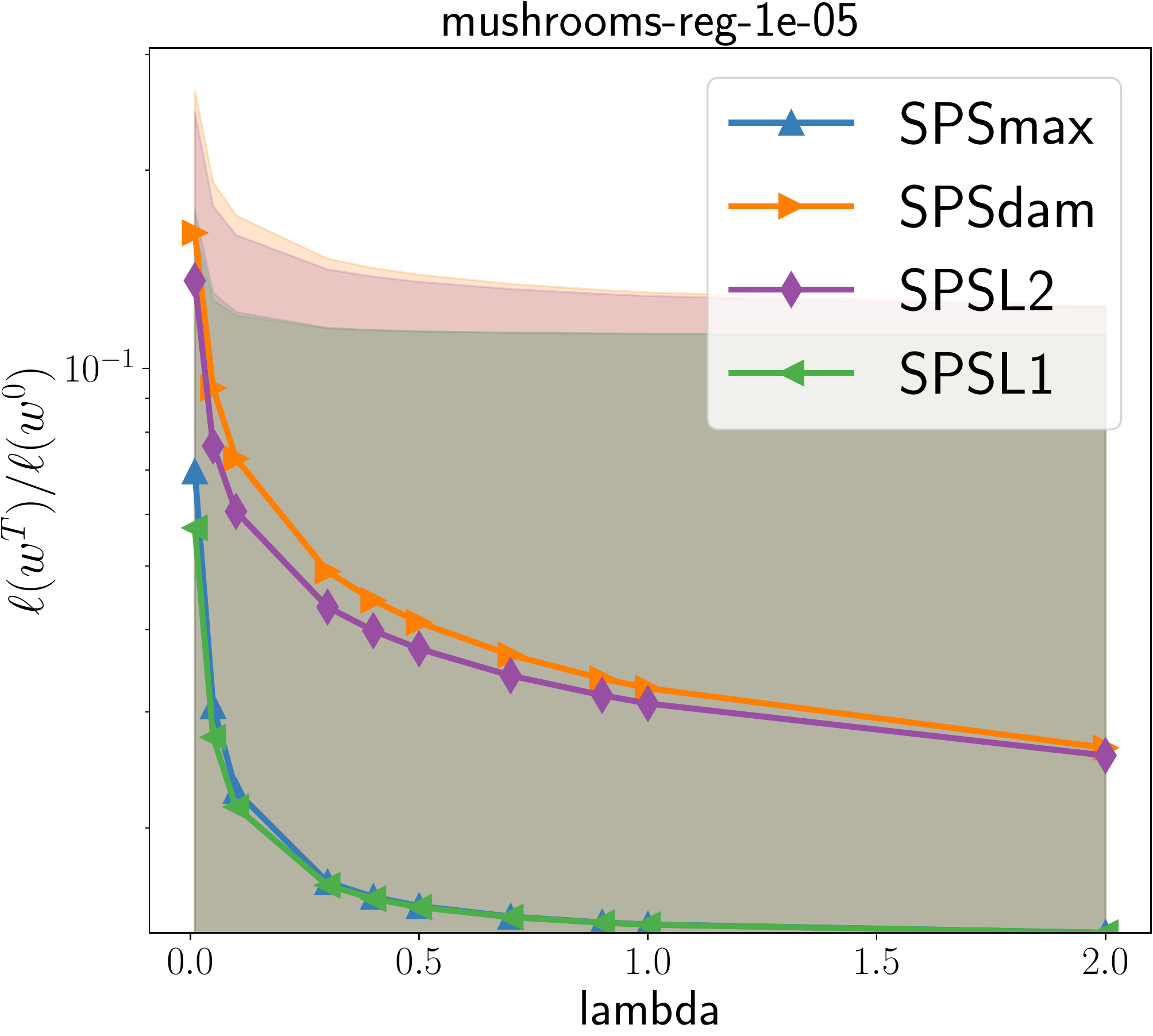}
\centerline{\footnotesize{(d) mushrooms $reg = 10^{-5}$}}
\end{minipage} 
    \caption{Relative suboptimality ($\ell(w^T)/\ell(w^0)$ y-axis)  for each method when using a fixed regularization $reg$ (x-axis) (Figures (a) and (b)) or a fixed $\lambda$ (Figures (c) and (d)). Each method was given a budget of 100 epochs on logistic regression.
    }
    \label{fig:regsens-main}
\end{figure}
 
First, we compare the methods in terms of sensitivity to the interpolation. We can generate problems that are further from interpolation by increasing the L2 regularization in logistic regression, as we have done in Figure~\ref{fig:regsens-main}.
For most data sets and settings of $\lambda$, the different methods showed to be approximately equally sensistive to increasing the regularization, excluding one case. In the case of \texttt{colon-cancer} data set with $\lambda =0.01$, we found that \texttt{SPSL1} and \texttt{SPSL2} were significantly faster across all regularization parameters than their unregularized counterparts \texttt{SPS}$_{\max}$ and \texttt{SPS}$_{dam}$, respectively, as we can see in the left of Figure~\ref{fig:regsens-main}.

We then tested the sensitivity of each method to $\lambda$ in Figure~\ref{fig:regsens-main}. Through these experiments we have two consistent findings. Our first finding is that when the regularization is large (far from interpolation), all methods benefit from a smaller $\lambda$ around $0.1$, see the left of Figure~\ref{fig:regsens-main}. Similarly, large $\lambda$ around $2$ was better when using a small regularization, see the right of Figure~\ref{fig:regsens-main}. Our second finding is that generally both \texttt{SPSL2} and \texttt{SPS}$_{dam}$ benefited from a larger $\lambda$, with $\lambda =1$ being a good default setting, one we will use later.

\subsection{Sensitivity to $\lambda$}
\label{sec:Ap-sens-lambda}

Here we investigate the sensitivity of the slack methods to their only parameter $\lambda$, which we call the slack parameter.

Our overall finding is that for problems that are close to interpolation, all methods work well with $\lambda$ close to $1$ or larger. For problems that are far from interpolation, we find that  they work much better for $\lambda$ close to zero. See Figure~\ref{fig:lambdaL1}  where we test each method for a fixed $\lambda$ (x-axis) and then run each  method for 100 epoches and register the final  training error (suboptimality $\ell(w^T)/\ell(w^0)$). The top row of figures have a larger regularization of $reg = 0.1$, and thus are all far from interpolation. For this row of problems we can see that $\lambda =0.01$ is best for both \texttt{SPSL1} and \texttt{SPS}$_{\max}$, whereas around $\lambda = 0.4$ is best for \texttt{SPSL2} and \texttt{SPS}$_{dam}.$ In the bottom row of Figure~\ref{fig:lambdaL1} we see that a large $\lambda =2$ is good for all methods.

\begin{figure}
    \centering
    \includegraphics[width=\figsizefour]{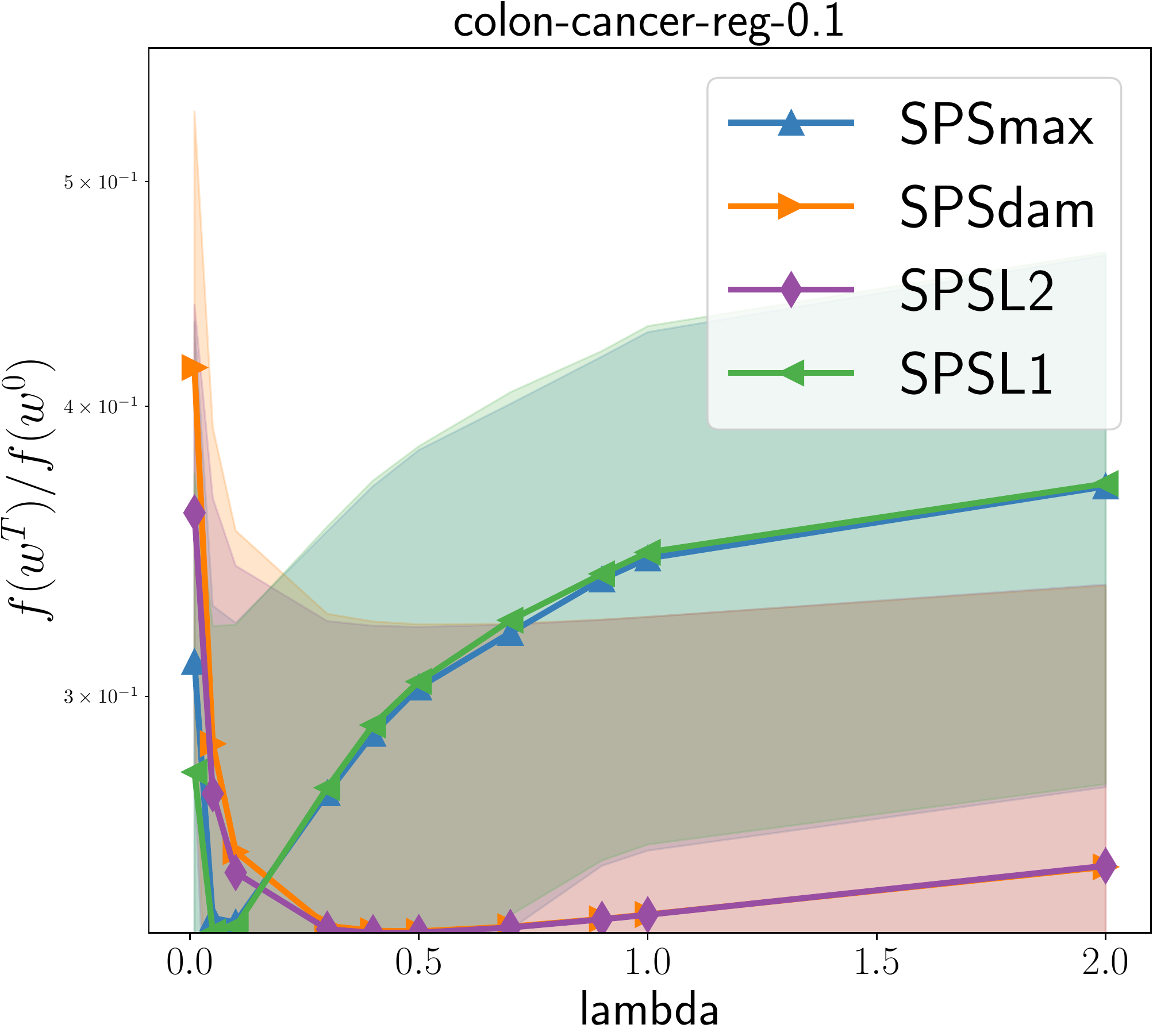} 
    \includegraphics[width=\figsizefour]{mushrooms-reg-0.1}
     \includegraphics[width=\figsizefour]{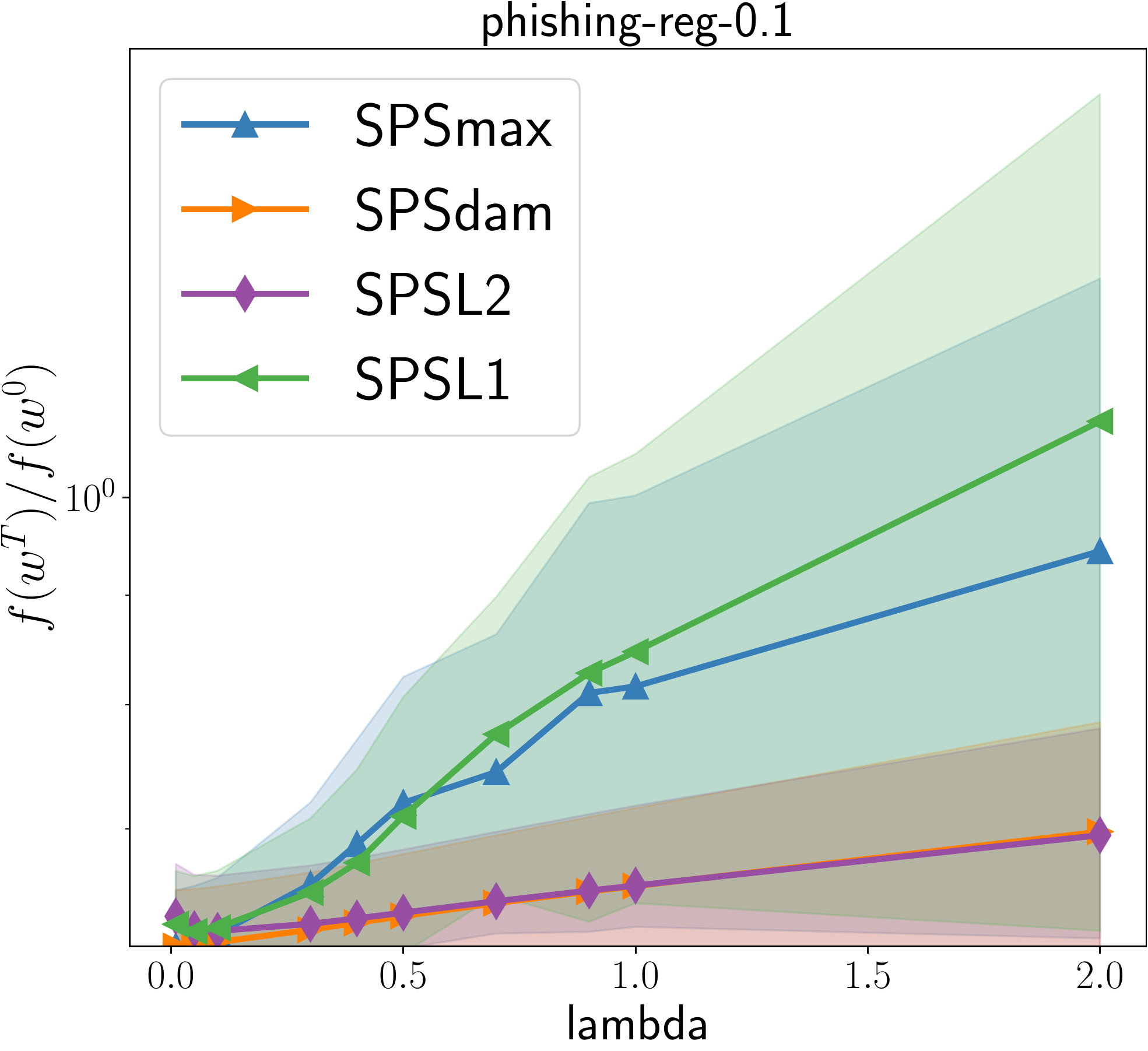} 
    \includegraphics[width=\figsizefour]{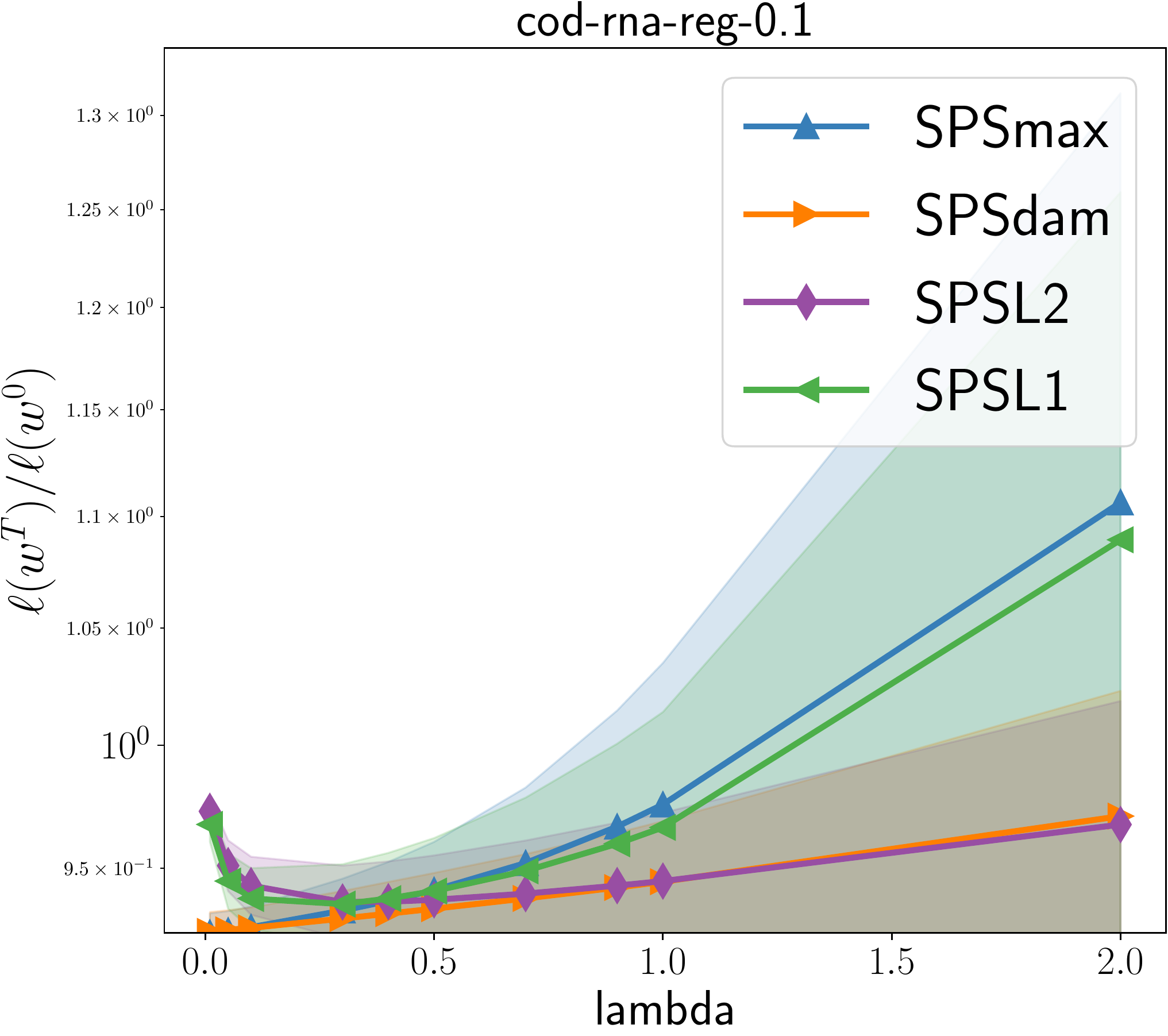}
    \includegraphics[width=\figsizefour]{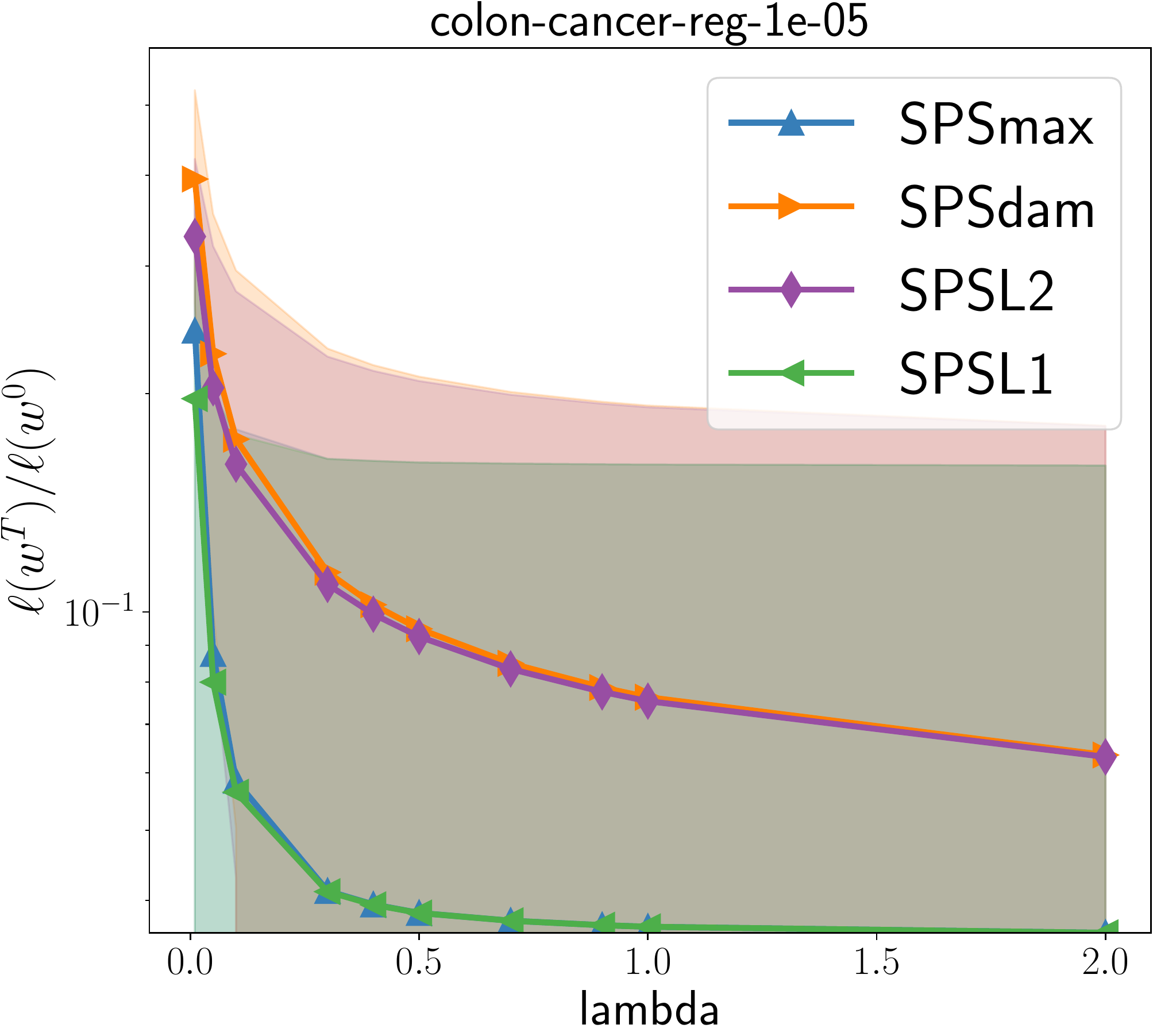} 
    \includegraphics[width=\figsizefour]{mushrooms-reg-1e-05}
     \includegraphics[width=\figsizefour]{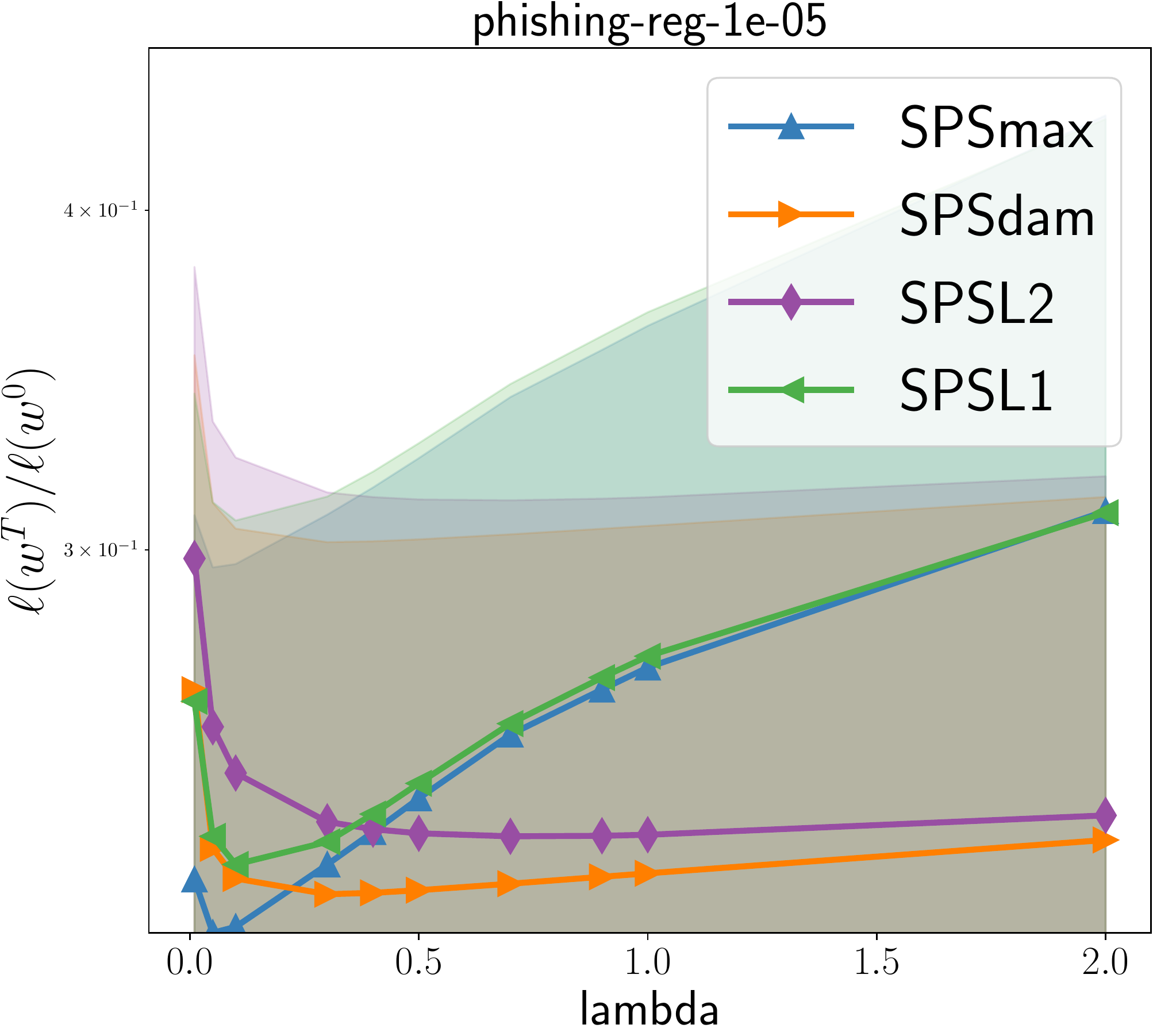} 
    \includegraphics[width=\figsizefour]{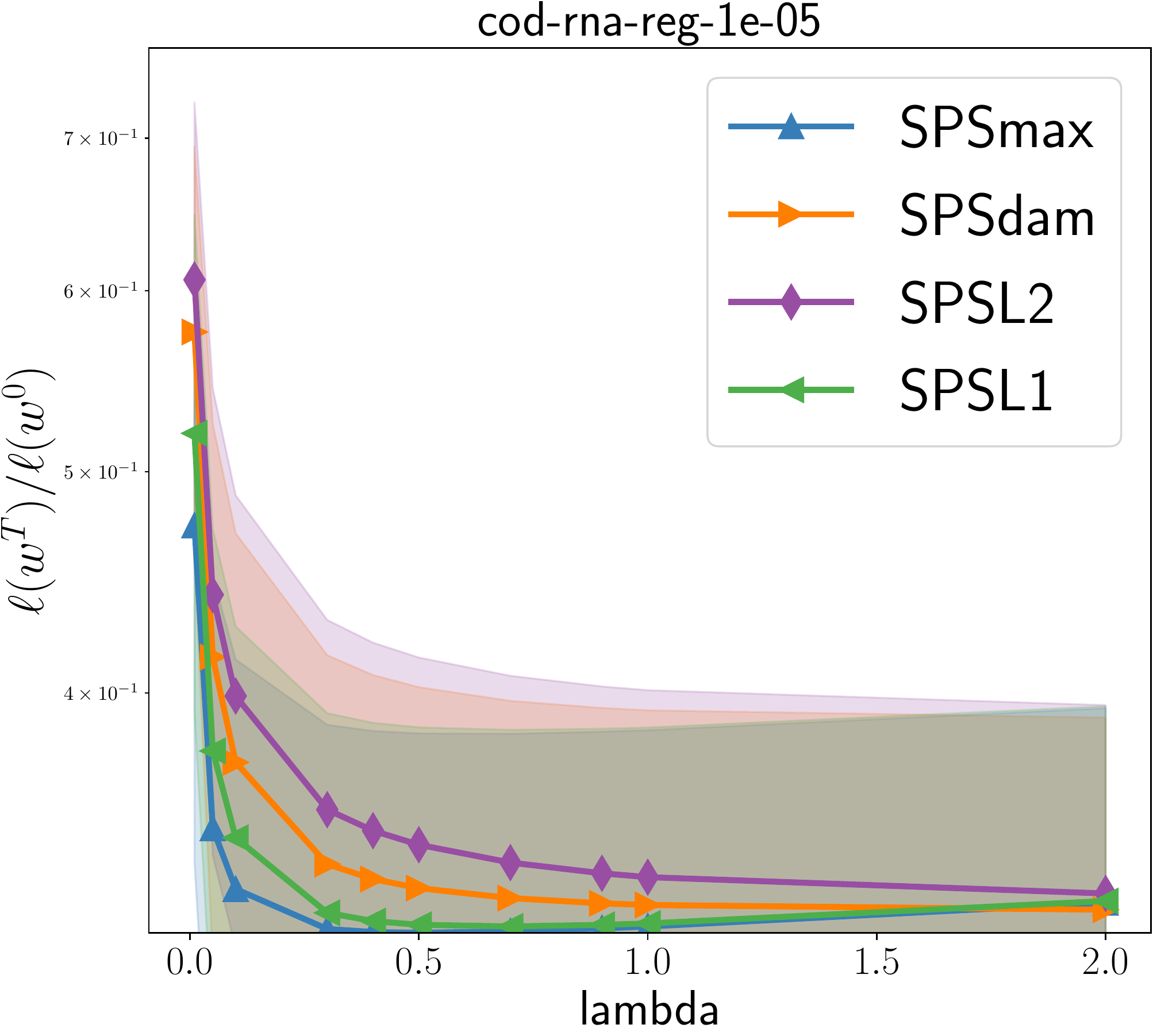}
    \caption{Relative suboptimality ($\ell(w^T)/\ell(w^0)$) for each method with slack parameter $\lambda$ (x-axis) after given a budget of 100 epochs on logistic regression. For the top row of figures, we used a regularization of $reg = 0.1$. For the bottom row, we used a regularization of $reg = 10^{-5}$. The data sets used from left to right are  \texttt{colon-cancer},  \texttt{mushrooms}, \texttt{phishing}, and \texttt{cod-rna}. We find consistently that for problems far from interpolation, $\lambda$ small works best, while for    problems close to interpolatiom $\lambda$ large works best. }
    \label{fig:lambdaL1}
\end{figure}

\subsection{Sensitivity to Interpolation}

Here we use regularization as a proxy for sensitivity to interpolation and compare the four slack methods with respect to their sensitivity to regularization, see Figure~\ref{fig:regsens}. In the top row of~\ref{fig:regsens} we have fixed $\lambda =1$ and in the bottom row we have fixed $
\lambda = 0.01.$ 
Our overall finding is that the methods are almost equally sensitive, with only \texttt{SPSL1} having an advantage in \texttt{colon-cancer} with $\lambda =0.01$. In all the other plots, \texttt{SPSL1} overlaps with \texttt{SPS}$_{\max}$ and \texttt{SPSL2} overlaps with \texttt{SPS}$_{dam}$, with only minor differences.
\begin{figure}
    \centering
    \includegraphics[width=\figsizefour]{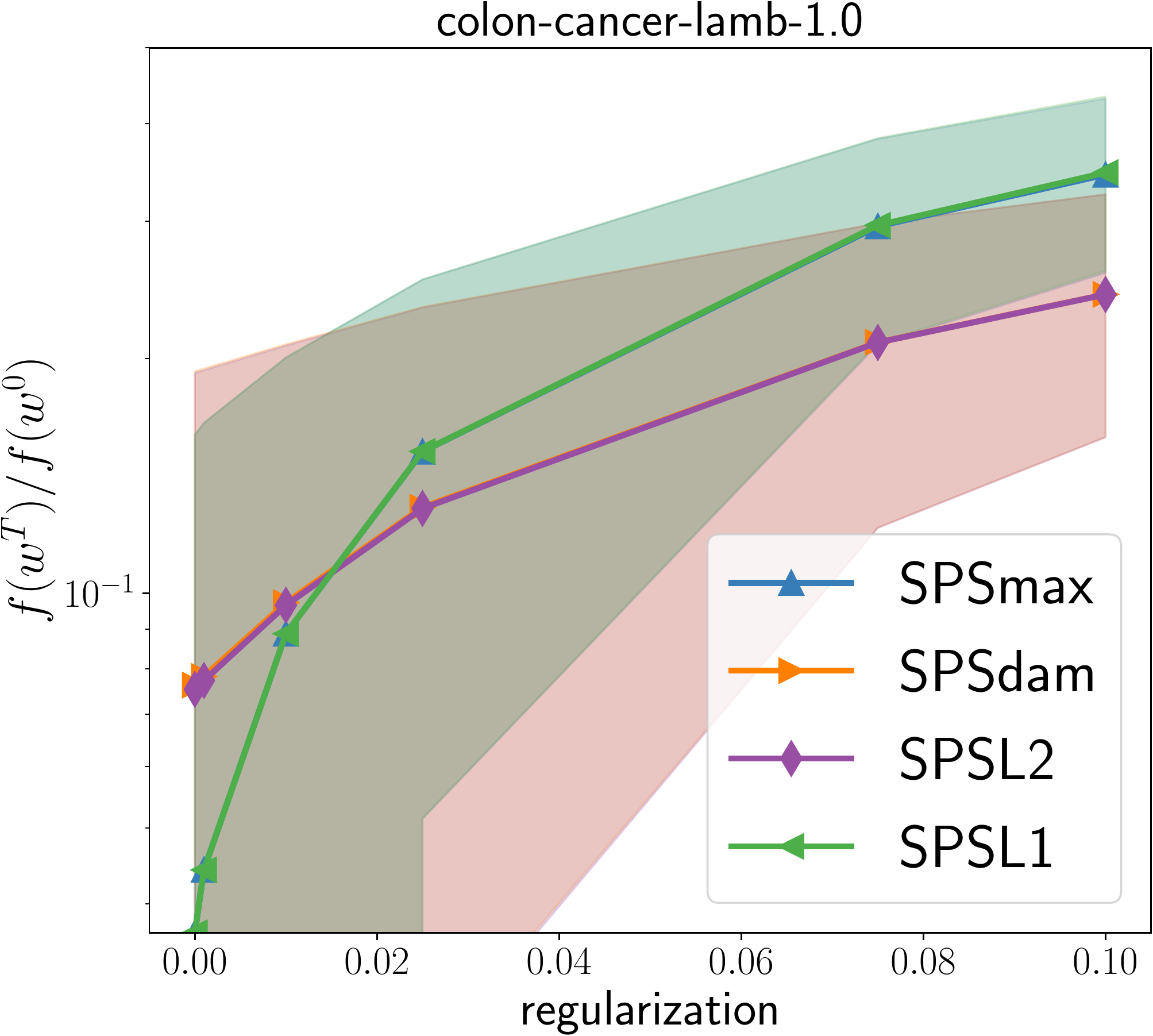} 
    \includegraphics[width=\figsizefour]{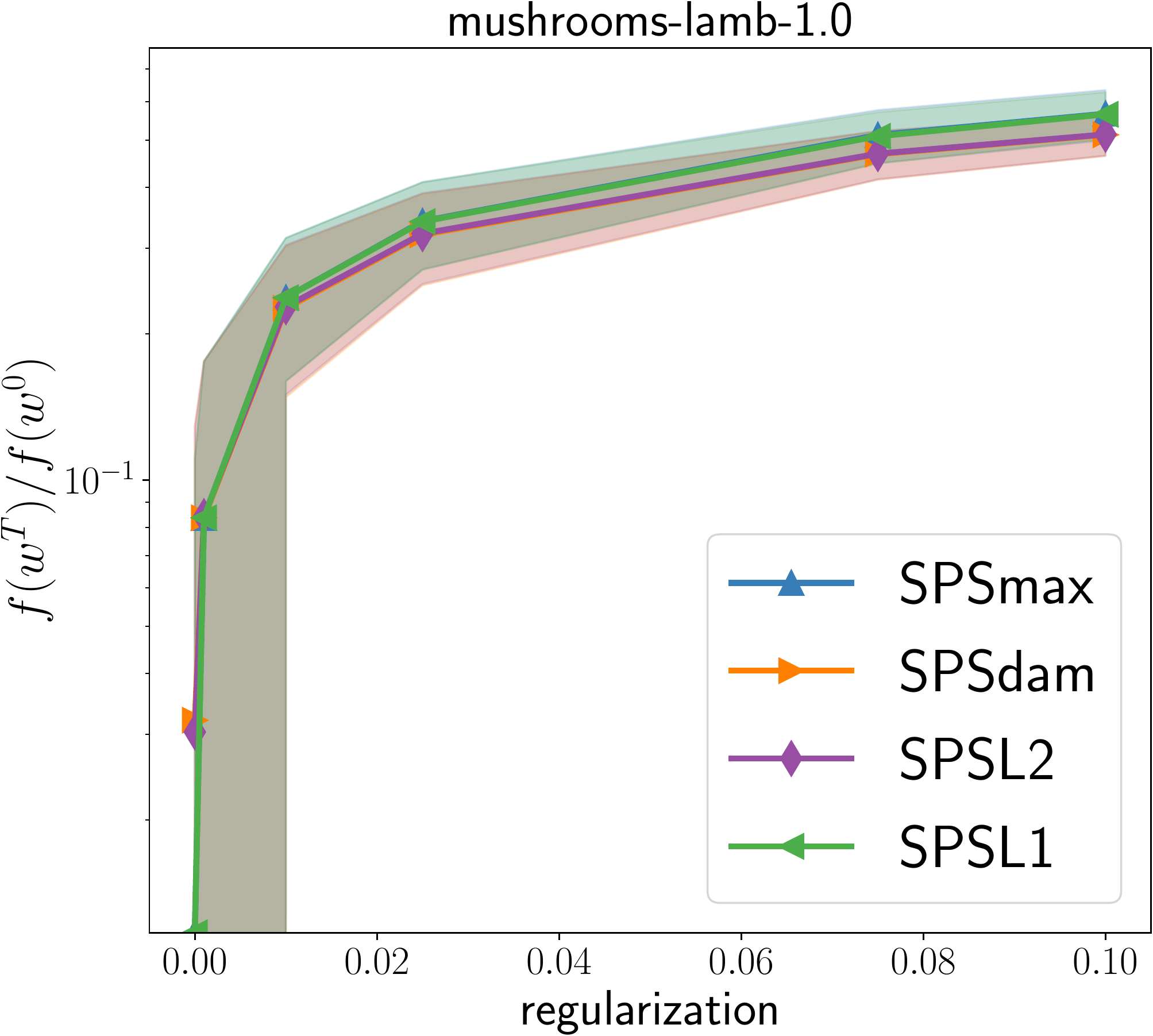}
    \includegraphics[width=\figsizefour]{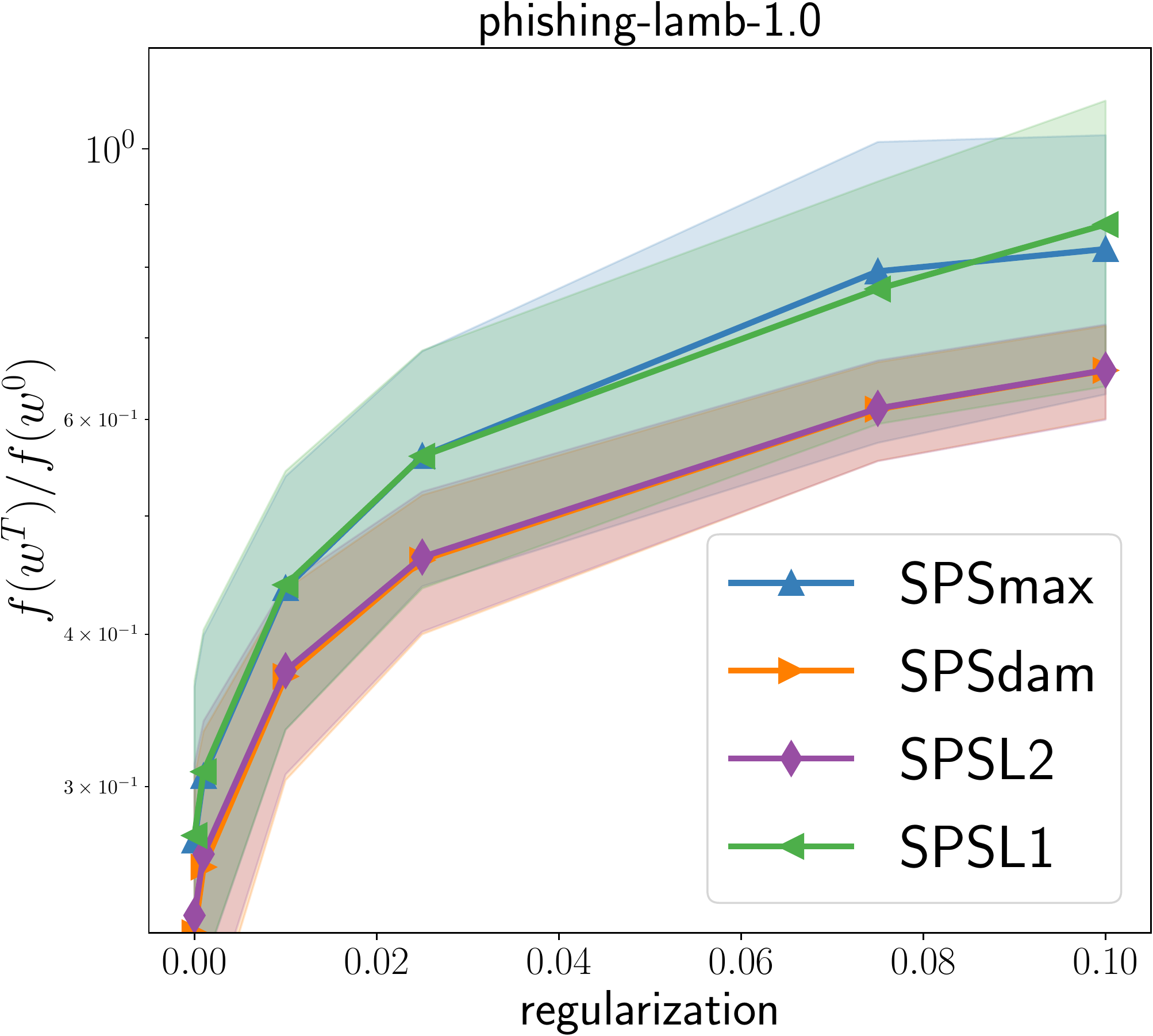} 
    \includegraphics[width=\figsizefour]{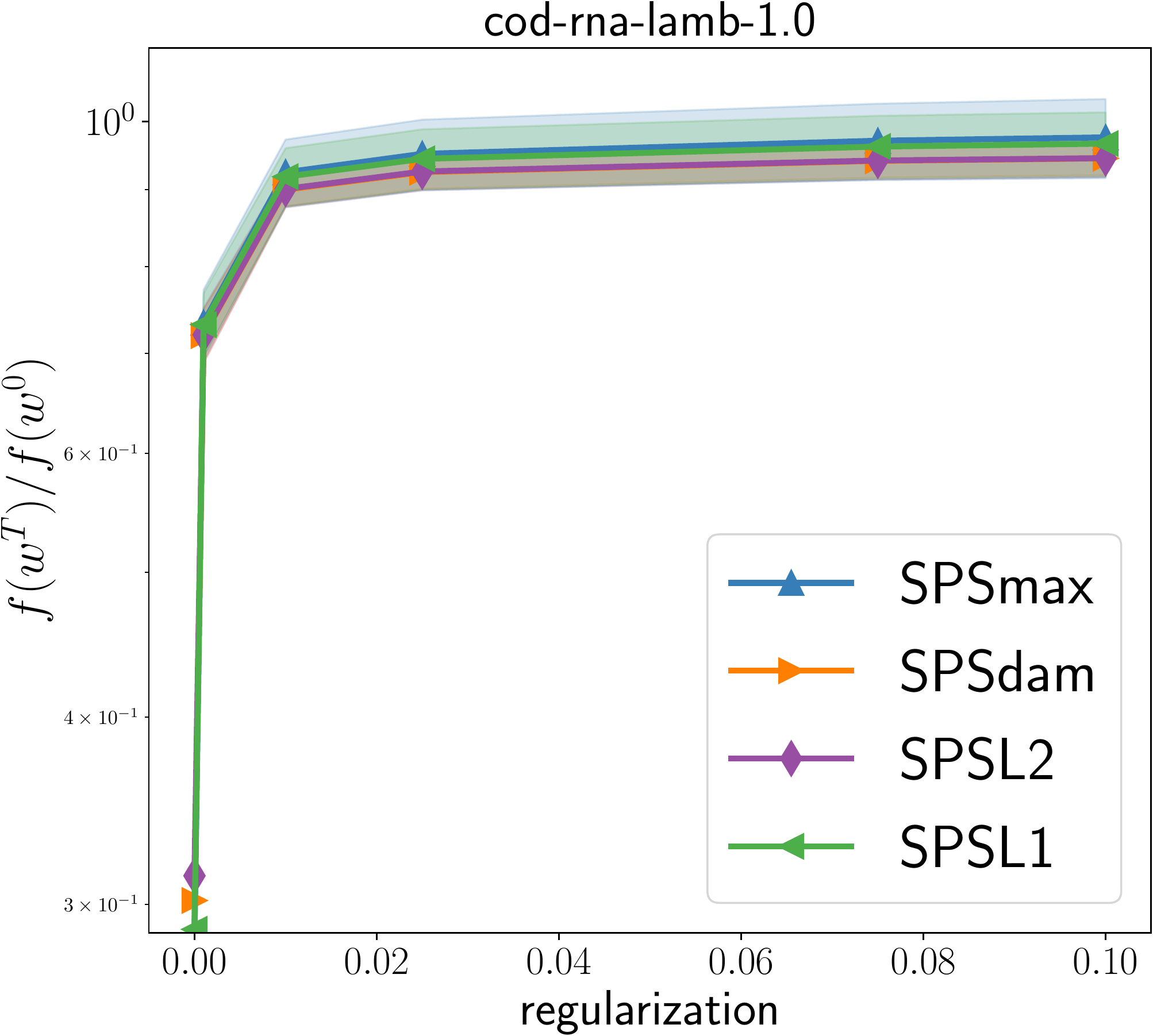}
        \includegraphics[width=\figsizefour]{colon-cancer-lamb-0.01} 
    \includegraphics[width=\figsizefour]{mushrooms-lamb-0.01}
    \includegraphics[width=\figsizefour]{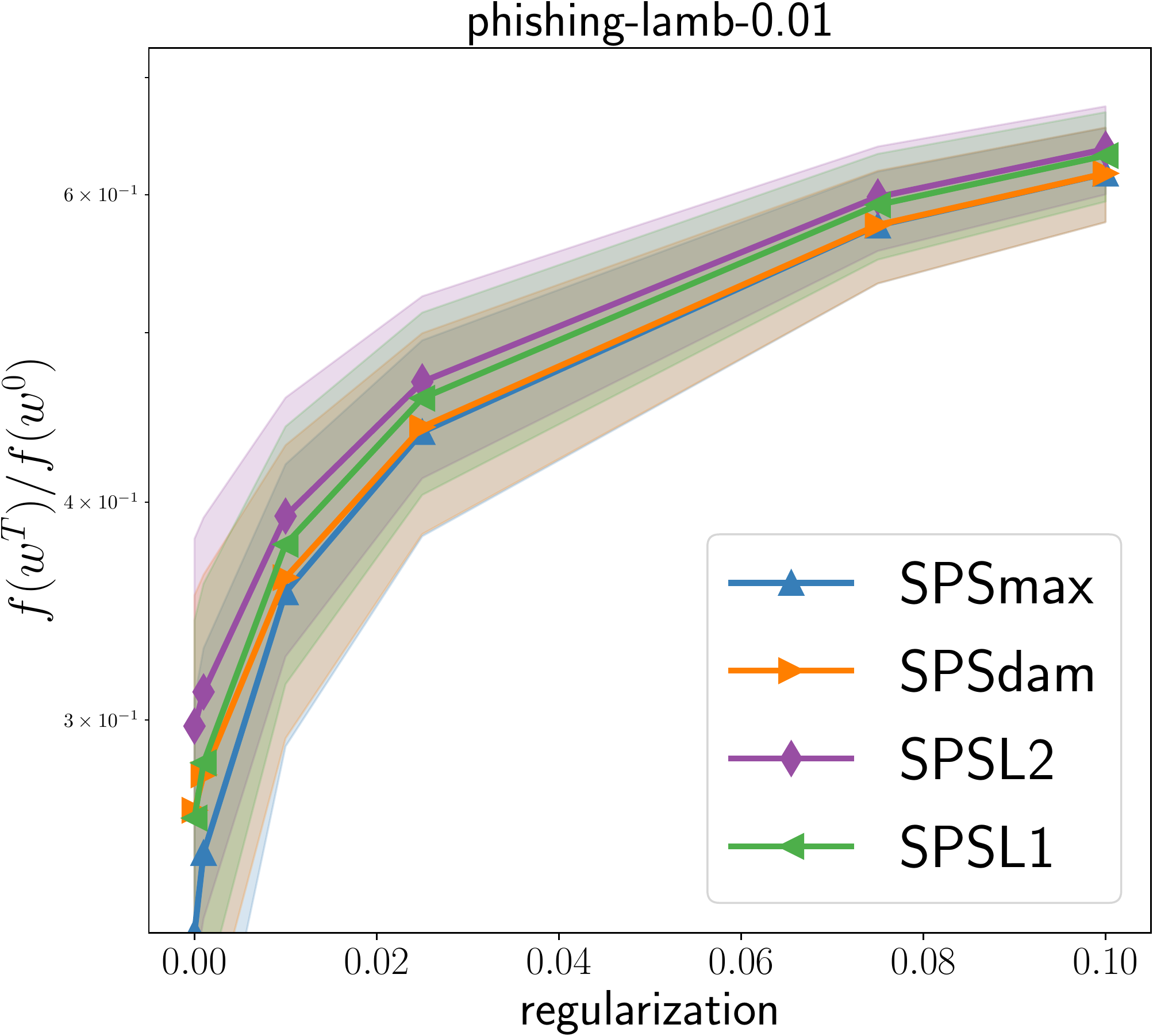} 
    \includegraphics[width=\figsizefour]{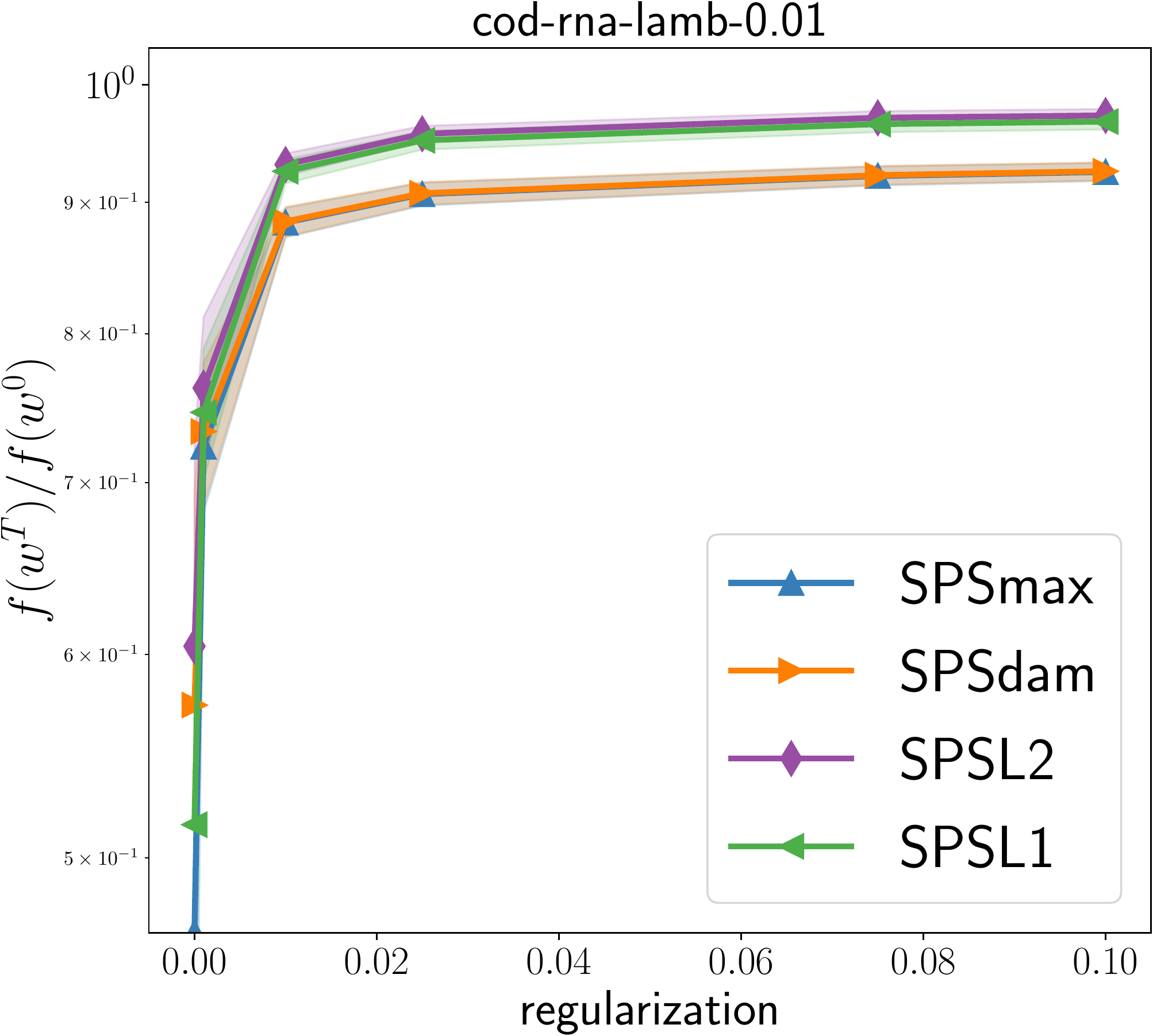}
    \caption{Relative suboptimality ($\ell(w^T)/\ell(w^0)$) for each method when using regularization $reg$ (x-axis) after given a budget of 100 epochs on logistic regression. On the top row, all methods used $\lambda =1.$  On the bottom row, all methods used $\lambda = 0.01.$ The data sets used from left to right are  \texttt{colon-cancer},  \texttt{mushrooms}, \texttt{phishing}, and \texttt{cod-rna}.  }
    \label{fig:regsens}
\end{figure}



\subsection{Comparative study of methods}
Here we compare the empirical convergence of the  four slack methods against \texttt{SGD}  on four different data sets in Figures~\ref{fig:comparsmall} and~\ref{fig:comparbig}. For \texttt{SGD}, we used the $\gamma = 1/2L_{\max}$ as justified~\cite{gower2019sgd}, where $L_{\max} =\frac{\max_{i=1,\ldots, n}}{4} \norm{x_i}^2 + reg. $

We first compare the methods when using a large regularization in Figure~\ref{fig:comparbig}. In accordance to our findings in Section~\ref{sec:Ap-sens-lambda}, we choose a small slack parameter $\lambda =0.01$ for the slack methods.
We then compare the methods when using a small regularization in Figure~\ref{fig:comparbig}, and thus use a large slack parameter $\lambda =1.0$ for the slack methods.

When using a large regularization, as we have in Figure~\ref{fig:comparbig}, the best all round method was \texttt{SGD}. The only exception was the \texttt{colon-cancer} data where \texttt{SPSL1} was clearly the fastest followed by \texttt{SPS}$_{\max}$. We can also see that \texttt{SPSL2} has an advantage over \texttt{SPS}$_{dam}$ on the \texttt{colon-cancer} data set, and matches if (or slightly improves) everywhere else.

When using a small regularization of $reg= 10^{-5}$ in  Figure~\ref{fig:comparsmall}, the \texttt{SPSL1} and  \texttt{SPS}$_{\max}$ perfectly overlapping, and the two fastest methods with exception of the \texttt{phishing} data set where they struggle to converge. They are faster on when using a small regularization because the resulting problem is now closer to interpolation. In Figure~\ref{fig:comparsmall}, \texttt{SPSL2} also closely matches \texttt{SPS}$_{dam}$. This matching is predictable, since for larger $\lambda$, the new regularized slack methods  \texttt{SPSL1}  and  \texttt{SPSL2} converge to  \texttt{SPS}$_{\max}$ and \texttt{SPS}$_{dam}$.

Overall, we can see that the two new regularized slack methods  \texttt{SPSL1}  and  \texttt{SPSL2} either offer small improvements over their counterparts  \texttt{SPS}$_{\max}$ and \texttt{SPS}$_{dam}$, respectively, or have the same performance. Thus supporting our hypothesis that by regularizing the slack results in a more stable method.

\begin{figure}
    \centering
    \centering
    \includegraphics[width=\figsizefour]{colon-cancer-compare-0.01-loss-reg-1.00e-01} 
    \includegraphics[width=\figsizefour]{mushrooms-compare-0.01-loss-reg-1.00e-01}
    \includegraphics[width=\figsizefour]{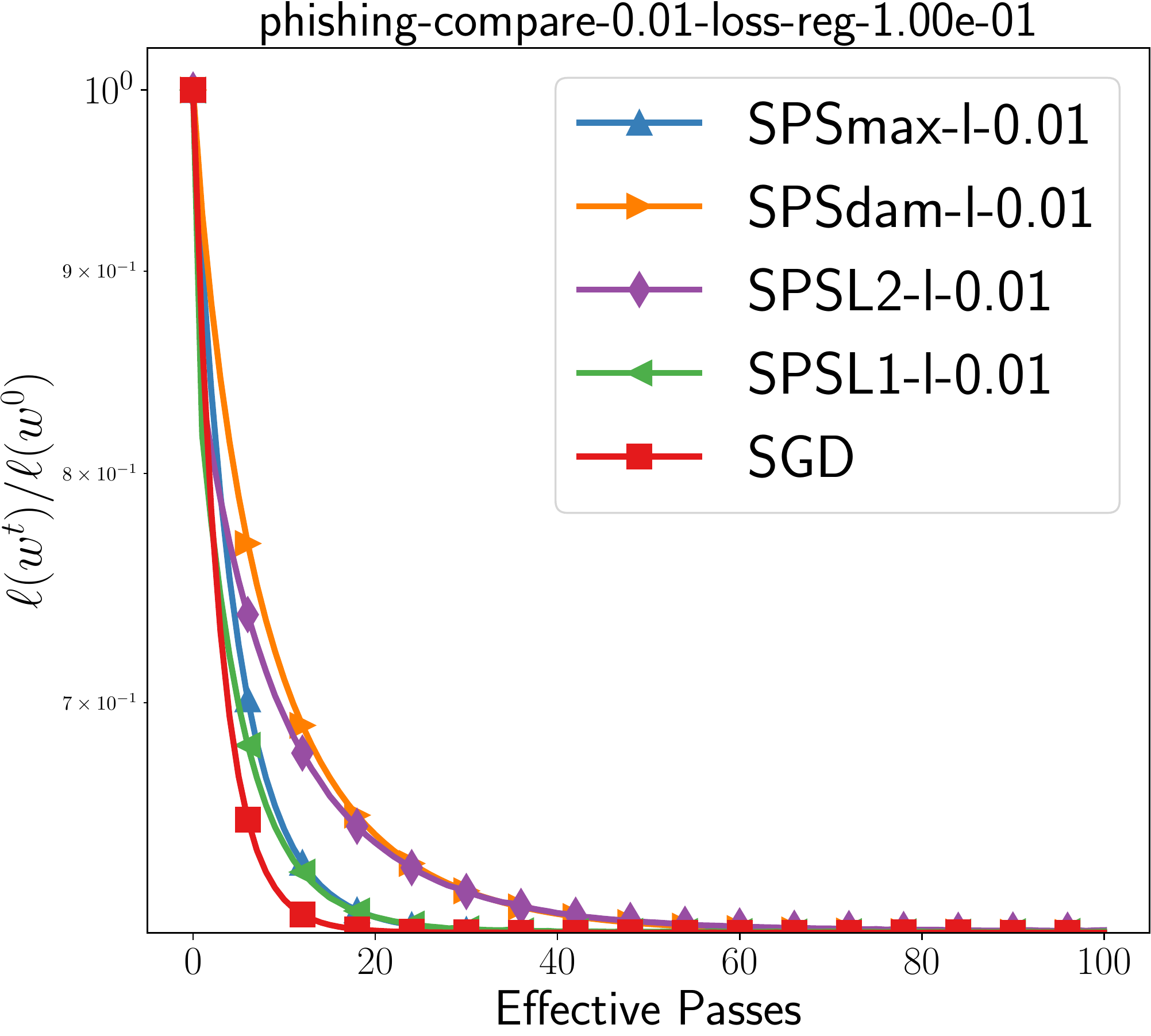} 
    \includegraphics[width=\figsizefour]{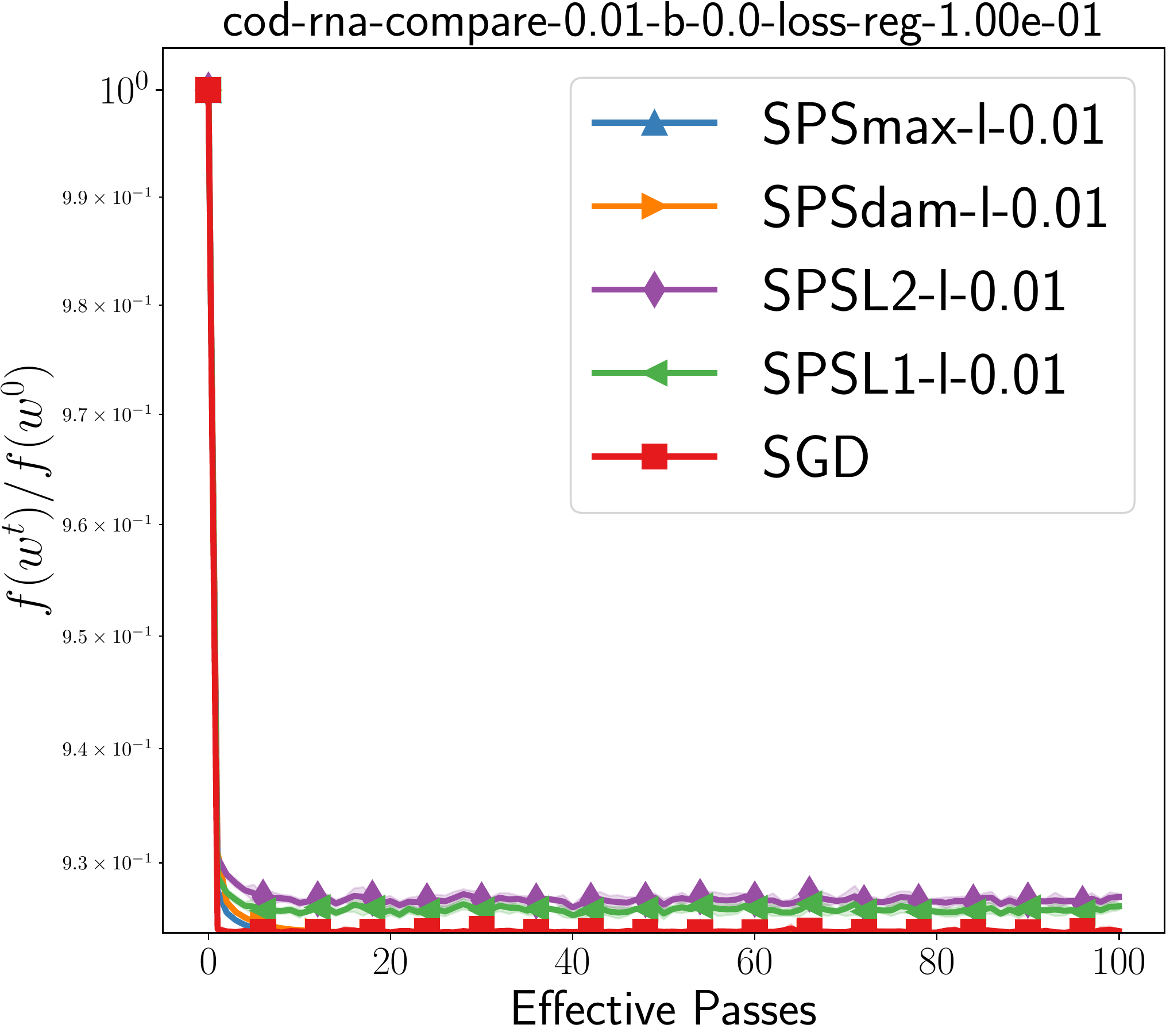} \\
    \includegraphics[width=\figsizefour]{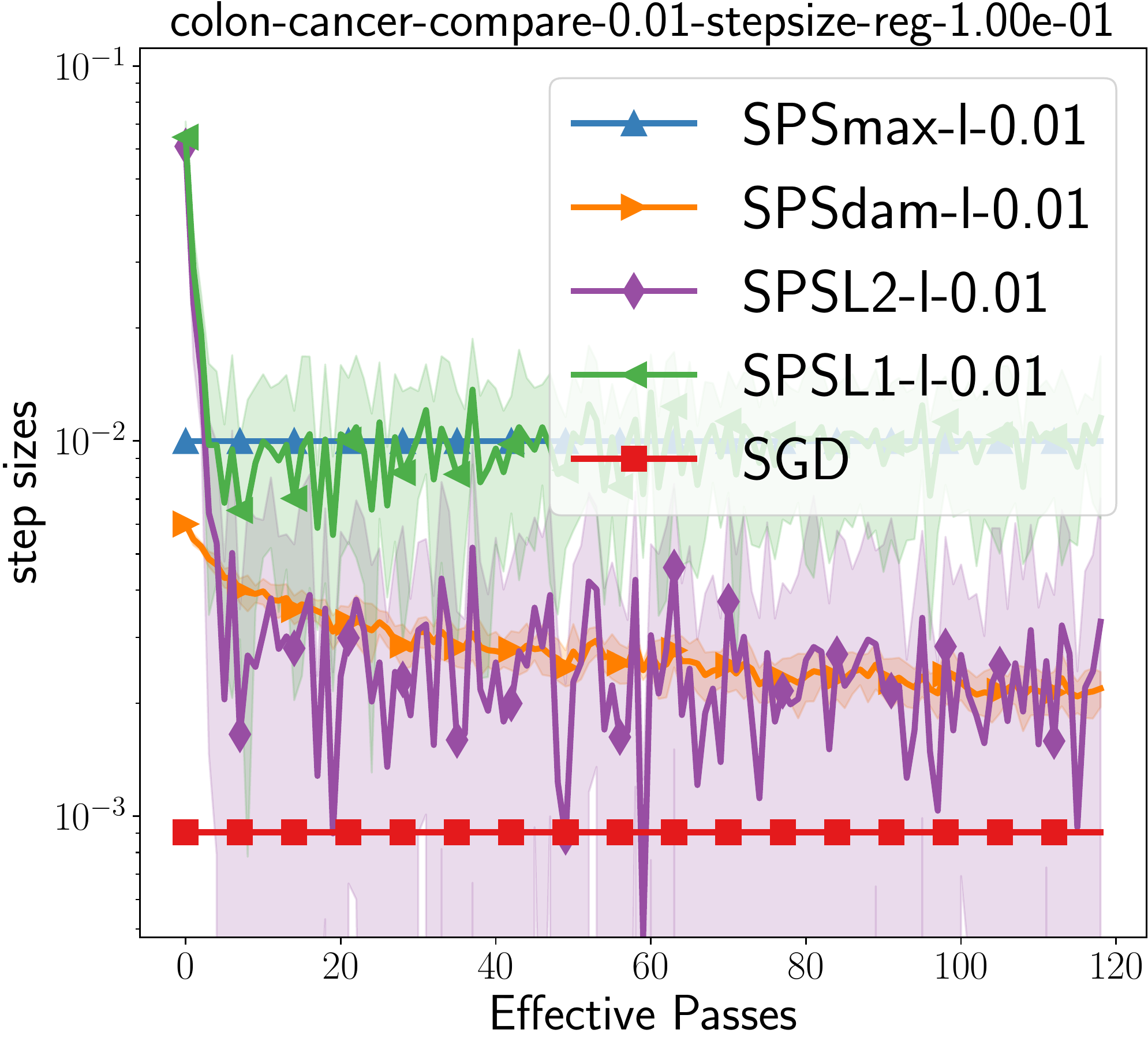} 
    \includegraphics[width=\figsizefour]{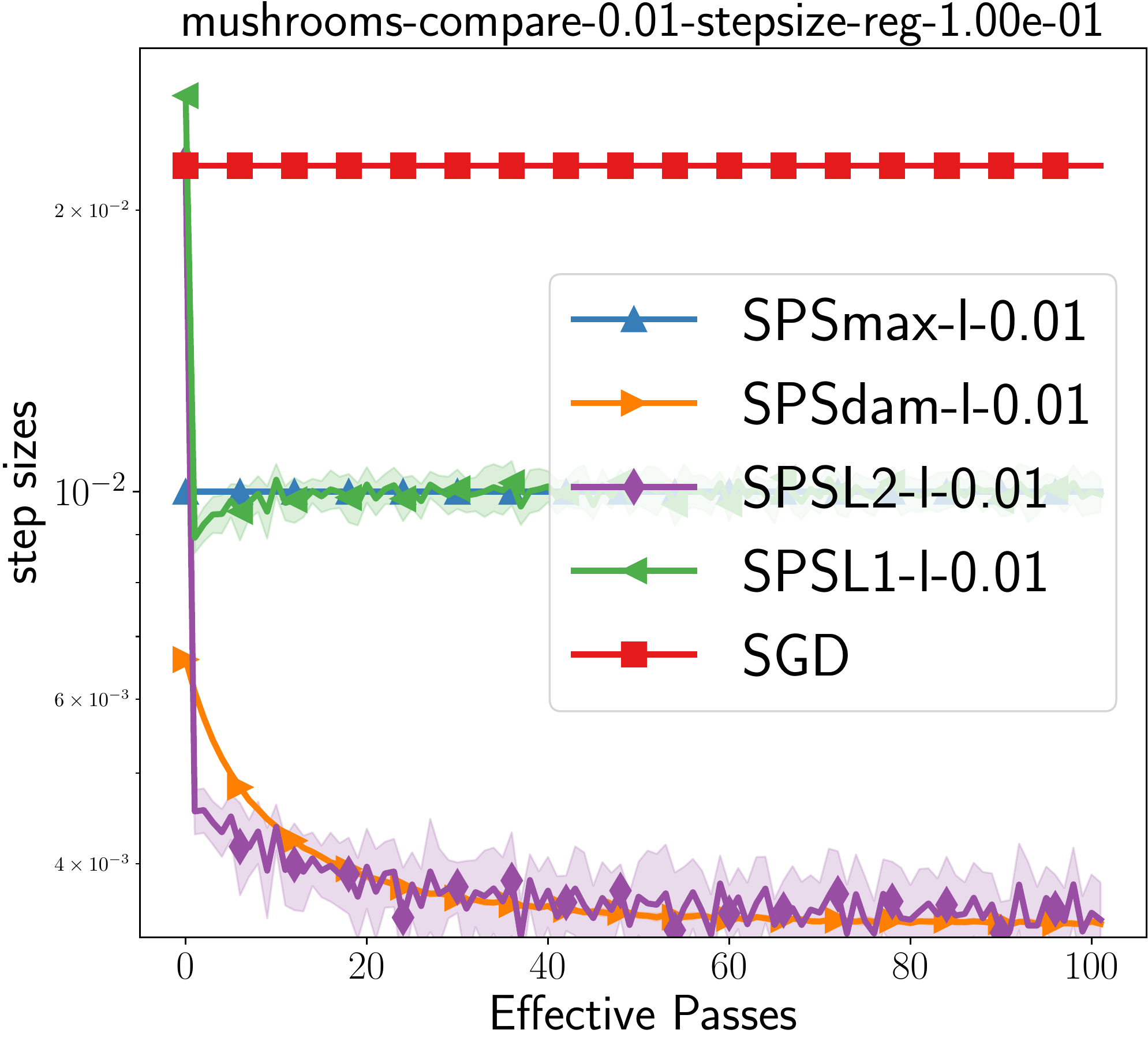}
    \includegraphics[width=\figsizefour]{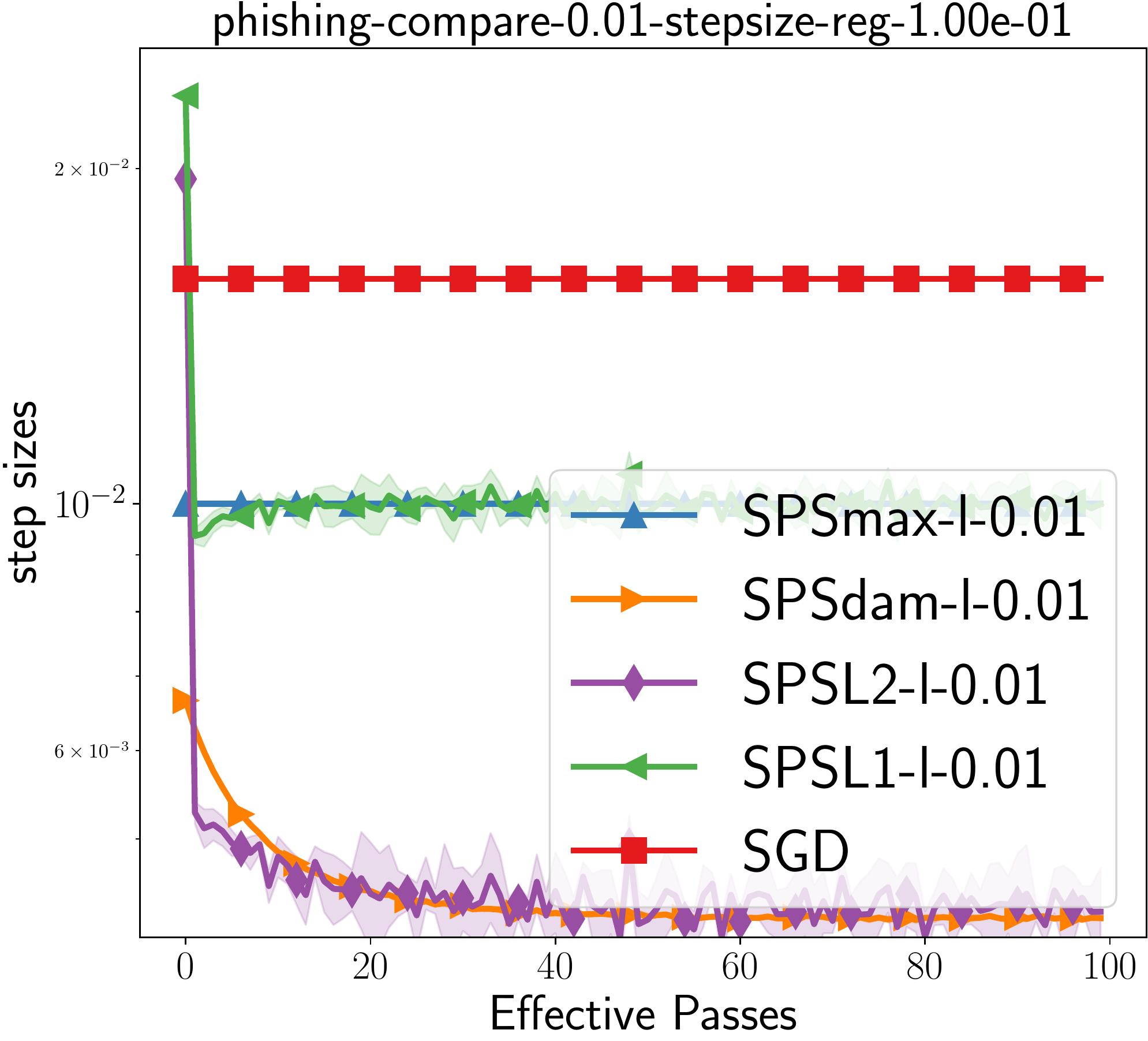}
    \includegraphics[width=\figsizefour]{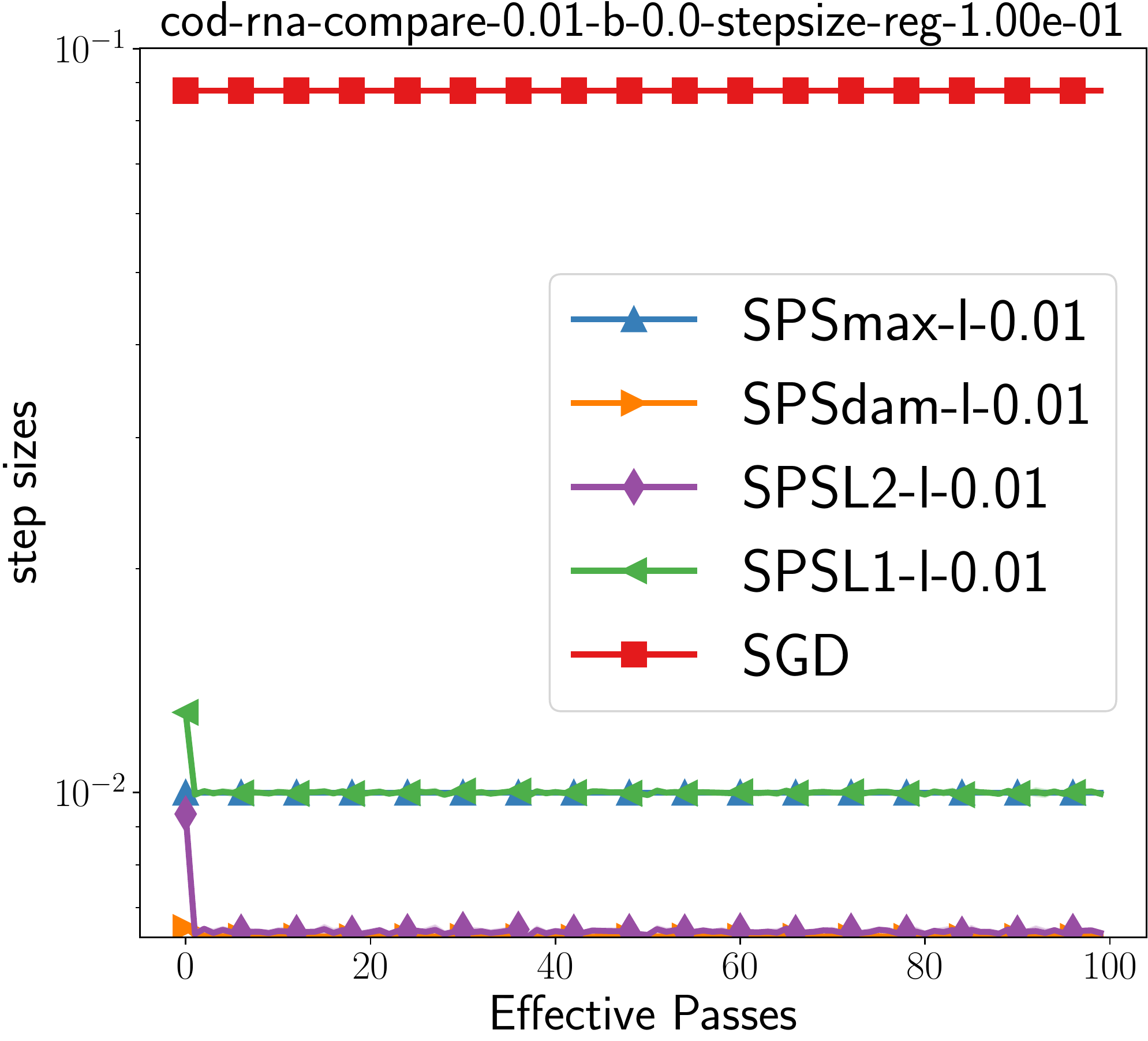}
    \caption{ Comparison of the variants of the \texttt{SPS} method. The data sets used from left to right are \texttt{colon-cancer}, \texttt{mushrooms}, \texttt{phishing}, and \texttt{cod-rna}. The regularization was set to $reg = 10^{-1}$ and $\lambda =0.01$ for all methods. In the top row we report the training error, in the bottom row we report the step size $\norm{w^{t+1}-w^t}^2$ used by each method. }
    \label{fig:comparbig}
\end{figure}

\begin{figure*} 
    \centering
    \includegraphics[width=\figsizefour]{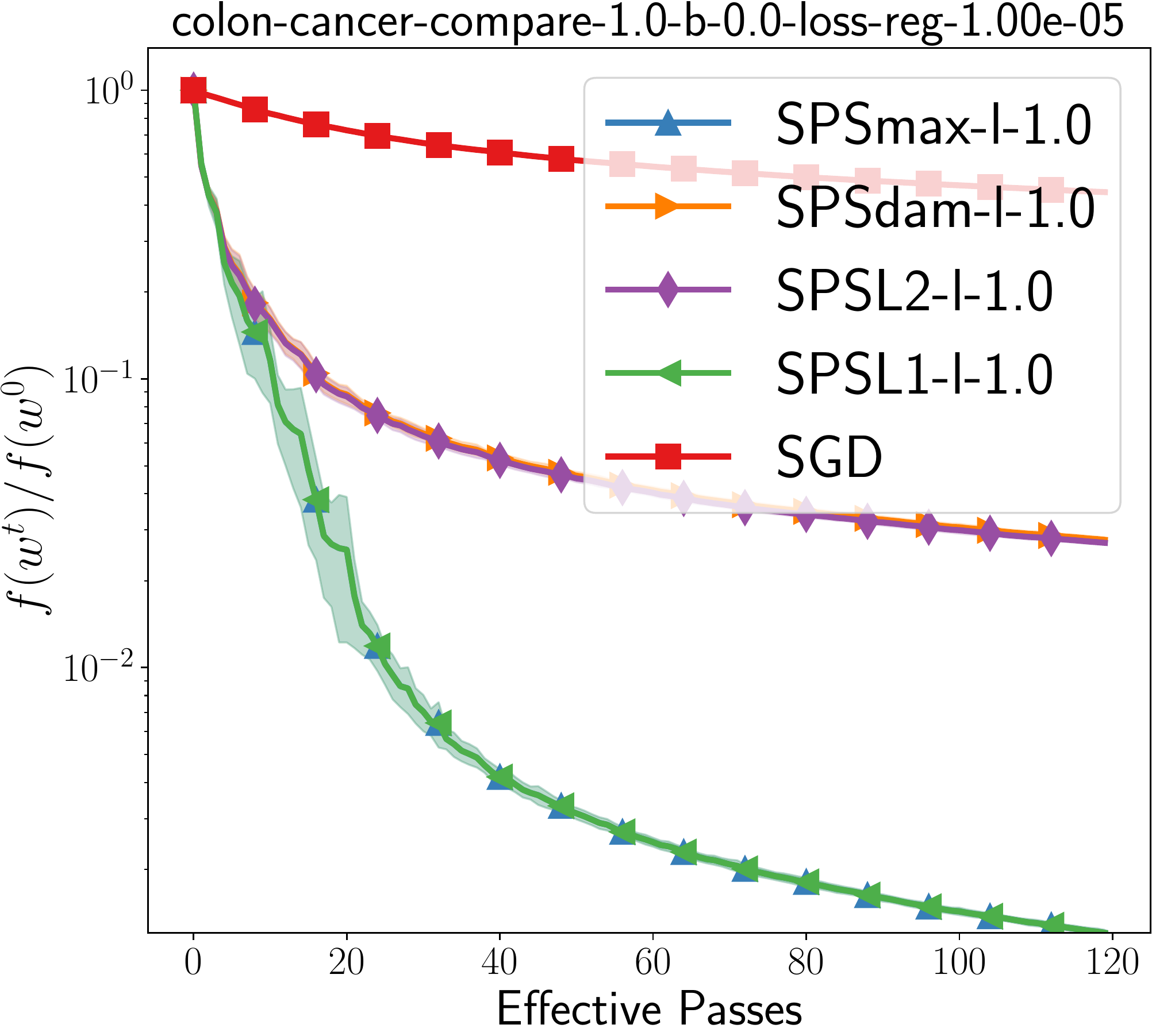} 
    \includegraphics[width=\figsizefour]{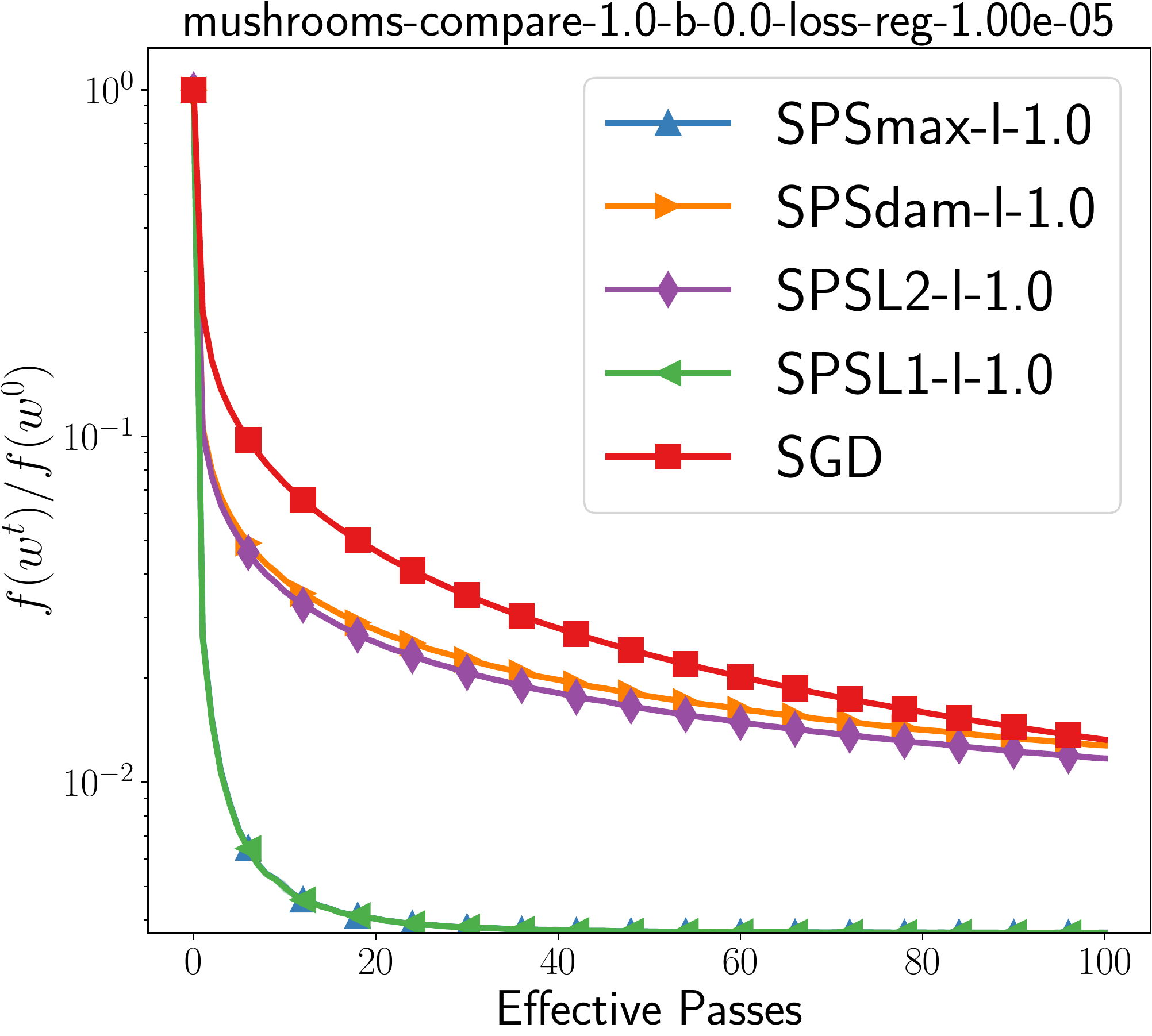}
    \includegraphics[width=\figsizefour]{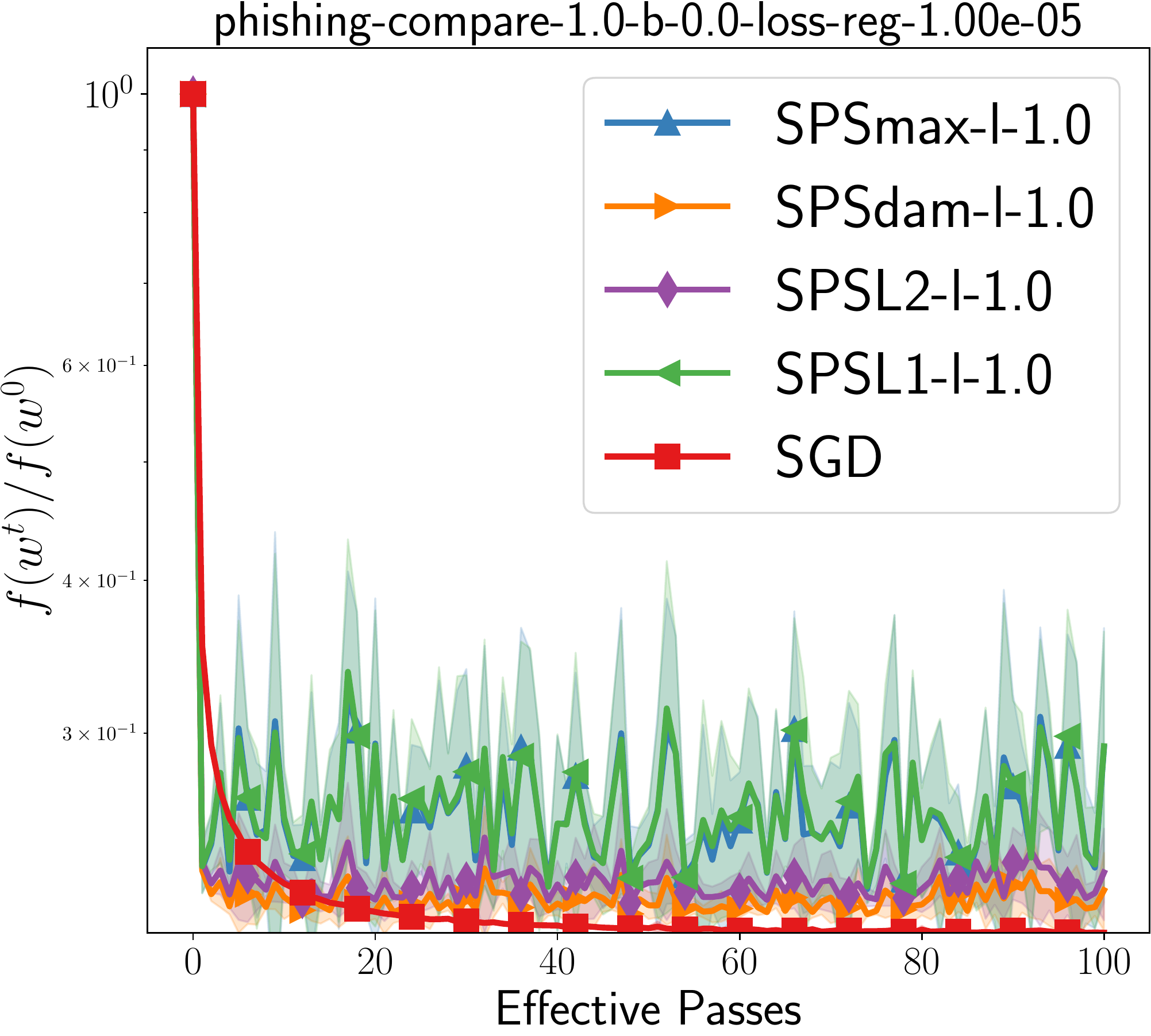} 
    \includegraphics[width=\figsizefour]{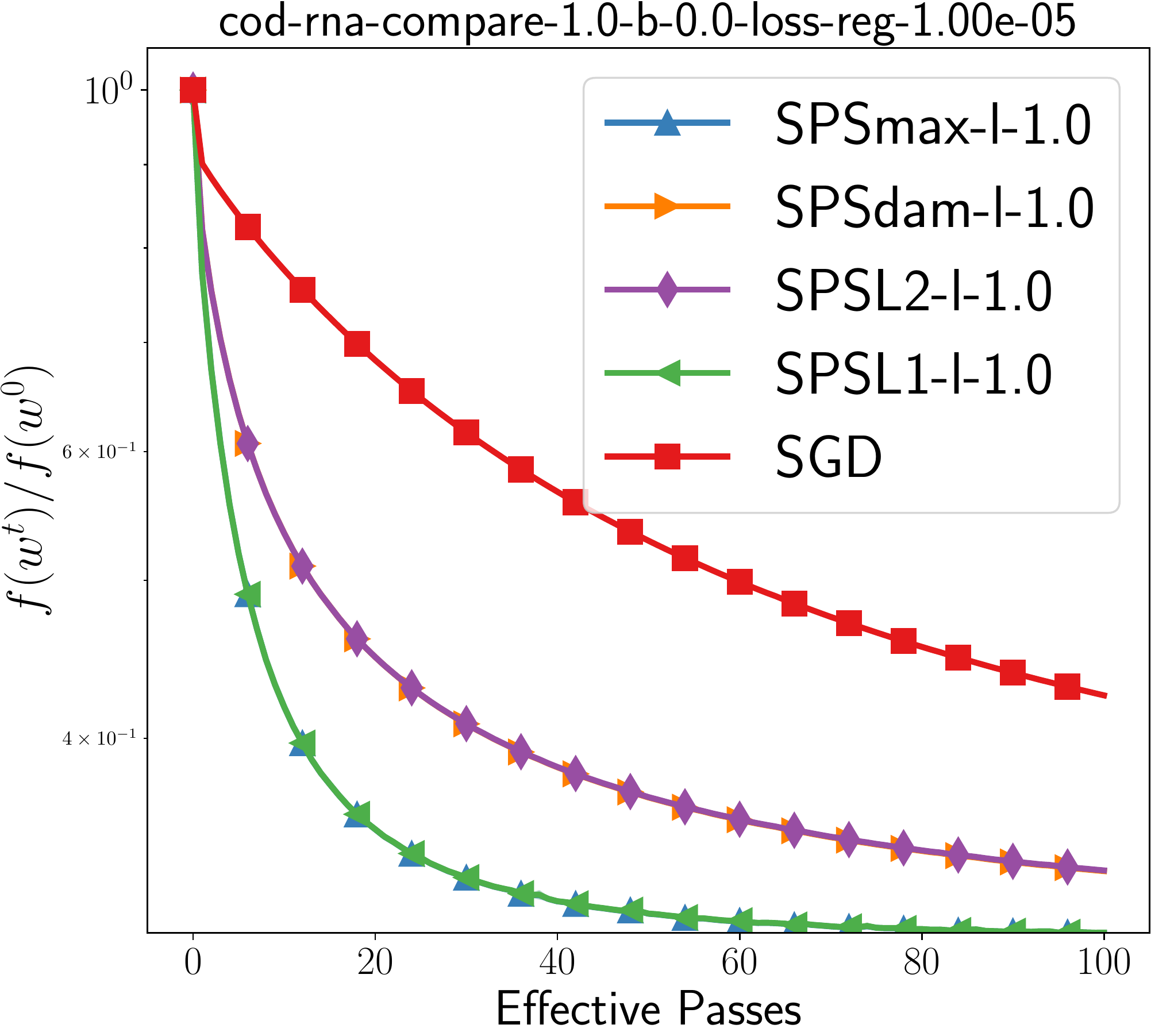} \\
    \includegraphics[width=\figsizefour]{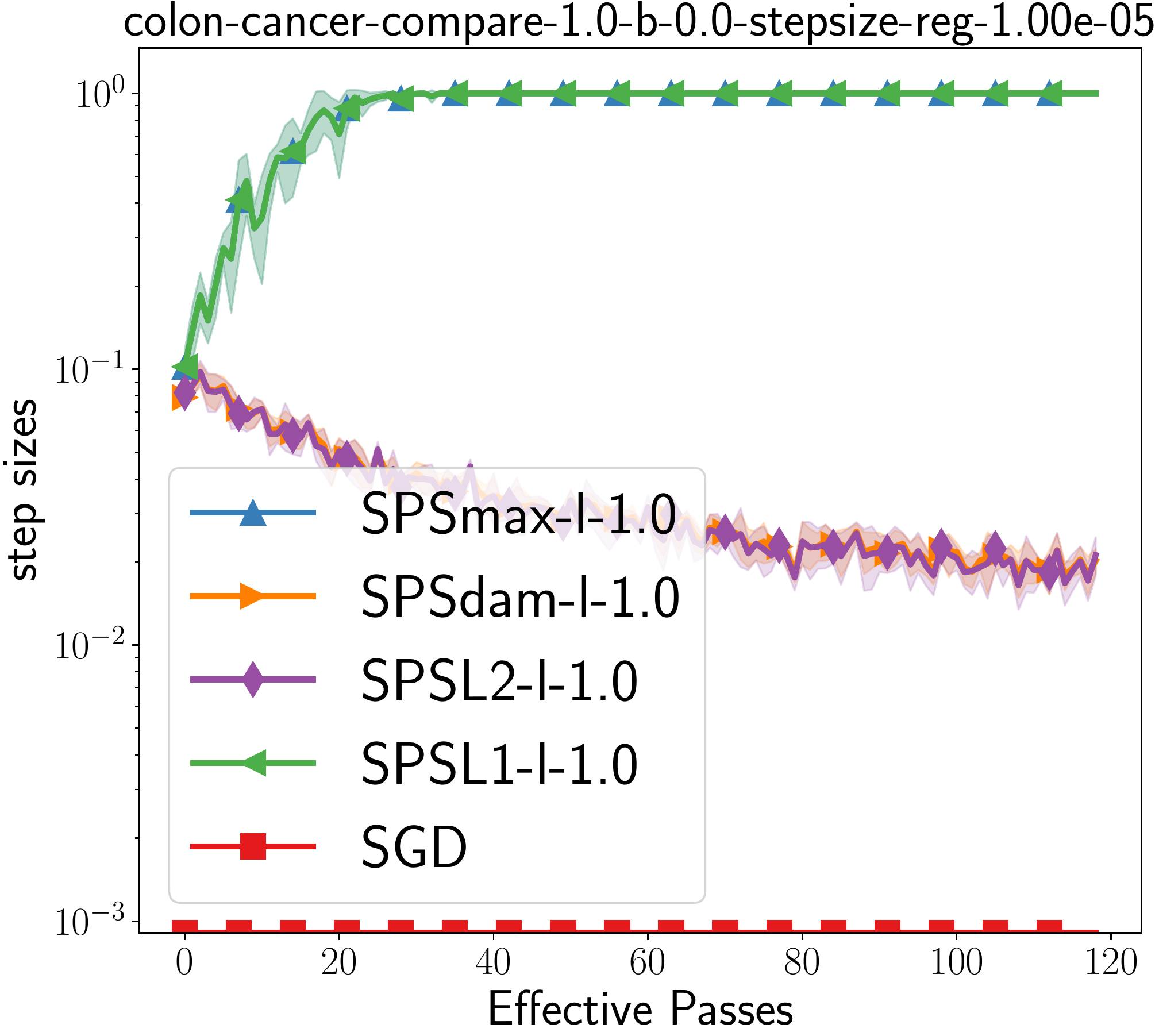} 
    \includegraphics[width=\figsizefour]{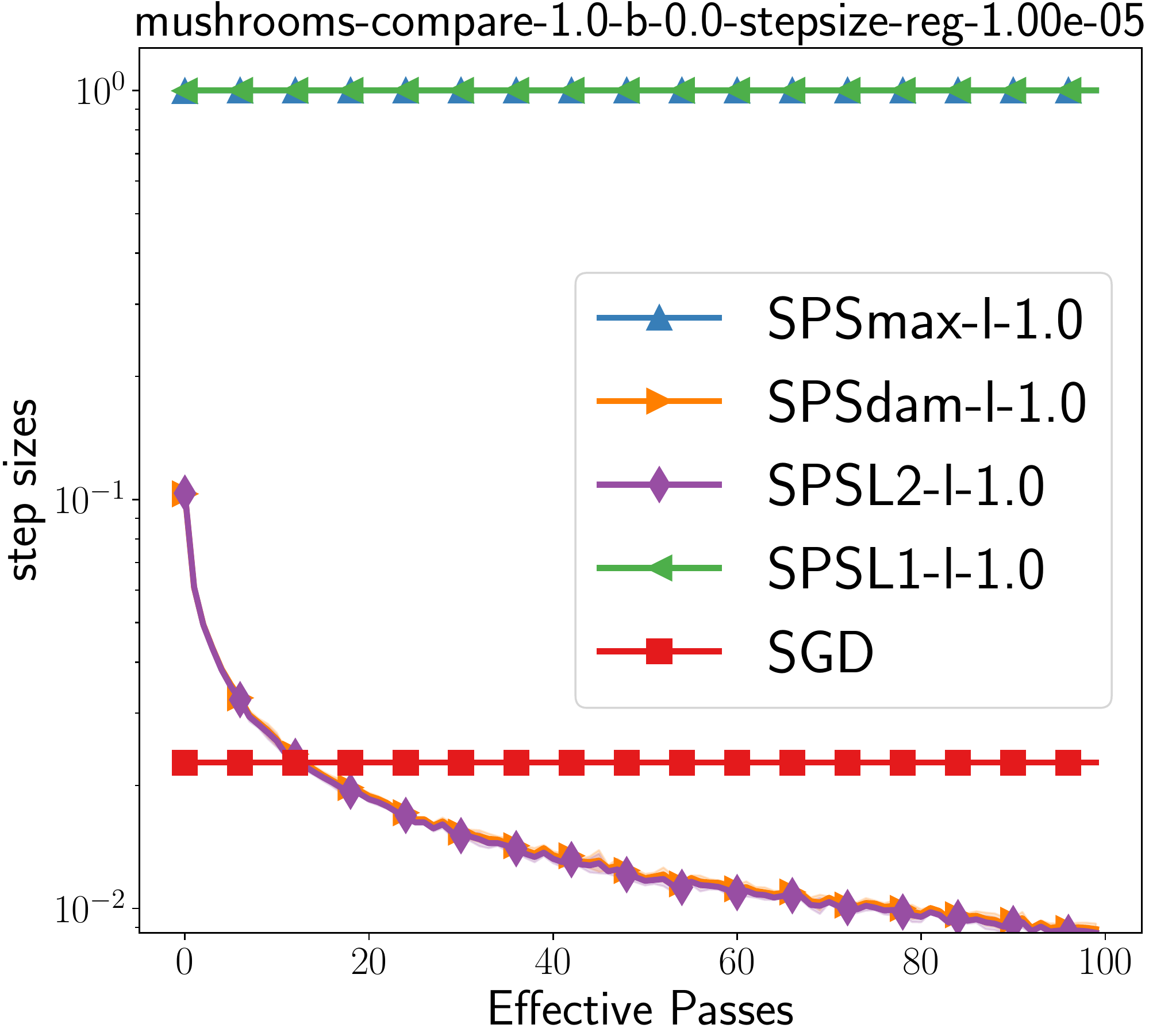}
    \includegraphics[width=\figsizefour]{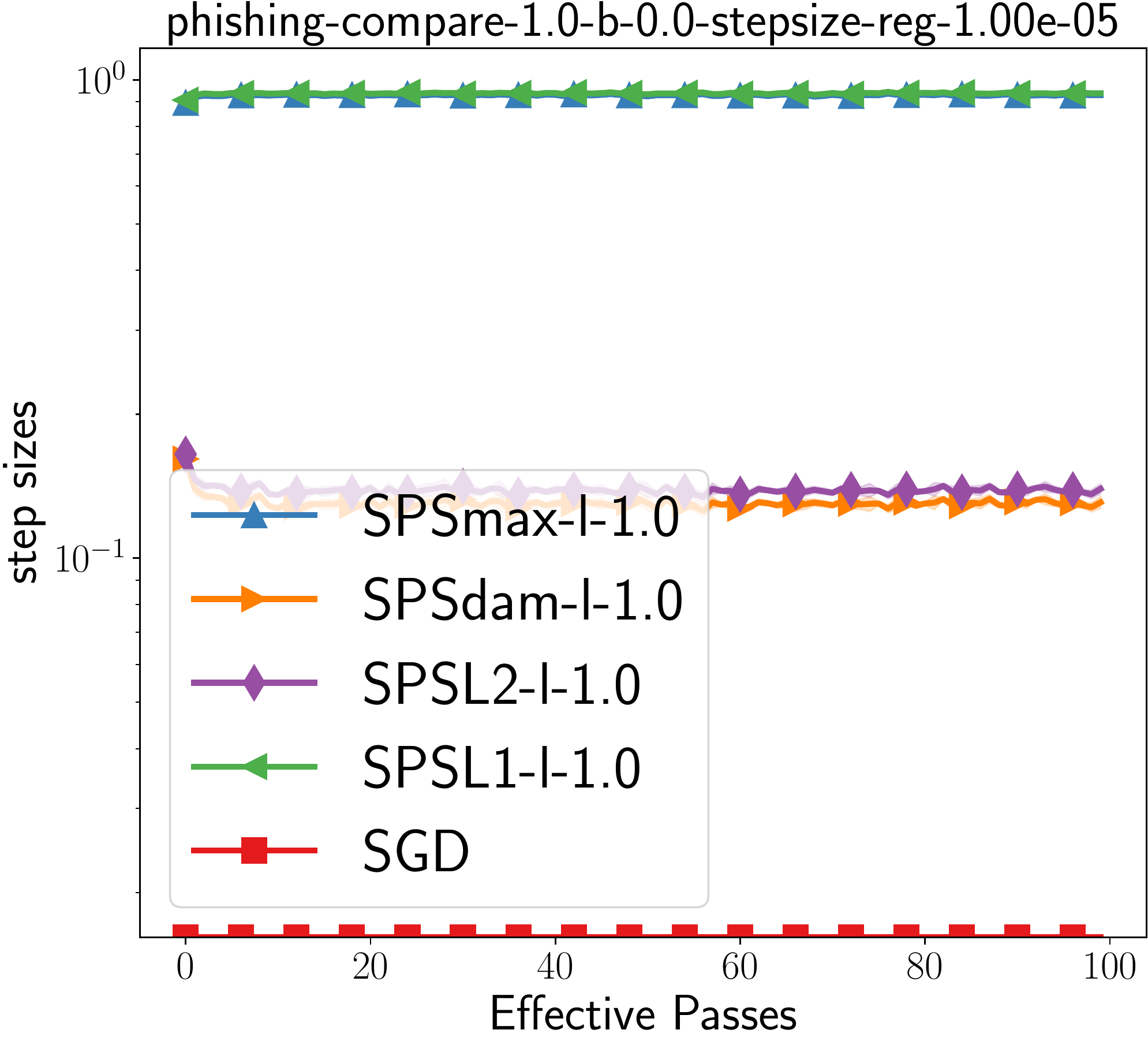} 
    \includegraphics[width=\figsizefour]{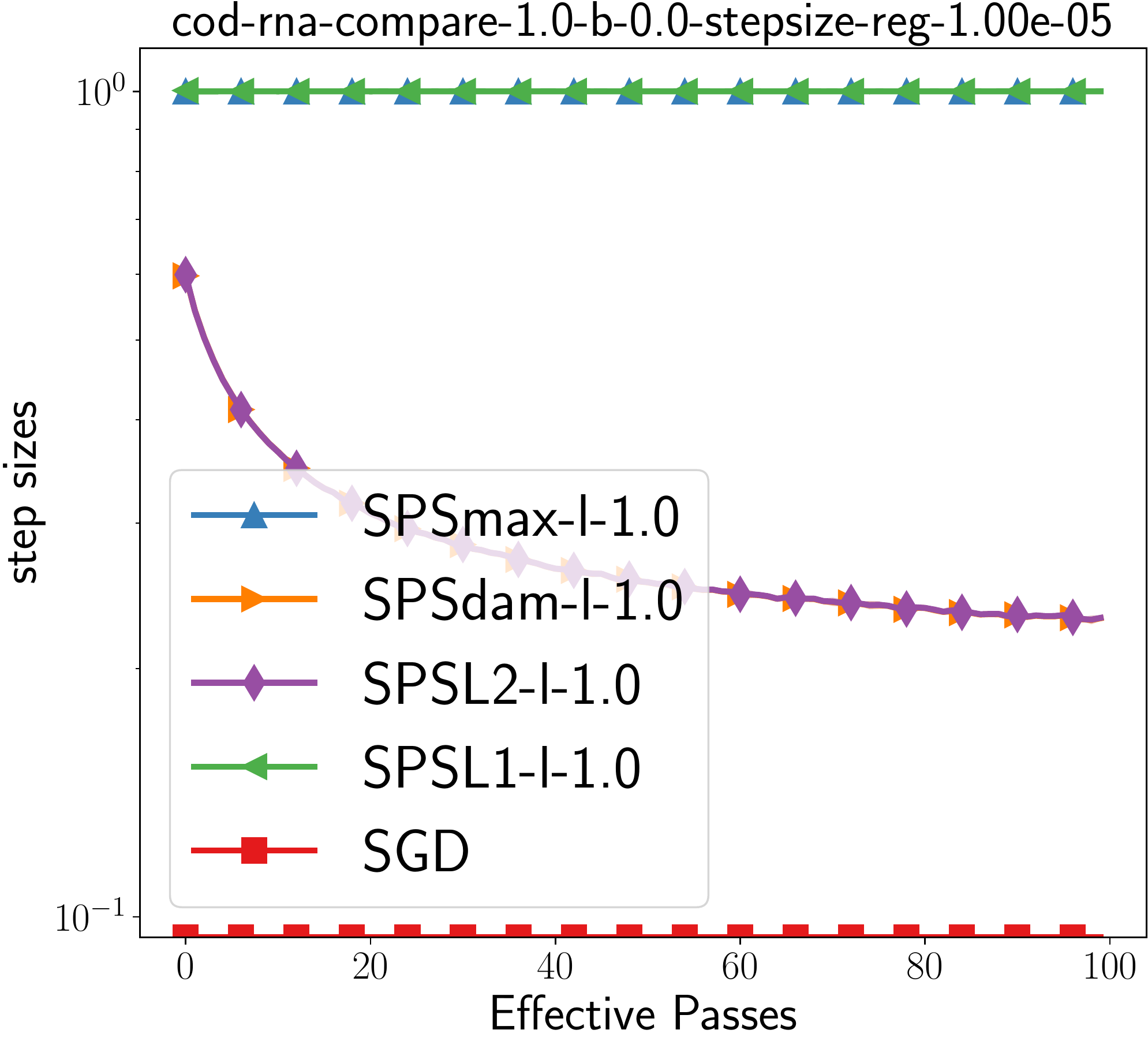}
    \caption{ Comparison of the variants of the \texttt{SPS} method. The data sets used from left to right are \texttt{colon-cancer}, \texttt{mushrooms}, \texttt{phishing}, and \texttt{cod-rna}. The regularization was set to $reg = 10^{-5}$ and $\lambda = 1.0$ for all methods. In the top row we report the training error, in the bottom row we report the step size $\norm{w^{t+1}-w^t}^2$ used by each method.}
    \label{fig:comparsmall}
\end{figure*}
\end{document}